\documentclass{article} 
\usepackage{iclr2026_conference,times}

\usepackage{amsmath,amssymb,amsthm,mathtools}
\usepackage{bm}
\usepackage{booktabs}
\usepackage{graphicx}
\usepackage{enumitem}
\usepackage{microtype}
\usepackage{algorithm}
\usepackage{algpseudocode}
\usepackage{wrapfig}
\usepackage{tikz}
\usetikzlibrary{arrows.meta,calc,positioning,decorations.pathreplacing,backgrounds}
\usepackage[colorlinks=true, urlcolor=blue]{hyperref}
\usepackage{url}
\usepackage[nameinlink,capitalize,noabbrev]{cleveref}
\crefname{appendix}{Appendix}{Appendices}
\Crefname{appendix}{Appendix}{Appendices}

\definecolor{src}{RGB}{37,94,165}
\definecolor{tgt}{RGB}{214,91,32}
\definecolor{transportcol}{RGB}{29,130,86}
\definecolor{candidate}{RGB}{90,155,58}
\definecolor{freegap}{RGB}{113,50,142}
\definecolor{usedg}{RGB}{150,150,150}
\definecolor{popcol}{RGB}{214,39,40}
\definecolor{inact}{RGB}{170,170,170}

\newtheorem{theorem}{Theorem}
\newtheorem{proposition}{Proposition}
\newtheorem{lemma}{Lemma}
\newtheorem{corollary}{Corollary}
\newtheorem{assumption}{Assumption}
\theoremstyle{definition}
\newtheorem{definition}{Definition}
\newtheorem{remark}{Remark}
\newtheorem{example}{Example}

\newcommand{\Sone}{\mathbb{S}^{1}_{L}}
\newcommand{\Sph}{\mathbb{S}^{d-1}}
\newcommand{\Sunit}{\mathbb{S}^{1}}
\newcommand{\dcirc}{d_{\mathbb{S}^{1}}}
\newcommand{\dsph}{d_{\mathbb{S}^{d-1}}}
\newcommand{\PW}{\operatorname{PW}}
\newcommand{\PAWC}{\textnormal{\textsc{PAWC}}}
\newcommand{\PAWL}{\textnormal{\textsc{PAWL}}}
\newcommand{\supp}{\operatorname{supp}}
\newcommand{\ind}{\mathbf{1}}
\newcommand{\RR}{\mathbb{R}}
\newcommand{\arc}[2]{[#1,#2]_{\circlearrowright}}
\newcommand{\arcopen}[2]{(#1,#2)_{\circlearrowright}}

\tikzset{
 srcpt/.style={circle,draw=src,fill=src,inner sep=1.8pt},
 srcin/.style={circle,draw=src,fill=white,inner sep=2.4pt,line width=1pt},
 tgtpt/.style={circle,draw=tgt,fill=tgt,inner sep=1.8pt},
 tgtin/.style={circle,draw=tgt,fill=white,inner sep=2.4pt,line width=1pt},
 gfree/.style={draw=freegap,line width=3pt,line cap=round},
 gused/.style={draw=usedg,line width=3pt,line cap=round},
 rep/.style={draw=freegap!60!black,line width=.6pt,rounded corners=1.5pt,
             minimum width=15pt,minimum height=9pt,inner sep=0pt},
 cell/.style={draw=candidate!80!black,line width=1.8pt,line cap=round},
 cellnew/.style={draw=candidate,densely dashed,line width=1.6pt,line cap=round},
 mold/.style={draw=src!55!black,line width=1pt},
 mnew/.style={draw=popcol,line width=1.1pt,-{Stealth[length=1.8mm]}},
 mrev/.style={draw=inact,densely dotted,line width=1pt},
}
\newcommand{\demoR}{2.1}
\def\aXi{75}\def\aYi{40}\def\aXii{-25}\def\aXiii{-55}
\def\aYii{-85}\def\aYiii{-115}\def\aXiv{-165}\def\aYiv{110}
\def\aYivc{-250}\def\aYic{-320}
\def\gXiYi{57.5}\def\gYiXii{7.5}\def\gXiiXiii{-40}\def\gXiiiYii{-70}
\def\gYiiYiii{-100}\def\gYiiiXiv{-140}\def\gXivYiv{-207.5}\def\gYivXi{92.5}
\newcommand{\basecircle}{%
 \draw[line width=.9pt] (0,0) circle (\demoR);
 \draw[gray!70] (90:\demoR-0.07) -- (90:\demoR+0.12);
 \node[gray,font=\scriptsize] at (90:\demoR+0.34) {$0$};}
\newcommand{\gapmark}[2]{%
 \draw[#1] ([shift={(#2+4.5:\demoR-0.24)}]0,0) arc[start angle=#2+4.5,end angle=#2-4.5,radius=\demoR-0.24];}
\newcommand{\repbox}[1]{\node[rep,rotate=#1+90] at (#1:\demoR-0.24) {};}
\newcommand{\cellarc}[3]{%
 \draw[#1] ([shift={(#2-5:\demoR+0.62)}]0,0) arc[start angle=#2-5,end angle=#3+5,radius=\demoR+0.62];}
\newcommand{\atomn}[5]{\node[#1] at (#2:\demoR) {};
 \node[font=\small,text=#3] at (#2:\demoR+#5) {#4};}

\title{Partial Optimal Transport on the Circle\\ for All Transported Masses in $O(N\log N)$}

\author{\textbf{Soheil Kolouri}\\
College of Connected Computing\\
Vanderbilt University, Nashville, TN\\
\href{mailto:soheil.kolouri@vanderbilt.edu}{soheil.kolouri@vanderbilt.edu}}

\iclrfinalcopy

\begin{document}

\maketitle
\lhead{Preprint.  Under review.}

\begin{abstract}
Partial optimal transport compares two measures while leaving part of the mass
unmatched, which is what makes it robust to outliers, occlusion, and clutter.  The
quantity of interest is usually the whole \emph{profile} --- the optimal cost at every
transported cardinality --- because the right amount to transport is rarely known in
advance, and on the real line the \PAWL{} algorithm returns that profile in $O(N\log
N)$.  Much data is periodic rather than linear: angles, phases, orientations, time of
day, hue, and every direction obtained by projecting onto a great circle.  On the circle
the same problem acquires a global circulation, or equivalently an optimized cut, which
the naive exact method handles by running the line algorithm once per support gap, at
$O(N^{2}\log N)$.  We show that this factor $N$ is unnecessary.  The line structure
survives in cut-free form, and a \emph{free-gap invariant} supplies, at every step, a cut
at which all previous local updates remain valid line updates.  This yields \PAWC{}: an
exact $O(N\log N)$ time, $O(N)$ memory algorithm returning all $K+1$ costs, nested
active sets and plans in one run, together with a single gap that is simultaneously
optimal for every cardinality.  Slicing over great circles extends it to
$\mathbb{S}^{d-1}$.  Empirically the whole profile costs $0.56$\,ms at $N=4096$ against
$1.5$\,s for a single transported fraction from a general solver; on occluded, cluttered
\textsc{mpeg-7} shapes, holding the descriptor fixed and varying only the cost, it
retains $66\%$ of the clean-data retrieval score against $16\%$ for balanced circular OT,
and on $\mathbb{S}^{2}$ it halves the fitting error of spherical sliced Wasserstein
against contaminated targets, synthetic and real.
Code is available at \url{https://github.com/mint-vu/Partial_Wasserstein_on_Circles}.
\end{abstract}

\section{Introduction}
\label{sec:intro}

Optimal transport (OT) provides a geometry-aware framework for comparing measures \citep{villani2008optimal,peyre2019computational}, but exact discrete solvers are generally expensive. Partial transport extends this framework by allowing only part of the distributions to be matched: conceptually, mass may be transported from the source to the target, destroyed if left unmatched at the source, or created if unmatched at the target, with corresponding transportation, destruction, and creation costs. In one dimension, \citet{bai2023sliced} exploited the ordering of the supports to develop an exact primal-dual solver for a fixed choice of these penalties, with worst-case complexity $O(n\max\{m,n\})$, or $O(N^2)$ for distributions of comparable size. More recently, \citet{chapel2025one} showed that the entire solution path indexed by the transported mass can be computed much faster: their \PAWL{} algorithm returns exact partial transport plans for every possible transported mass in $O(N\log N)$ time, using nested active sets, neighboring candidate pairs, balanced chains, and a heap of chain marginals.

Unmatched mass is often part of the problem rather than a modeling nuisance. A partial correspondence isolates the structure shared by two measures without forcing spurious detections, occluded features, mislocated records, or outliers into geometrically implausible matches \citep{nietert2022outlier}. The appropriate amount of mass to match, however, is rarely known beforehand and may depend on a downstream robustness or coverage requirement. The natural object is therefore the entire partial-transport profile
\(
k \longmapsto C_k^\circ,
\)
where $C_k^\circ$ is the minimum cost of matching exactly $k$ pairs. This profile quantifies the tradeoff between coverage and geometric fidelity: it shows which correspondences are retained as the transported mass increases and when admitting additional mass begins to require costly matches. Computing all $C_k^\circ$ at once allows the transported mass to be selected after examining this tradeoff, rather than fixed before solving the transport problem \citep{chapel2025one}.

Periodic geometry is ubiquitous rather than exceptional. Angles, phases, orientations, time-of-day variables, and hue are naturally represented on the circle, while many sliced constructions for directional data on the sphere produce measures supported on great circles. However, the cyclic geometry of the circle prevents a direct application of the efficient algorithms available on the line.
 The circle retains a cyclic order but has no canonical first point, last point, or zero-flow boundary. For balanced $W_1$, this missing boundary condition appears as a free additive circulation in the flow formulation, or equivalently as a cut over which a lifted line problem must be optimized \citep{delon2010fast,rabin2011transportation}. Partial transport couples this topological freedom to the transported mass: as $k$ varies, both the active atoms and the optimal circulation, and hence the optimal cut, may change. Consequently, no single unwrapping of the circle recovers the full profile $k\mapsto C_k^\circ$. A direct exact reduction cuts the circle at every support gap, runs \PAWL{} on each resulting line, and takes the lower envelope of the resulting profiles. Although correct, this repeats an $O(N\log N)$ computation across $O(N)$ cuts and therefore costs $O(N^2\log N)$. The central algorithmic challenge is to recover the entire circular partial-transport profile without paying separately for every possible cut.

\paragraph{Contributions.}
The line structure survives on the circle in cut-free form (\cref{sec:structure}): circular-optimal active sets can be chosen nested, the two atoms
activated at each step are cyclic neighbours among the inactive ones, and a
\emph{free-gap invariant} --- every current cell always contains an original gap not yet
used by any selected cell --- supplies a cut at which all previous local updates are
simultaneously valid line updates.  This gives an exact greedy rule and a gap that is
optimal for \emph{every} cardinality at once, constructed rather than searched for
(\cref{thm:greedy-main,thm:cut-main}), and an $O(N\log N)$ time, $O(N)$ memory algorithm
returning all costs, nested active sets and plans in one run (\cref{sec:algorithm}).
Slicing over great circles extends it to $\Sph$ (\cref{sec:sphere}), and
\cref{sec:experiments} verifies exactness against an LP oracle, measures the scaling,
and gives two applications where the profile is the object that matters --- partial
matching of shape descriptors on \textsc{mpeg-7}, and robust fitting on $\mathbb{S}^{2}$
including a real earthquake catalogue.

\section{Background}
\label{sec:setting}

\subsection{Partial optimal transport on the circle}

Let $\Sone=\RR/L\mathbb{Z}$ with geodesic distance
$\dcirc(x,y)=\min\{|x-y|,L-|x-y|\}$, and let
$\mu=w\sum_{i=1}^{n}\delta_{x_i}$ and $\nu=w\sum_{j=1}^{m}\delta_{y_j}$ be uniformly
weighted empirical measures; write $N=n+m$ and $K=\min(n,m)$.  For a transported mass
$s\in[0,Kw]$, define the circular partial $1$-Wasserstein cost by
\begin{equation}
\PW_\circ(s):=\min_{\pi\in\RR_+^{n\times m}}\Bigl\{
\sum_{i=1}^{n}\sum_{j=1}^{m}\dcirc(x_i,y_j)\pi_{ij}
\;:\;\sum_{j=1}^{m}\pi_{ij}\le w,\;
\sum_{i=1}^{n}\pi_{ij}\le w,\;
\sum_{i,j}\pi_{ij}=s\Bigr\}.
\label{eq:partial-circular}
\end{equation}
Thus $\PW_\circ(s)$ is the minimum cost of transporting exactly $s$ units of mass, with
at most the full mass $w$ of each atom participating.  Relaxing mass conservation to an
inequality and fixing the transported total is the partial problem of
\citet{caffarelli2010free}, who established existence and uniqueness for costs $h(x-y)$
with $h$ convex and disjoint supports; \citet{figalli2010optimal} relaxed disjointness
for the quadratic cost and studied the map $s\mapsto\PW(s)$.  It is that map, not any
single value of it, that we compute.

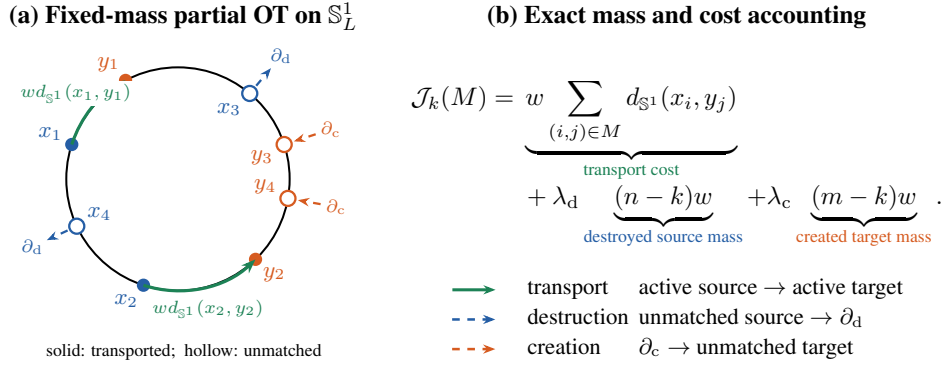
\begin{figure}[t]
\centering
\resizebox{0.94\textwidth}{!}{\begin{minipage}[t]{0.36\linewidth}
\centering
\textbf{(a) Fixed-mass partial OT on $\mathbb S^1_L$}\par\smallskip
\begin{tikzpicture}[
  scale=1.0,font=\small,
  srcpt/.style={circle,draw=src,fill=src,inner sep=1.6pt,line width=.7pt},
  srcinactive/.style={circle,draw=src,fill=white,inner sep=2.1pt,line width=.9pt},
  tgtpt/.style={circle,draw=tgt,fill=tgt,inner sep=1.6pt,line width=.7pt},
  tgtinactive/.style={circle,draw=tgt,fill=white,inner sep=2.1pt,line width=.9pt},
  transport/.style={draw=transportcol,line width=1.25pt,-{Stealth[length=2mm]}},
  discard/.style={draw=src,dashed,line width=.9pt,-{Stealth[length=1.6mm]}},
  create/.style={draw=tgt,dashed,line width=.9pt,-{Stealth[length=1.6mm]}}
]
  \def\R{1.55}
  \draw[line width=.8pt] (0,0) circle (\R);

  \node[srcpt,label={[src,inner sep=1pt]165:$x_1$}] at (162:\R) {};
  \node[tgtpt,label={[tgt,inner sep=1pt]115:$y_1$}] at (118:\R) {};
  \node[srcpt,label={[src,inner sep=1pt]250:$x_2$}] at (252:\R) {};
  \node[tgtpt,label={[tgt,inner sep=1pt]315:$y_2$}] at (314:\R) {};

  \node[srcinactive,label={[src,inner sep=1pt]230:$x_3$}] at (50:\R) {};
  \node[tgtinactive,label={[tgt,inner sep=1pt]198:$y_3$}] at (18:\R) {};
  \node[srcinactive,label={[src,inner sep=1pt]25:$x_4$}]  at (205:\R) {};
  \node[tgtinactive,label={[tgt,inner sep=1pt]170:$y_4$}] at (350:\R) {};

  \draw[transport] (162:\R) arc[start angle=162,end angle=118,radius=\R];
  \draw[transport] (252:\R) arc[start angle=252,end angle=314,radius=\R];
  \node[transportcol,font=\scriptsize,fill=white,inner sep=1pt] at (140:1.85) {$w \dcirc(x_1,y_1)$};
  \node[transportcol,font=\scriptsize,fill=white,inner sep=1pt] at (283:1.85) {$w \dcirc(x_2,y_2)$};

  \draw[discard] (50:\R+0.14)  -- (50:\R+0.46);
  \draw[discard] (205:\R+0.14) -- (205:\R+0.46);
  \node[src,font=\scriptsize] at (50:\R+0.68)  {$\partial_{\rm d}$};
  \node[src,font=\scriptsize] at (205:\R+0.68) {$\partial_{\rm d}$};

  \draw[create] (18:\R+0.46)  -- (18:\R+0.14);
  \draw[create] (350:\R+0.46) -- (350:\R+0.14);
  \node[tgt,font=\scriptsize] at (18:\R+0.68)  {$\partial_{\rm c}$};
  \node[tgt,font=\scriptsize] at (350:\R+0.68) {$\partial_{\rm c}$};
\end{tikzpicture}
\par\smallskip
{\scriptsize solid: transported; \ hollow: unmatched}
\end{minipage}\hfill
\begin{minipage}[t]{0.62\linewidth}
\centering
\textbf{(b) Exact mass and cost accounting}\par\smallskip
\begin{equation*}
\begin{aligned}
\mathcal J_k(M)
={}&\underbrace{w\sum_{(i,j)\in M}\dcirc(x_i,y_j)}_{\textcolor{transportcol}{\text{transport cost}}}\\[-1mm]
&+\lambda_{\rm d}\underbrace{(n-k)w}_{\textcolor{src}{\text{destroyed source mass}}}
+\lambda_{\rm c}\underbrace{(m-k)w}_{\textcolor{tgt}{\text{created target mass}}}.
\end{aligned}
\end{equation*}
\par\smallskip
{\small
\begin{tabular}{@{}ll@{\hspace{4pt}}l@{}}
\tikz\draw[transportcol,line width=1.2pt,-{Stealth[length=1.7mm]}] (0,0)--(.6,0); & transport & active source $\to$ active target\\[.6mm]
\tikz\draw[src,dashed,line width=1pt,-{Stealth[length=1.7mm]}] (0,0)--(.6,0); & destruction & unmatched source $\to\partial_{\rm d}$\\[.6mm]
\tikz\draw[tgt,dashed,line width=1pt,-{Stealth[length=1.7mm]}] (0,0)--(.6,0); & creation & $\partial_{\rm c}\to$ unmatched target
\end{tabular}}
\end{minipage}}
\caption{Partial transport on the circle.  (a) solid arcs are geodesic transports
between selected atoms; hollow atoms are left unmatched, their mass destroyed at
$\partial_{\rm d}$ or created at $\partial_{\rm c}$.  (b) at a fixed transported cardinality
$|M|=k$ those two amounts are constants, so any fixed penalties $\lambda_{\rm d},
\lambda_{\rm c}$ contribute a constant and the problem is exactly the
cardinality-constrained matching \eqref{eq:profile}.}
\label{fig:problem-accounting}
\end{figure}

\begin{assumption}[Standing assumptions]
\label{ass:standing-main}
All atoms carry the same mass $w$; the points $\{x_i\}\cup\{y_j\}$ are pairwise
distinct; and the ground cost is $\dcirc$.
\end{assumption}

At an integer mass $s=kw$, scaling the feasible set by $w$ gives the cardinality-$k$
bipartite matching polytope.  Its extreme points are integral, and hence
\begin{equation}
C_k^{\circ}:=\PW_\circ(kw)=\min_{M:\,|M|=k}\;
w\!\!\sum_{(x_i,y_j)\in M}\!\!\dcirc(x_i,y_j),
\qquad k=0,\dots,K,
\label{eq:profile}
\end{equation}
where $M$ ranges over matchings, so no atom appears in more than one pair.  We call
$k\mapsto C_k^{\circ}$ the \emph{profile} and denote by $A(M)$ the set of atoms incident
to $M$, called its \emph{active set}.  For $s\in[kw,(k+1)w]$, write
$\lambda=(s-kw)/w$; the value interpolates linearly,
\begin{equation}
\PW_\circ(s)=(1-\lambda)\,C_k^{\circ}+\lambda\,C_{k+1}^{\circ},
\label{eq:profile-interpolation}
\end{equation}
so the discrete profile $\{C_k^{\circ}\}_{k=0}^{K}$ determines $\PW_\circ(s)$ for every
$s\in[0,Kw]$; see \cref{app:plans}.

\subsection{Computing partial transport}
Exact solvers for OT and its unbalanced variants are cubic; the standard remedies
are entropic regularisation \citep{cuturi2013sinkhorn} and the one-dimensional
quantile formula underlying sliced transport
\citep{rabin2012wasserstein,peyre2019computational}. For partial transport,
\citet{phatak2023computing} compute the full mass profile in $O(n^3)$;
\citet{bai2023sliced} give a quadratic one-dimensional solver;
\citet{bonneel2019spot} solve injective partial assignment in quasi-linear time;
\citet{sejourne2022faster} solve one-dimensional unbalanced OT quasi-linearly but
cannot prescribe the transported mass; and \citet{chapel2021unbalanced} obtain a
regularisation path. \PAWL{} \citep{chapel2025one} is the first to return the entire
partial profile on the line in $O(n\log n)$, and is what we extend.

On the circle, prior work is balanced: \citet{delon2010fast} and
\citet{rabin2011transportation} solve $W_1$ by optimising a cut or global
circulation; \citet{martin2024lcot} linearise circular OT; and
\citet{bonet2023spherical,liu2025linear} use circular transport within spherical
slicing. None addresses prescribed-mass partial transport on the circle.

\paragraph{Why the circle is not the line.}
For a \emph{fixed} active set the circular cost is the minimum of the unwrapped line
costs over the $N$ support gaps (\cref{prop:fixed-cut} in \cref{app:background}), so the
envelope over gaps and cardinalities is an exact $O(N^{2}\log N)$ reference solver
(\cref{cor:cut-envelope}, also there) ---
which we use as a baseline in \cref{sec:experiments} and as a trusted second opinion in
testing.  \Cref{ex:one-cut-unbounded} in \cref{app:greedy} shows that reusing one fixed cut can be off
by an unbounded factor, and \cref{sec:numerics-cut} measures how often a single cut would
in fact have sufficed.

\section{Structure of circular partial transport}
\label{sec:structure}

\subsection{Nested active sets and cyclic neighbours: $O(N)$ candidates per step}

\begin{theorem}[Nested optimal extension; proved in \cref{app:nested}]
\label{thm:nested-main}
Let $M_k$ be any optimal cardinality-$k$ matching, $k<K$.  Then there is an optimal
cardinality-$(k+1)$ matching $M_{k+1}$ with $A(M_{k+1})=A(M_k)\cup\{x,y\}$ for a source
atom $x$ and a target atom $y$, both previously inactive.
\end{theorem}

Nestedness means the profile can be built incrementally: never revisit a decision, only
extend.  The second structural fact says where to look for the extension.

\begin{theorem}[Cyclic-neighbour property; proved in \cref{app:nested}]
\label{thm:neighbour-main}
In such an extension, $x$ and $y$ are consecutive in the cyclic ordering of the
inactive set $A_k^{c}$.
\end{theorem}

Deleting the active atoms therefore leaves a circle of inactive atoms, and only the
$O(N)$ adjacent pairs of \emph{opposite type} are ever candidates.  We call the arc
between two consecutive inactive atoms a \emph{current cell}, a cell whose endpoints
have opposite types a \emph{candidate}, and write $\mathfrak m(\mathcal C)$ for its
\emph{local marginal}: the increase in the cell's internal line-transport cost caused by
activating its two endpoints (\cref{def:cell-marginal} in \cref{app:cells}).  \Cref{fig:cell-locality} is
the whole mechanism --- activating a candidate pair rematches the atoms inside its cell
and nothing else, so the increment is computable from that cell alone.

\begin{figure}[t]
\centering
\begin{minipage}[t]{0.285\linewidth}
\centering
\textbf{(a) Candidate cell}\par\medskip
\resizebox{\linewidth}{!}{%
\begin{tikzpicture}[
 scale=.85,font=\small,
 srcpt/.style={circle,draw=src,fill=src,inner sep=1.8pt},
 srcinactive/.style={circle,draw=src,fill=white,inner sep=2.5pt,line width=1pt},
 tgtpt/.style={circle,draw=tgt,fill=tgt,inner sep=1.8pt},
 tgtinactive/.style={circle,draw=tgt,fill=white,inner sep=2.5pt,line width=1pt}
]
\def\R{1.55}
\draw[line width=.8pt] (0,0) circle (\R);
\draw[candidate!70!black,line width=1.5pt] (130:\R) arc[start angle=130,end angle=410,radius=\R];
\node[srcinactive,label={[src]140:$u$}] at (130:\R) {};
\node[tgtinactive,label={[tgt]45:$v$}] at (50:\R) {};
\node[tgtpt,label={[tgt]190:$y_1$}] at (185:\R) {};
\node[srcpt,label={[src]225:$x_1$}] at (225:\R) {};
\node[tgtpt,label={[tgt]280:$y_2$}] at (280:\R) {};
\node[srcpt,label={[src]340:$x_2$}] at (340:\R) {};
\node[candidate!50!black] at (0,0) {$\mathcal C=[u,v]_{\circlearrowright}$};
\end{tikzpicture}}
\end{minipage}\hfill
\begin{minipage}[t]{0.31\linewidth}
\centering
\textbf{(b) Before activation}\par\medskip
\begin{tikzpicture}[
 x=1cm,y=1cm,font=\small,
 srcpt/.style={circle,draw=src,fill=src,inner sep=1.8pt},
 tgtpt/.style={circle,draw=tgt,fill=tgt,inner sep=1.8pt},
 oldmatch/.style={draw=black!55,line width=.9pt}
]
\draw[line width=.7pt] (0,0)--(4.2,0);
\node[tgtpt,label={[tgt]below:$y_1$}] (y1) at (.6,0) {};
\node[srcpt,label={[src]below:$x_1$}] (x1) at (1.6,0) {};
\node[tgtpt,label={[tgt]below:$y_2$}] (y2) at (2.6,0) {};
\node[srcpt,label={[src]below:$x_2$}] (x2) at (3.6,0) {};
\draw[oldmatch] (x1) to[bend right=40] (y1);
\draw[oldmatch] (x2) to[bend right=40] (y2);
\draw[decorate,decoration={brace,amplitude=5pt}] (3.9,-.62)--(.3,-.62) node[midway,below=6pt] {$c_{\mathbb R}(\mathcal C^\circ)$};
\end{tikzpicture}
\end{minipage}\hfill
\begin{minipage}[t]{0.365\linewidth}
\centering
\textbf{(c) After activating $u,v$}\par\medskip
\begin{tikzpicture}[
 x=.88cm,y=1cm,font=\small,
 srcpt/.style={circle,draw=src,fill=src,inner sep=1.8pt},
 tgtpt/.style={circle,draw=tgt,fill=tgt,inner sep=1.8pt},
 oldmatch/.style={draw=black!35,densely dashed,line width=.8pt},
 newmatch/.style={draw=magenta!75!black,line width=1.1pt,-{Stealth[length=1.7mm]}}
]
\draw[line width=.7pt] (0,0)--(5.5,0);
\node[srcpt,label={[src]below:$u$}] (u) at (.2,0) {};
\node[tgtpt,label={[tgt]below:$y_1$}] (y1) at (1.2,0) {};
\node[srcpt,label={[src]below:$x_1$}] (x1) at (2.2,0) {};
\node[tgtpt,label={[tgt]below:$y_2$}] (y2) at (3.2,0) {};
\node[srcpt,label={[src]below:$x_2$}] (x2) at (4.2,0) {};
\node[tgtpt,label={[tgt]below:$v$}] (v) at (5.2,0) {};
\draw[oldmatch] (x1) to[bend right=24] (y1);
\draw[oldmatch] (x2) to[bend right=24] (y2);
\draw[newmatch] (u) to[bend left=42] (y1);
\draw[newmatch] (x1) to[bend left=42] (y2);
\draw[newmatch] (x2) to[bend left=42] (v);
\draw[decorate,decoration={brace,amplitude=5pt}] (5.45,-.62)--(.0,-.62) node[midway,below=6pt] {$c_{\mathbb R}(\mathcal C)$};
\end{tikzpicture}
\[
\mathfrak m(\mathcal C)=c_{\mathbb R}(\mathcal C)-c_{\mathbb R}(\mathcal C^\circ).
\]
\end{minipage}
\caption{A candidate cell and its exact local cost increment.  The inactive endpoints
$u$ and $v$ are cyclic neighbours of opposite type and the open interior
$\mathcal C^{\circ}$ is active and balanced.  Before activation the interior atoms are
matched in sorted order; activating $u,v$ revokes every match inside the cell (grey
dashed) and replaces it by the new sorted matching (magenta), while every match outside
the cell is untouched.  The increment is therefore
$\mathfrak m(\mathcal C)=c_{\mathbb R}(\mathcal C)-c_{\mathbb R}(\mathcal C^{\circ})$,
computable from the atoms of the cell alone --- the circular analogue of \PAWL{}'s
chain update \citep{chapel2025one}.}
\label{fig:cell-locality}
\end{figure}
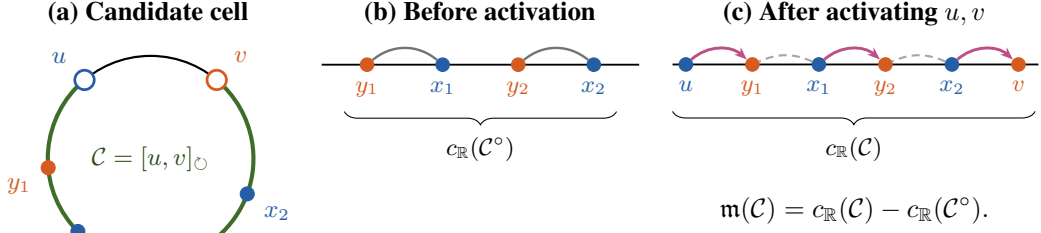

\subsection{The free-gap invariant: one cut for all $K+1$ problems}

The difficulty specific to the circle is that a local update is only meaningful
relative to a cut, and cells are updated at different times.  The following invariant
is what makes one cut serve all of them.  Call an original support gap \emph{free} if
it lies inside no cell that has been selected so far.

\begin{lemma}[Free-gap invariant; proved in \cref{app:greedy}]
\label{lem:freegap-main}
Before termination, every current cell contains at least one free original gap.
\end{lemma}

\begin{figure}[t]
\centering
\begin{minipage}[t]{0.48\linewidth}
\centering
\textbf{(a) Before selecting $[u,v]_{\circlearrowright}$}\par\medskip
\resizebox{\linewidth}{!}{%
\begin{tikzpicture}[
 x=1cm,y=1cm,font=\small,
 srcpt/.style={circle,draw=src,fill=src,inner sep=1.6pt},
 srcinactive/.style={circle,draw=src,fill=white,inner sep=2.5pt,line width=1pt},
 tgtpt/.style={circle,draw=tgt,fill=tgt,inner sep=1.6pt},
 tgtinactive/.style={circle,draw=tgt,fill=white,inner sep=2.5pt,line width=1pt},
 free/.style={draw=freegap,line width=3.2pt,line cap=round},
 rep/.style={draw=freegap!60!black,line width=.6pt,rounded corners=1.5pt}
]
\begin{scope}
\clip (2.45,.82) rectangle (4.75,2.0);
\fill[candidate!18] (0,1.2) .. controls (2.4,1.9) and (4.8,1.9) .. (7.2,1.2) -- (7.2,.82) -- (0,.82) -- cycle;
\end{scope}
\draw[line width=.9pt] (0,1.2) .. controls (2.4,1.9) and (4.8,1.9) .. (7.2,1.2);
\begin{scope}
\clip (2.45,.5) rectangle (4.75,2.0);
\draw[candidate!70!black,line width=1.5pt] (0,1.2) .. controls (2.4,1.9) and (4.8,1.9) .. (7.2,1.2);
\end{scope}
\node[tgtinactive,label={[tgt]above:$p$}] at (.5,1.36) {};
\node[srcinactive,label={[src]above:$u$}] at (2.45,1.66) {};
\node[tgtinactive,label={[tgt]above:$v$}] at (4.75,1.66) {};
\node[srcinactive,label={[src]above:$q$}] at (6.7,1.36) {};
\node[srcpt] at (1.15,1.52) {}; \node[tgtpt] at (1.65,1.60) {};
\node[tgtpt] at (3.05,1.74) {}; \node[srcpt] at (4.05,1.74) {};
\node[srcpt] at (5.55,1.58) {}; \node[tgtpt] at (6.05,1.48) {};
\draw[decorate,decoration={brace,amplitude=5pt}] (.5,.78)--(2.45,.78) node[midway,below=6pt] {$[p,u]_{\circlearrowright}$};
\draw[decorate,decoration={brace,amplitude=5pt},draw=candidate!70!black] (2.45,.78)--(4.75,.78) node[midway,below=6pt,text=candidate!50!black] {$[u,v]_{\circlearrowright}$};
\draw[decorate,decoration={brace,amplitude=5pt}] (4.75,.78)--(6.7,.78) node[midway,below=6pt] {$[v,q]_{\circlearrowright}$};
\draw[free] (1.3,-.12)--(1.75,-.12);  \draw[rep] (1.17,-.27) rectangle (1.88,.03);
\draw[free] (3.32,-.12)--(3.77,-.12); \draw[rep] (3.19,-.27) rectangle (3.9,.03);
\draw[free] (5.5,-.12)--(5.95,-.12);  \draw[rep] (5.37,-.27) rectangle (6.08,.03);
\node[freegap!60!black,font=\scriptsize] at (1.52,-.5) {rep.};
\node[freegap!60!black,font=\scriptsize] at (3.55,-.5) {rep.};
\node[freegap!60!black,font=\scriptsize] at (5.72,-.5) {rep.};
\end{tikzpicture}}
\par\scriptsize Every current cell has a balanced active interior and stores one representative free gap (boxed).  Shaded: the selected candidate cell.
\end{minipage}\hfill
\begin{minipage}[t]{0.48\linewidth}
\centering
\textbf{(b) Activate $u,v$ and merge the three cells}\par\medskip
\resizebox{\linewidth}{!}{%
\begin{tikzpicture}[
 x=1cm,y=1cm,font=\small,
 srcpt/.style={circle,draw=src,fill=src,inner sep=1.6pt},
 srcinactive/.style={circle,draw=src,fill=white,inner sep=2.5pt,line width=1pt},
 tgtpt/.style={circle,draw=tgt,fill=tgt,inner sep=1.6pt},
 tgtinactive/.style={circle,draw=tgt,fill=white,inner sep=2.5pt,line width=1pt},
 free/.style={draw=freegap,line width=3.2pt,line cap=round},
 used/.style={draw=black!35,line width=3.2pt,line cap=round},
 rep/.style={draw=freegap!60!black,line width=.6pt,rounded corners=1.5pt}
]
\begin{scope}
\clip (.5,.82) rectangle (6.7,2.0);
\fill[candidate!18] (0,1.2) .. controls (2.4,1.9) and (4.8,1.9) .. (7.2,1.2) -- (7.2,.82) -- (0,.82) -- cycle;
\end{scope}
\draw[line width=.9pt] (0,1.2) .. controls (2.4,1.9) and (4.8,1.9) .. (7.2,1.2);
\begin{scope}
\clip (.5,.5) rectangle (6.7,2.0);
\draw[candidate!70!black,line width=1.5pt] (0,1.2) .. controls (2.4,1.9) and (4.8,1.9) .. (7.2,1.2);
\end{scope}
\node[tgtinactive,label={[tgt]above:$p$}] at (.5,1.36) {};
\node[srcpt,label={[src]above:$u$}] at (2.45,1.66) {};
\node[tgtpt,label={[tgt]above:$v$}] at (4.75,1.66) {};
\node[srcinactive,label={[src]above:$q$}] at (6.7,1.36) {};
\node[srcpt] at (1.15,1.52) {}; \node[tgtpt] at (1.65,1.60) {};
\node[tgtpt] at (3.05,1.74) {}; \node[srcpt] at (4.05,1.74) {};
\node[srcpt] at (5.55,1.58) {}; \node[tgtpt] at (6.05,1.48) {};
\draw[decorate,decoration={brace,amplitude=5pt},draw=candidate!70!black] (.5,.78)--(6.7,.78) node[midway,below=6pt,text=candidate!50!black] (pq) {$[p,q]_{\circlearrowright}$};
\draw[free] (1.3,-.12)--(1.75,-.12);  \draw[rep] (1.17,-.27) rectangle (1.88,.03);
\node[freegap!60!black,font=\scriptsize] at (1.52,-.5) {free (stored)};
\draw[-{Stealth[length=1.6mm]},freegap!60!black,line width=.6pt] (pq.west) to[bend right=18] (1.95,.05);
\draw[used] (3.32,-.12)--(3.77,-.12); \node[black!55,font=\scriptsize] at (3.55,-.5) {used};
\draw[free] (5.5,-.12)--(5.95,-.12);  \node[freegap,font=\scriptsize] at (5.72,-.5) {free (not stored)};
\draw[black!45,densely dashed] (2.45,1.55)--(2.45,.85);
\draw[black!45,densely dashed] (4.75,1.55)--(4.75,.85);
\end{tikzpicture}}
\par\scriptsize $u,v$ become active interior atoms and leave the inactive list.  Only the gaps \emph{inside} the selected cell become used.  Both flanking arcs keep their free gaps; the merged cell adopts one of them as its stored representative (here the left) and simply does not track the other.  Since $p,q$ have opposite types (as here), the new cell enters the heap.
\end{minipage}
\caption{Merging two cells after an activation.  The new cell inherits a free gap from
one of its parents, so \cref{lem:freegap-main} is maintained at $O(1)$ cost, and that
gap is a cut at which every update made so far is a valid line update.}
\vspace{-.2in}
\label{fig:merge-freegap}
\end{figure}
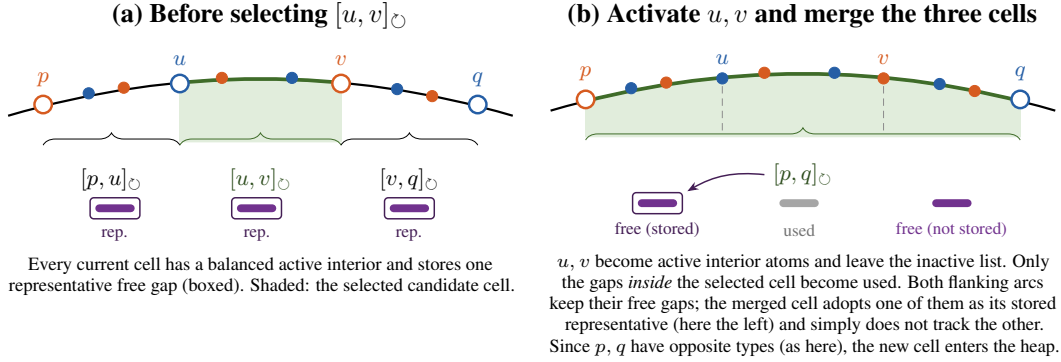

A free gap in a cell is a cut outside every previously selected cell, so unwrapping
there leaves all earlier local computations valid; and when two cells merge, the new
cell inherits a free gap from a parent (\cref{fig:merge-freegap}), so the invariant is
maintained in constant time.  Combined with a pairwise common-cut lemma
(\cref{lem:common-cut} in \cref{app:nested}), which produces one cut valid for two
competing cells at once,
this yields the greedy rule.

\begin{theorem}[Exact greedy choice; proved in \cref{app:greedy}]
\label{thm:greedy-main}
If the current active set $A_k$ is circularly optimal and was produced by successive
cell selections, then
$C_{k+1}^{\circ}-C_k^{\circ}=\min_{\mathcal C\in\mathfrak C_k}\mathfrak m(\mathcal C)$,
where $\mathfrak C_k$ is the set of current candidate cells, and activating the
endpoints of any minimiser yields a circularly optimal $A_{k+1}$.
\end{theorem}

\begin{theorem}[Simultaneous optimal cut; proved in \cref{app:greedy}]
\label{thm:cut-main}
There is an original support gap $\theta^{\star}$ with
$C_{k,\theta^{\star}}^{\mathrm{line}}=C_k^{\circ}$ for every $k=0,\dots,K$, and the
greedy construction produces one.
\end{theorem}

\cref{thm:cut-main} is what removes the outer loop over cuts (\cref{fig:simultaneous-cut} in \cref{app:greedy}
illustrates it): rather than search $N$
cuts and take an envelope, the algorithm \emph{constructs} a gap that is optimal for
every cardinality, as a by-product of the free-gap bookkeeping it already does.  Once
$\theta^\star$ is known, every plan is recovered by sorting the active atoms of each
type in the unwrapped order and pairing them (\cref{app:plans}).  The theorem is not a
convenience: cutting once at a gap that is optimal for the full-mass problem and reusing
it is strictly suboptimal at some cardinality on $13.5\%$ of $5600$ exact instances, and
when it fails the relative excess reaches $1.8\times10^{3}$; other one-cut rules also
fail (\cref{sec:numerics-cut}). 

\section{The \texorpdfstring{\PAWC{}}{PAWC} algorithm: $O(1)$ marginals, $O(N\log N)$ overall}
\label{sec:algorithm}

It remains to evaluate candidate marginals cheaply.  Following \PAWL{}'s chain
machinery, we duplicate the cyclic support into a doubled sequence of length $2N$ so
that every arc becomes an interval, and precompute prefix arrays of differential ranks
and signed positions.  A balanced interval's minimal-chain cost is then a single table
lookup, and the marginal of a candidate cell is calculated via:

\begin{corollary}[Constant-time candidate marginal; proved in \cref{app:constant}]
\label{cor:marginal-main}
For a candidate cell represented by the balanced doubled interval $[a,b]$,
$\mathfrak m([a,b])=(Q_b-Q_{a-1})-(Q_{b-1}-Q_a)$, with the second term zero when the
interior is empty.  Hence each candidate marginal costs $O(1)$ after preprocessing.
\end{corollary}

\begin{algorithm}[tb]
\caption{\PAWC{}: all exact partial $W_1$ costs on the circle}
\label{alg:pawc-main}
\begin{algorithmic}[1]
\Require Circular source points $x_1,\dots,x_n$, target points $y_1,\dots,y_m$, weight $w$, circumference $L$
\Ensure Costs $C_0^{\circ},\dots,C_K^{\circ}$, activation ranks, simultaneous cut $\theta^\star$
\State Sort the union support cyclically and form labels $\sigma_1,\dots,\sigma_N$
\State Build the doubled sequence and precompute the prefix tables $R,S,p,Q$ (\cref{app:constant})
\State Initialize a circular doubly linked list containing all support indices
\State For each initial directed cell, store its unique original gap as its free-gap representative
\State Insert every opposite-endpoint initial cell into a min-heap keyed by $\mathfrak m$ (\cref{cor:marginal-main})
\State $C_0^{\circ}\gets 0$
\For{$k=0,1,\dots,K-1$}
    \Repeat
        \State Pop the minimum heap entry $(u,v,\mathfrak m)$
    \Until{$u,v$ are inactive, $v=\operatorname{succ}(u)$, and $u,v$ have opposite types}
    \State Record activation rank $k+1$ for $u$ and $v$
    \State $C_{k+1}^{\circ}\gets C_k^{\circ}+\mathfrak m$
    \State Let $p\gets\operatorname{pred}(u)$ and $q\gets\operatorname{succ}(v)$ before deletion
    \If{only $u,v$ are inactive}
        \State $\theta^\star\gets$ free-gap representative of the complementary cell $\arc{v}{u}$
        \State Remove $u,v$ and \textbf{continue}
    \EndIf
    \State Remove $u,v$ and link $p$ directly to $q$
    \State Let the new cell $\arc{p}{q}$ inherit a free-gap representative from $\arc{p}{u}$ or $\arc{v}{q}$
    \If{$p$ and $q$ have opposite types}
        \State Compute $\mathfrak m(\arc{p}{q})$ in $O(1)$ and insert the new candidate into the heap
    \EndIf
\EndFor
\If{$\theta^\star$ has not been set}
    \State Choose it as the free-gap representative of any remaining current cell
\EndIf
\State \Return $\{C_k^{\circ}\}_{k=0}^{K}$, activation ranks, $\theta^\star$
\end{algorithmic}
\end{algorithm}

\begin{theorem}[Correctness and complexity; proved in \cref{app:complexity}]
\label{thm:complexity-main}
Under \cref{ass:standing-main}, \cref{alg:pawc-main} returns the exact $C_k^{\circ}$ and
an optimal active set for every $k$, together with a gap $\theta^{\star}$ valid for all
$k$, in $O(N\log N)$ time and $O(N)$ memory.
\end{theorem}

The loop is a textbook greedy with lazy deletion: at most $N$ initial candidates and at
most one new candidate per activation, so $O(N)$ heap entries in total, each popped
once, valid or stale, at $O(\log N)$; everything else --- key evaluation, list surgery,
free-gap inheritance --- is $O(1)$.  Sorting dominates.

\section{Slicing to the sphere}
\label{sec:sphere}

Directional data in $d$ dimensions lives on $\Sph$, where exact OT is expensive.  Slicing
replaces it by an average over one-dimensional projections; on the sphere the natural
projections are onto great circles \citep{bonet2023spherical}.  For $U$ in the Stiefel
manifold $V_{d,2}$, the geodesic projection $P^{U}(x)=UU^{\top}x/\|U^{\top}x\|$ maps
$\Sph$ onto the great circle spanned by $U$, and reading the result in the circle's
own coordinate turns each slice into an instance of \eqref{eq:profile} with $L=2\pi$.

\begin{figure}[tb]
\centering
\input{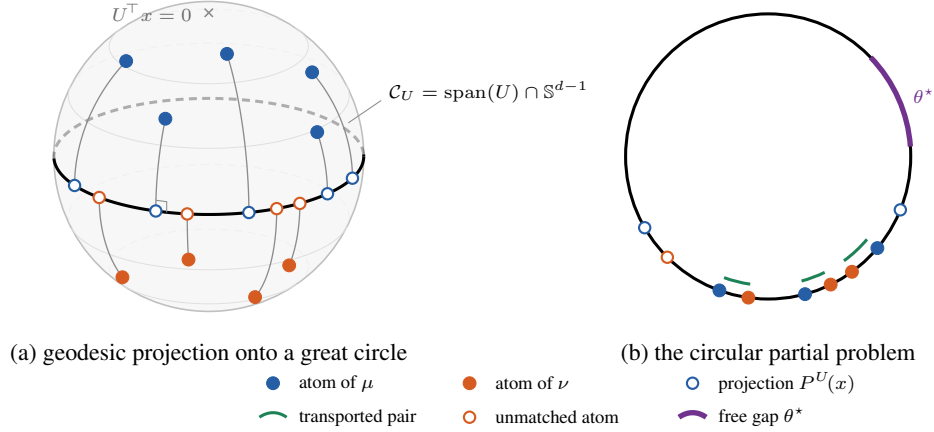}
\caption{Spherical slicing, after \citet[Figure~1]{bonet2023spherical}.  (a) atoms of
$\mu$ (blue) and $\nu$ (orange) on $\Sph$, the great circle $\mathcal C_U$ (dashed
behind the sphere), and the geodesic projection of each onto it.  (b) the circular
problem that slice induces, with a cardinality-$3$ matching and a free gap
$\theta^{\star}$.  \PAWC{} returns the optimal matching for \emph{every} $k$ in one
$O(N\log N)$ sweep per slice.}
\vspace{-.2in}
\label{fig:spherical-slicing}
\end{figure}

\begin{definition}[Sliced spherical partial Wasserstein; \cref{app:sphere}]
\label{def:sspw-main}
$\mathrm{SSPW}(s;\mu,\nu)=\int_{V_{d,2}}\PW_{\circ}(s;P^{U}_{\#}\mu,P^{U}_{\#}\nu)\,
d\sigma(U)$, estimated by the average over $M$ i.i.d.\ slices.
\end{definition}

\Cref{fig:spherical-slicing} shows one slice.  Because \PAWC{} returns the whole profile
per slice, one pass over $M$ slices gives the estimate at \emph{every} transported mass, at cost $O(M(N\log N + dN))$.  The estimator
concentrates uniformly over the profile: $M\ge(\pi^{2}K^{2}w^{2}/2\varepsilon^{2})
\log(2(K+1)/\delta)$ slices suffice for $\varepsilon$-accuracy on the entire curve with
probability $1-\delta$ (\cref{prop:sspw-concentration} in \cref{app:sphere}), and at $s=Kw$ the construction
collapses to spherical sliced Wasserstein, so the partial version is a strict
generalisation of the balanced one.

\section{Experiments}
\label{sec:experiments}

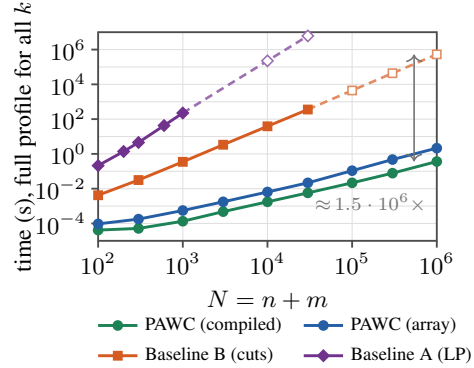
\begin{wrapfigure}[11]{r}{0.48\textwidth}
\vspace{-5.7\baselineskip}
\centering
\centering
\begin{tikzpicture}[
  font=\small,
  every node/.style={inner sep=1pt},
  ax/.style={draw=black!70,line width=.7pt},
  grid/.style={draw=black!12,line width=.4pt},
  meas/.style={line width=1.1pt},
  extrap/.style={line width=1.0pt,densely dashed}
]
\draw[grid] (0.000,0) -- (0.000,2.760);
\draw[grid] (1.120,0) -- (1.120,2.760);
\draw[grid] (2.240,0) -- (2.240,2.760);
\draw[grid] (3.360,0) -- (3.360,2.760);
\draw[grid] (4.480,0) -- (4.480,2.760);
\draw[grid] (0,0.000) -- (4.480,0.000);
\draw[grid] (0,0.230) -- (4.480,0.230);
\draw[grid] (0,0.460) -- (4.480,0.460);
\draw[grid] (0,0.690) -- (4.480,0.690);
\draw[grid] (0,0.920) -- (4.480,0.920);
\draw[grid] (0,1.150) -- (4.480,1.150);
\draw[grid] (0,1.380) -- (4.480,1.380);
\draw[grid] (0,1.610) -- (4.480,1.610);
\draw[grid] (0,1.840) -- (4.480,1.840);
\draw[grid] (0,2.070) -- (4.480,2.070);
\draw[grid] (0,2.300) -- (4.480,2.300);
\draw[grid] (0,2.530) -- (4.480,2.530);
\draw[grid] (0,2.760) -- (4.480,2.760);
\draw[ax] (0,0) rectangle (4.480,2.760);
\draw[ax] (0.000,0) -- (0.000,-.09);
\node[below] at (0.000,-.12) {$10^{2}$};
\draw[ax] (1.120,0) -- (1.120,-.09);
\node[below] at (1.120,-.12) {$10^{3}$};
\draw[ax] (2.240,0) -- (2.240,-.09);
\node[below] at (2.240,-.12) {$10^{4}$};
\draw[ax] (3.360,0) -- (3.360,-.09);
\node[below] at (3.360,-.12) {$10^{5}$};
\draw[ax] (4.480,0) -- (4.480,-.09);
\node[below] at (4.480,-.12) {$10^{6}$};
\draw[ax] (0,0.230) -- (-.09,0.230);
\node[left] at (-.12,0.230) {$10^{-4}$};
\draw[ax] (0,0.690) -- (-.09,0.690);
\node[left] at (-.12,0.690) {$10^{-2}$};
\draw[ax] (0,1.150) -- (-.09,1.150);
\node[left] at (-.12,1.150) {$10^{0}$};
\draw[ax] (0,1.610) -- (-.09,1.610);
\node[left] at (-.12,1.610) {$10^{2}$};
\draw[ax] (0,2.070) -- (-.09,2.070);
\node[left] at (-.12,2.070) {$10^{4}$};
\draw[ax] (0,2.530) -- (-.09,2.530);
\node[left] at (-.12,2.530) {$10^{6}$};
\node[below] at (2.240,-.62) {$N=n+m$};
\node[rotate=90,above] at (-0.78,1.380) {time (s), full profile for all $k$};
\draw[meas,transportcol] plot coordinates {(0.000,0.142) (0.534,0.164) (1.120,0.259) (1.654,0.386) (2.240,0.516) (2.774,0.634) (3.360,0.767) (3.894,0.895) (4.480,1.049)};
\filldraw[draw=transportcol,fill=transportcol,line width=.6pt] (0.000,0.142) circle (1.7pt);
\filldraw[draw=transportcol,fill=transportcol,line width=.6pt] (0.534,0.164) circle (1.7pt);
\filldraw[draw=transportcol,fill=transportcol,line width=.6pt] (1.120,0.259) circle (1.7pt);
\filldraw[draw=transportcol,fill=transportcol,line width=.6pt] (1.654,0.386) circle (1.7pt);
\filldraw[draw=transportcol,fill=transportcol,line width=.6pt] (2.240,0.516) circle (1.7pt);
\filldraw[draw=transportcol,fill=transportcol,line width=.6pt] (2.774,0.634) circle (1.7pt);
\filldraw[draw=transportcol,fill=transportcol,line width=.6pt] (3.360,0.767) circle (1.7pt);
\filldraw[draw=transportcol,fill=transportcol,line width=.6pt] (3.894,0.895) circle (1.7pt);
\filldraw[draw=transportcol,fill=transportcol,line width=.6pt] (4.480,1.049) circle (1.7pt);
\draw[meas,src] plot coordinates {(0.000,0.224) (0.534,0.287) (1.120,0.401) (1.654,0.516) (2.240,0.646) (2.774,0.767) (3.360,0.928) (3.894,1.075) (4.480,1.226)};
\filldraw[draw=src,fill=src,line width=.6pt] (0.000,0.224) circle (1.7pt);
\filldraw[draw=src,fill=src,line width=.6pt] (0.534,0.287) circle (1.7pt);
\filldraw[draw=src,fill=src,line width=.6pt] (1.120,0.401) circle (1.7pt);
\filldraw[draw=src,fill=src,line width=.6pt] (1.654,0.516) circle (1.7pt);
\filldraw[draw=src,fill=src,line width=.6pt] (2.240,0.646) circle (1.7pt);
\filldraw[draw=src,fill=src,line width=.6pt] (2.774,0.767) circle (1.7pt);
\filldraw[draw=src,fill=src,line width=.6pt] (3.360,0.928) circle (1.7pt);
\filldraw[draw=src,fill=src,line width=.6pt] (3.894,1.075) circle (1.7pt);
\filldraw[draw=src,fill=src,line width=.6pt] (4.480,1.226) circle (1.7pt);
\draw[meas,tgt] plot coordinates {(0.000,0.603) (0.534,0.804) (1.120,1.045) (1.654,1.271) (2.240,1.515) (2.774,1.737)};
\draw[extrap,tgt!70] plot coordinates {(2.774,1.737) (3.360,1.989) (3.894,2.217) (4.480,2.467)};
\filldraw[draw=tgt!70,fill=white,line width=.6pt] (3.360,1.989) +(-1.55pt,-1.55pt) rectangle +(1.55pt,1.55pt);
\filldraw[draw=tgt!70,fill=white,line width=.6pt] (3.894,2.217) +(-1.55pt,-1.55pt) rectangle +(1.55pt,1.55pt);
\filldraw[draw=tgt!70,fill=white,line width=.6pt] (4.480,2.467) +(-1.55pt,-1.55pt) rectangle +(1.55pt,1.55pt);
\filldraw[draw=tgt,fill=tgt,line width=.6pt] (0.000,0.603) +(-1.55pt,-1.55pt) rectangle +(1.55pt,1.55pt);
\filldraw[draw=tgt,fill=tgt,line width=.6pt] (0.534,0.804) +(-1.55pt,-1.55pt) rectangle +(1.55pt,1.55pt);
\filldraw[draw=tgt,fill=tgt,line width=.6pt] (1.120,1.045) +(-1.55pt,-1.55pt) rectangle +(1.55pt,1.55pt);
\filldraw[draw=tgt,fill=tgt,line width=.6pt] (1.654,1.271) +(-1.55pt,-1.55pt) rectangle +(1.55pt,1.55pt);
\filldraw[draw=tgt,fill=tgt,line width=.6pt] (2.240,1.515) +(-1.55pt,-1.55pt) rectangle +(1.55pt,1.55pt);
\filldraw[draw=tgt,fill=tgt,line width=.6pt] (2.774,1.737) +(-1.55pt,-1.55pt) rectangle +(1.55pt,1.55pt);
\draw[meas,freegap] plot coordinates {(0.000,0.995) (0.337,1.182) (0.534,1.301) (0.872,1.525) (1.120,1.691)};
\draw[extrap,freegap!70] plot coordinates {(1.120,1.691) (2.240,2.381) (2.774,2.711)};
\filldraw[draw=freegap!70,fill=white,line width=.6pt] (2.240,2.381) +(0,2.1pt) -- +(2.1pt,0) -- +(0,-2.1pt) -- +(-2.1pt,0) -- cycle;
\filldraw[draw=freegap!70,fill=white,line width=.6pt] (2.774,2.711) +(0,2.1pt) -- +(2.1pt,0) -- +(0,-2.1pt) -- +(-2.1pt,0) -- cycle;
\filldraw[draw=freegap,fill=freegap,line width=.6pt] (0.000,0.995) +(0,2.1pt) -- +(2.1pt,0) -- +(0,-2.1pt) -- +(-2.1pt,0) -- cycle;
\filldraw[draw=freegap,fill=freegap,line width=.6pt] (0.337,1.182) +(0,2.1pt) -- +(2.1pt,0) -- +(0,-2.1pt) -- +(-2.1pt,0) -- cycle;
\filldraw[draw=freegap,fill=freegap,line width=.6pt] (0.534,1.301) +(0,2.1pt) -- +(2.1pt,0) -- +(0,-2.1pt) -- +(-2.1pt,0) -- cycle;
\filldraw[draw=freegap,fill=freegap,line width=.6pt] (0.872,1.525) +(0,2.1pt) -- +(2.1pt,0) -- +(0,-2.1pt) -- +(-2.1pt,0) -- cycle;
\filldraw[draw=freegap,fill=freegap,line width=.6pt] (1.120,1.691) +(0,2.1pt) -- +(2.1pt,0) -- +(0,-2.1pt) -- +(-2.1pt,0) -- cycle;
\draw[<->,black!55,line width=.6pt] (4.180,1.049) -- (4.180,2.467);
\node[anchor=east,black!55,font=\scriptsize,align=left] at (4.380,0.499) {$\approx$\,$1.5\cdot 10^{6}\times$};
\draw[meas,transportcol] (0.000,-1.120) -- (0.520,-1.120);
\filldraw[draw=transportcol,fill=transportcol,line width=.6pt] (0.260,-1.120) circle (1.7pt);
\node[right,font=\scriptsize] at (0.580,-1.120) {PAWC (compiled)};
\draw[meas,src] (2.720,-1.120) -- (3.240,-1.120);
\filldraw[draw=src,fill=src,line width=.6pt] (2.980,-1.120) circle (1.7pt);
\node[right,font=\scriptsize] at (3.300,-1.120) {PAWC (array)};
\draw[meas,tgt] (0.000,-1.520) -- (0.520,-1.520);
\filldraw[draw=tgt,fill=tgt,line width=.6pt] (0.260,-1.520) +(-1.55pt,-1.55pt) rectangle +(1.55pt,1.55pt);
\node[right,font=\scriptsize] at (0.580,-1.520) {Baseline B (cuts)};
\draw[meas,freegap] (2.720,-1.520) -- (3.240,-1.520);
\filldraw[draw=freegap,fill=freegap,line width=.6pt] (2.980,-1.520) +(0,2.1pt) -- +(2.1pt,0) -- +(0,-2.1pt) -- +(-2.1pt,0) -- cycle;
\node[right,font=\scriptsize] at (3.300,-1.520) {Baseline A (LP)};
\end{tikzpicture}
\vspace{-.17in}
\setlength{\emergencystretch}{2em}%
\caption{Measured cost of the complete profile.  Baseline B is the cut-envelope
solver, Baseline A the LP oracle; dashed segments are extrapolated ($N^{2}\log N$,
$N^{3}$).  Fitted log--log slopes: $1.03$--$1.29$ for \PAWC{} against $2.00$ for B.}
\label{fig:timings}
\vspace{-1.3\baselineskip}
\end{wrapfigure}

We report that \PAWC{} computes what it claims, at the stated cost, and that the profile
is the object applications want.  Full protocols, seeds, and the negative results we
found along the way are in \cref{app:numerics,app:apps}; code and records accompany the
submission.

\subsection{Exactness and scaling}
\label{sec:exp-exact}

Across $504$ instances spanning nine adversarial families (clustered, near-antipodal,
seam-straddling, tiny-arc, and others) \PAWC{} matches an LP oracle over the
cardinality-$k$ matching polytope to $4.7\times10^{-15}$ and the cut-envelope solver
exactly, and every structural invariant --- nestedness, the free-gap invariant, cell
balance, feasibility, and the validity of $\theta^{\star}$ --- holds at every step of
every instance (\cref{app:numerics}).

\Cref{fig:timings} shows the cost of the entire profile.  The cut-envelope baseline fits
a log--log slope of $2.00$; \PAWC{} fits $1.03$--$1.29$, with the excess over $1$
attributable to cache effects rather than a hidden term: heap operation counts are
exactly linear ($0.81N$ pops, $0.31N$ pushes at $N=10^{4},10^{5},10^{6}$ alike), and the
published line algorithm \PAWL{} measures $1.12$--$1.25$ on the same hardware.  In
absolute terms, one solve returning \emph{all} transported fractions takes $0.09$\,ms at
$N=256$ and $0.56$\,ms at $N=2048$, against $4.4$\,ms and $1.54$\,s for a \emph{single}
fraction from POT's general partial solver \citep{flamary2021pot}, and $30$\,ms and
$11.4$\,s for the seven fractions used in \cref{sec:exp-shapes}.  The sweep over the
transported fraction is free.

\subsection{Partial matching of shape descriptors}
\label{sec:exp-shapes}

The outward normal directions of a closed planar curve, sampled uniformly in arclength,
form a measure on $\Sone$ --- the two-dimensional extended Gaussian image
\citep{horn1984extended}.  It is invariant to translation, scale and starting point and
rotates rigidly with the shape, so comparing shapes up to rotation is alignment on the
circle, and its two failure modes are exactly the ones partial transport addresses:
\emph{occlusion} deletes part of the measure, \emph{clutter} adds mass belonging to
nothing.  It keeps no part structure, so it is far weaker than the descriptors built for
this benchmark, which score $0.75$--$0.85$ bullseye \citep[Table~II]{ling2007shape}; what
follows is not an entry in that benchmark, but a comparison of costs on one fixed
descriptor.

\begin{figure}[t]
\centering
\resizebox{\textwidth}{!}{\input{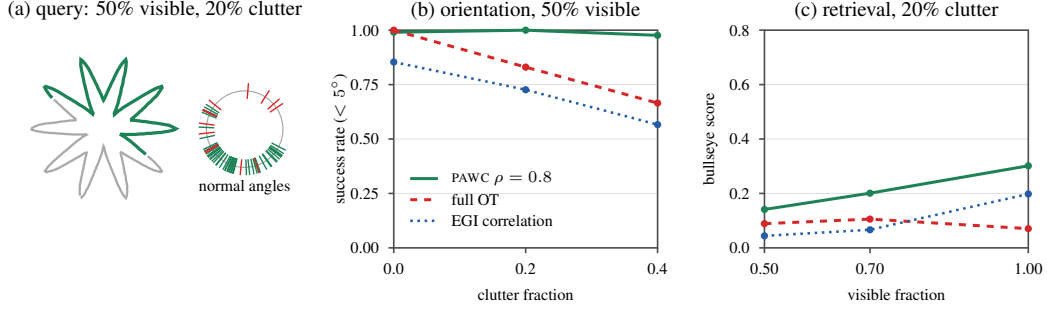}}
\vspace{-0.5\baselineskip}
\caption{Partial matching of shape descriptors on \textsc{mpeg-7}.  (a) one query: the
visible arc of the boundary (green), the occluded remainder (grey), and the
normal-angle measure the matcher sees, with clutter in red.  (b) orientation success
against clutter at half the boundary visible and (c) retrieval bullseye against the
visible fraction at $20\%$ clutter, for \PAWC{} at a fixed $\rho=0.8$, full circular OT,
and the extended-Gaussian-image correlation matcher.  OT fails on clutter;
correlation fails on occlusion.}
\label{fig:shapes}
\end{figure}

\begin{table}[t]
\centering
\caption{Shape matching under occlusion and clutter on all $1400$ \textsc{mpeg-7}
silhouettes: $v$ is the visible fraction of the boundary, $c$ the clutter fraction.
\PAWC{} transports $\rho=1-c$, so the $c=0$ rows \emph{are} full circular OT and tie by
construction.  Orientation and correspondence average $212$ shapes per cell, retrieval
$138$ queries against the full database.  ``corr.'' is the correlation matcher at its
best of three bandwidths; ``NN'' is mutual nearest neighbours in normal angle.}
\label{tab:shapes}
\small
\begin{tabular}{ccccccccccc}
\toprule
 & & \multicolumn{3}{c}{orientation ($<5^{\circ}$)}
   & \multicolumn{3}{c}{retrieval (bullseye)}
   & \multicolumn{3}{c}{correspondence ($F_1$)}\\
\cmidrule(lr){3-5}\cmidrule(lr){6-8}\cmidrule(lr){9-11}
$v$ & $c$ & \PAWC{} & full OT & corr. & \PAWC{} & full OT & corr. & \PAWC{} & full OT & NN\\
\midrule
$1.00$ & $0.0$ & $\mathbf{1.00}$ & $\mathbf{1.00}$ & $\mathbf{1.00}$ & $\mathbf{0.457}$ & $\mathbf{0.457}$ & $0.384$ & $\mathbf{1.00}$ & $\mathbf{1.00}$ & $0.77$ \\
$1.00$ & $0.2$ & $\mathbf{1.00}$ & $0.62$ & $0.98$ & $\mathbf{0.302}$ & $0.071$ & $0.198$ & $\mathbf{0.84}$ & $0.21$ & $0.69$ \\
$1.00$ & $0.4$ & $\mathbf{1.00}$ & $0.45$ & $0.92$ & $\mathbf{0.118}$ & $0.038$ & $0.077$ & $\mathbf{0.66}$ & $0.12$ & $0.53$ \\
$0.70$ & $0.0$ & $\mathbf{1.00}$ & $\mathbf{1.00}$ & $0.96$ & $\mathbf{0.337}$ & $\mathbf{0.337}$ & $0.126$ & $\mathbf{0.94}$ & $\mathbf{0.94}$ & $0.77$ \\
$0.70$ & $0.2$ & $\mathbf{1.00}$ & $0.83$ & $0.90$ & $\mathbf{0.201}$ & $0.106$ & $0.066$ & $\mathbf{0.81}$ & $0.45$ & $0.69$ \\
$0.70$ & $0.4$ & $\mathbf{0.99}$ & $0.62$ & $0.73$ & $\mathbf{0.104}$ & $0.039$ & $0.050$ & $\mathbf{0.68}$ & $0.29$ & $0.56$ \\
$0.50$ & $0.0$ & $\mathbf{1.00}$ & $\mathbf{1.00}$ & $0.85$ & $\mathbf{0.228}$ & $\mathbf{0.228}$ & $0.052$ & $\mathbf{0.91}$ & $\mathbf{0.91}$ & $0.77$ \\
$0.50$ & $0.2$ & $\mathbf{1.00}$ & $0.83$ & $0.73$ & $\mathbf{0.141}$ & $0.089$ & $0.044$ & $\mathbf{0.79}$ & $0.58$ & $0.68$ \\
$0.50$ & $0.4$ & $\mathbf{0.94}$ & $0.67$ & $0.57$ & $\mathbf{0.081}$ & $0.048$ & $0.037$ & $\mathbf{0.69}$ & $0.41$ & $0.57$ \\
\bottomrule
\end{tabular}

\end{table}

\paragraph{Protocol.}
All $1400$ silhouettes of \textsc{mpeg-7} \textsc{ce-shape-1} \citep{latecki2000shape};
contours traced at the $0.5$ level, resampled to $128$ points, converted to normal
angles.  A query keeps a contiguous fraction $v$ of the arclength, replaces a fraction
$c$ of the survivors by uniform normals, and is rotated by a uniformly random angle.
Every method scores the same grid of $360$ rotations; only the cost differs.  \PAWC{}
transports $\rho=1-c$: the mass the task declares genuine.  Three tasks: recover the
rotation (success $<5^{\circ}$); rank the other $1399$ shapes and report the
\textsc{mpeg-7} bullseye score; and score the pairs of the recovered plan against the
arclength indices the query was built from ($F_1$; clutter samples have no correct
partner).

\paragraph{Result.}
\Cref{tab:shapes} separates the two failure modes.  The clean cell fixes the scale:
uncontaminated and fully visible, this descriptor retrieves at $0.457$, and on this
representation no cost can do better.  Occlusion alone costs full OT little --- with
$n\neq m$ the balanced problem already transports only $K=\min(n,m)$ atoms, so size
asymmetry supplies the needed partiality --- but clutter is fatal to it: at full
visibility and $c=0.2$ its bullseye falls to $0.071$, $16\%$ of that ceiling, where
\PAWC{} holds $0.302$, or $66\%$ of it --- a factor of $4.3$ ($+0.231\pm0.053$ paired,
ahead on $95$ of $138$ queries), and at $c=0.4$ its orientation success falls to $0.45$
where \PAWC{} is still at $1.00$.  The correlation matcher fails the other way round,
collapsing from $0.384$ to $0.052$ in bullseye as $v$ falls from $1$ to $0.5$, because a
truncated density correlates best where the template has most mass rather than where the
visible piece belongs.  \PAWC{} at $\rho=1-c$ is at least as good as both baselines in
all nine cells and strictly better in the six with clutter, and a \emph{single} fixed
$\rho=0.8$ used everywhere still beats both wherever there is clutter.  An unbalanced
(KL-relaxed) Sinkhorn cost at its best of three relaxation weights ties \PAWC{} on nearly
complete queries and falls behind as the boundary disappears ($0.71$ against $0.85$ at
$v=0.35$, $c=0.4$), with no single weight competitive across cells and at $20\times$ the
cost (\cref{app:apps}).

\subsection{Robust fitting on \texorpdfstring{$\mathbb{S}^{2}$}{S2}}
\label{sec:exp-sphere}

\begin{figure}[t]
\centering
\input{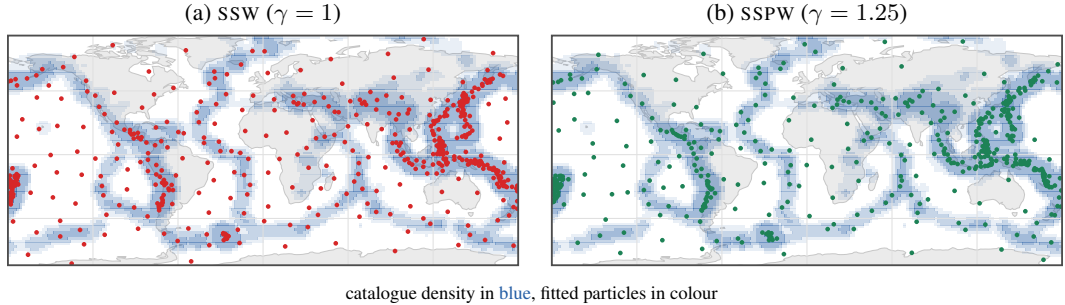}
\caption{Fitting $500$ particles to the \textsc{usgs} earthquake catalogue whose
minibatches are contaminated at $\varepsilon=0.3$, over a Natural Earth coastline.  The blue background is a kernel estimate of the
catalogue's own density --- the target the particles should cover.
\textsc{ssw} ($\gamma=1$) scatters particles across the Pacific and the continental
interiors, where no earthquakes are; the partial fit stays on the ridges and
subduction zones.}
\label{fig:sphere}
\end{figure}

\begin{table}[t]
\centering
\caption{Spherical fitting against a contaminated target, eight seeds per cell.  Energy
distance to clean target samples, $\times10^{-3}$.  \textsc{sspw} transports the rule
fraction $\gamma=1/(1-\varepsilon)$ rounded to the grid, so the $\varepsilon=0$ row
\emph{is} \textsc{ssw}; ``gain'' is the reduction from \textsc{ssw}, ``wins'' counts
seeds ahead, and ``best'' is the gain at the best $\gamma$ of six on the grid, which at
$\varepsilon=0$ is a selection over noise.}
\label{tab:sphere}
\small
\begin{tabular}{ccccccccccc}
\toprule
 & \multicolumn{5}{c}{synthetic vMF mixture}
 & \multicolumn{5}{c}{earthquake catalogue}\\
\cmidrule(lr){2-6}\cmidrule(lr){7-11}
$\varepsilon$ & \textsc{ssw} & \textsc{sspw} & gain & wins & best & \textsc{ssw} & \textsc{sspw} & gain & wins & best\\
\midrule
$0.0$ & $0.50$ & $\mathbf{0.50}$ & $0\%$ & $0/8$ & $14\%$ & $0.79$ & $\mathbf{0.79}$ & $0\%$ & $0/8$ & $0\%$ \\
$0.1$ & $1.97$ & $\mathbf{0.79}$ & $60\%$ & $8/8$ & $70\%$ & $3.23$ & $\mathbf{1.15}$ & $64\%$ & $8/8$ & $77\%$ \\
$0.2$ & $7.13$ & $\mathbf{2.30}$ & $68\%$ & $8/8$ & $68\%$ & $10.25$ & $\mathbf{3.05}$ & $70\%$ & $8/8$ & $70\%$ \\
$0.3$ & $15.92$ & $\mathbf{6.09}$ & $62\%$ & $8/8$ & $63\%$ & $21.86$ & $\mathbf{10.35}$ & $53\%$ & $8/8$ & $53\%$ \\
$0.4$ & $26.40$ & $\mathbf{14.46}$ & $45\%$ & $8/8$ & $50\%$ & $36.40$ & $\mathbf{31.42}$ & $14\%$ & $7/8$ & $40\%$ \\
\bottomrule
\end{tabular}

\end{table}

We use $\mathrm{SSPW}$ as a training loss and compare against the method it generalises,
spherical sliced Wasserstein (\textsc{ssw}), which is its $\rho=1$ case.  One design
point matters: when fitting a model, partiality must sit on the \emph{target} side.  A
trimmed model particle receives no gradient and never moves again --- and those are
precisely the badly placed particles, since they are the expensive ones to match.  We
therefore draw target batches of size $m=\gamma n$, $\gamma\ge1$, and transport
$k=K=n$: every particle moves, and the $1-1/\gamma$ of the target batch the matching
declines to use is what partial transport buys.

$500$ particles from a uniform start, $32$ slices redrawn per step, $2000$ Adam steps,
gradients taken with the matching held fixed (the envelope theorem) and retracted to the
sphere; a fraction $\varepsilon$ of every target batch is replaced by a uniform point.
Quality is the geodesic energy distance to $4000$ \emph{clean} target points, which
shares no implementation with the training loss.  Two targets: a five-component von
Mises--Fisher mixture with unequal masses and concentrations, and $94\,953$ earthquake
epicentres of magnitude $\ge4.5$ from the \textsc{usgs} catalogue, 2010--2024
\citep{usgs_comcat}, held out $80/20$.

\begin{wrapfigure}{r}{0.44\textwidth}
\vspace{-2.0\baselineskip}
\centering
\centering
\tikzset{ax/.style={draw=black!70,line width=.7pt},
  grid/.style={draw=black!12,line width=.4pt},
  land/.style={fill=black!8,draw=black!22,line width=.25pt}}
\begin{tikzpicture}[font=\small]
\draw[grid] (0.000,0.000) -- (4.420,0.000);
\draw[ax] (0.000,0.000) -- (-0.080,0.000);
\node[left,font=\scriptsize] at (-0.100,0.000) {0.0};
\draw[grid] (0.000,0.813) -- (4.420,0.813);
\draw[ax] (0.000,0.813) -- (-0.080,0.813);
\node[left,font=\scriptsize] at (-0.100,0.813) {0.5};
\draw[grid] (0.000,1.627) -- (4.420,1.627);
\draw[ax] (0.000,1.627) -- (-0.080,1.627);
\node[left,font=\scriptsize] at (-0.100,1.627) {1.0};
\draw[grid] (0.000,2.440) -- (4.420,2.440);
\draw[ax] (0.000,2.440) -- (-0.080,2.440);
\node[left,font=\scriptsize] at (-0.100,2.440) {1.5};
\draw[ax] (0.000,0) -- (0.000,-.08);
\node[below,font=\scriptsize] at (0.000,-.10) {1.00};
\draw[ax] (1.423,0) -- (1.423,-.08);
\node[below,font=\scriptsize] at (1.423,-.10) {1.25};
\draw[ax] (2.281,0) -- (2.281,-.08);
\node[below,font=\scriptsize] at (2.281,-.10) {1.43};
\draw[ax] (4.420,0) -- (4.420,-.08);
\node[below,font=\scriptsize] at (4.420,-.10) {2.00};
\draw[ax] (0.000,0) rectangle (4.420,2.440);
\node[below,font=\scriptsize] at (2.210,-.50) {target-batch ratio $\gamma$};
\node[rotate=90,font=\scriptsize] at (-0.820,1.220) {energy distance / \textsc{ssw}};
\node[font=\small] at (2.210,2.740) {synthetic vMF mixture on $\mathbb{S}^{2}$};
\draw[draw=black!25,line width=.5pt,dashed] (0.000,1.627) -- (4.420,1.627);
\draw[draw=black!55,line width=1.0pt] (0.000,1.627) -- (0.608,1.401) -- (1.423,2.440) -- (2.281,2.440) -- (3.270,2.440) -- (4.420,2.440);
\fill[fill=black!55] (0.000,1.627) circle (0.05);
\fill[fill=black!55] (0.608,1.401) circle (0.05);
\fill[fill=black!55] (1.423,2.440) circle (0.05);
\fill[fill=black!55] (2.281,2.440) circle (0.05);
\fill[fill=black!55] (3.270,2.440) circle (0.05);
\fill[fill=black!55] (4.420,2.440) circle (0.05);
\draw[draw=src,line width=1.2pt] (0.000,1.627) -- (0.608,1.019) -- (1.423,0.524) -- (2.281,0.554) -- (3.270,0.835) -- (4.420,0.935);
\fill[fill=src] (0.000,1.627) circle (0.05);
\fill[fill=src] (0.608,1.019) circle (0.05);
\fill[fill=src] (1.423,0.524) circle (0.05);
\fill[fill=src] (2.281,0.554) circle (0.05);
\fill[fill=src] (3.270,0.835) circle (0.05);
\fill[fill=src] (4.420,0.935) circle (0.05);
\draw[draw=popcol,line width=1.3pt] (0.000,1.627) -- (0.608,1.391) -- (1.423,1.041) -- (2.281,0.820) -- (3.270,0.895) -- (4.420,1.264);
\fill[fill=popcol] (0.000,1.627) circle (0.05);
\fill[fill=popcol] (0.608,1.391) circle (0.05);
\fill[fill=popcol] (1.423,1.041) circle (0.05);
\fill[fill=popcol] (2.281,0.820) circle (0.05);
\fill[fill=popcol] (3.270,0.895) circle (0.05);
\fill[fill=popcol] (4.420,1.264) circle (0.05);
\draw[draw=black!55,line width=1.0pt] (2.370,2.140) -- (2.790,2.140);
\node[right,font=\scriptsize] at (2.870,2.140) {$\varepsilon=0.0$};
\draw[draw=src,line width=1.2pt] (2.370,1.820) -- (2.790,1.820);
\node[right,font=\scriptsize] at (2.870,1.820) {$\varepsilon=0.2$};
\draw[draw=popcol,line width=1.3pt] (2.370,1.500) -- (2.790,1.500);
\node[right,font=\scriptsize] at (2.870,1.500) {$\varepsilon=0.4$};
\end{tikzpicture}
\vspace{-.1in}
\setlength{\emergencystretch}{2em}%
\caption{Energy distance against $\gamma$ on the synthetic mixture, each curve divided
by its own $\gamma=1$ value.  The optimum moves right as contamination grows; the clean
curve leaves the panel at once.}
\label{fig:sphere-curve}
\vspace{-0.6\baselineskip}
\end{wrapfigure}
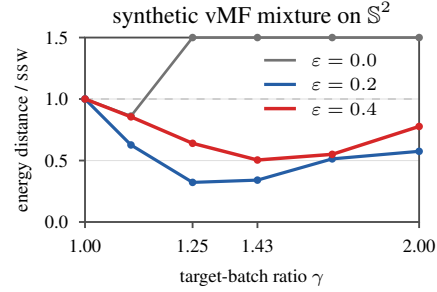

\Cref{fig:sphere-curve} shows the shape of the trade-off: on clean data any trimming
costs accuracy, and as contamination grows the optimum migrates towards
$1/(1-\varepsilon)$ and the curve flattens, so the exact value of $\gamma$ matters less
than being on the right side of $1$.  \Cref{tab:sphere} gives the same picture on both
targets: from $\varepsilon=0.1$ upward the rule
cuts the energy distance to the clean target by $45$--$68$ per cent on the synthetic
mixture and $53$--$70$ per cent on the catalogue, every seed agreeing, while the
$\varepsilon=0$ row --- where \textsc{ssw} is best --- is the control
that licenses reading this as robustness rather than a smaller effective step
(\cref{fig:sphere} shows the difference geographically).  The one cell where the rule
underperforms is the catalogue at $\varepsilon=0.4$: with $40\%$ of a heavily clustered
target replaced by uniform noise, some noise lands on the fault lines and trimming
\emph{less} than the nominal contamination is better.

\section{Discussion}
\label{sec:discussion}

\PAWC{} computes the complete partial-transport profile on the circle in the same
asymptotic cost as sorting, with no search over cuts, and returns a single cut valid for
every cardinality.  Three limitations bound the result.  Equal atomic weights make every
activation add exactly one source and one target; non-uniform masses would require
event-driven handling of saturating capacities.  The constant-time chain oracle relies
on the $W_1$ geodesic cost; $p>1$ and general Monge costs need separate analysis.  Degeneracies --- coincident supports, antipodal ties,
equal candidate marginals --- are handled by deterministic tie-breaking or symbolic
perturbation rather than intrinsically.
On the application side the transported fraction is set by a rule rather than inferred;
\cref{app:apps} reports what we tried and where inference failed.

\subsubsection*{Reproducibility statement}
\cref{app:numerics,app:apps} give the full protocols, seeds, machine and package
versions for every number reported.  The reference and optimised solvers, the LP and
cut-envelope baselines, the invariant checks, all experiment scripts and the raw records
accompany the submission; each figure and table in this paper is generated from those
records by a script, not transcribed.

\subsubsection*{Acknowledgments}
This work was supported by NSF CAREER Award \#2339898 and NSF DMS Award \#2603774.
\newpage
\clearpage
\bibliographystyle{iclr2026_conference}
\bibliography{references}

\newpage 
\clearpage
\appendix
\crefalias{section}{appendix}
\crefalias{subsection}{appendix}
\crefalias{subsubsection}{appendix}

\section*{Appendix overview}

\cref{app:demo} comes first and runs \cref{alg:pawc-main} by hand on one instance,
showing the heap, the linked list and the free-gap bookkeeping at every activation: it is
the concrete account of what the algorithm does, and the mechanisms it exercises are the
ones the rest of the appendix proves correct.
\cref{app:setup,app:background,app:nested,app:cells,app:greedy,app:constant,%
app:complexity,app:plans} then contain the full development and every proof, in the order
in which the main text uses them: notation and the cardinality-$k$ formulation; the line
structure inherited from \PAWL{} and the cut envelope; nested active sets, the common-cut
lemma and the cyclic-neighbour theorem; cells and their marginals; the free-gap
invariant, the exact greedy theorem and the constructive simultaneous cut; the
doubled-sequence preprocessing that makes candidate marginals $O(1)$; correctness and
complexity; and plan recovery with the interpolation for non-integer masses.
\cref{app:sphere} develops the spherical construction and its concentration bound.

\cref{app:numerics} reports the numerical validation --- solvers and trust hierarchy,
exactness against an LP oracle, the structural invariants, the complexity study, the
one-cut experiment of \cref{sec:numerics-cut} and the spherical estimator's behaviour.
\cref{app:apps} gives the protocols for the two applications in the main text, the
unbalanced-transport comparison, the in-situ exactness check, and the application
families we piloted and dropped.



\section{A worked example, step by step}
\label{app:demo}

\label{sec:demo}

This appendix runs \cref{alg:pawc-main} by hand on a single instance, so that every
mechanism the paper introduces can be seen firing on concrete numbers before any of it
is proved.  Notation is the main text's; the statements the trace exercises ---
\cref{lem:free-gap,thm:simultaneous-cut,prop:balanced-query} --- are proved in the
sections that follow it.  The mechanisms are: a
wrap-around selection, a merged-cell selection with matching revocation, lazy heap
skips, free-gap inheritance including the consumption of a selected cell's own
representative, and the two-directed-cell endgame.  A standalone copy is kept at
\texttt{demos/pawc\_demo.tex}.
We run PAWC on the circle of circumference $L=12$ with unit weights $w=1$ and geodesic
cost $\dcirc(u,v)=\min(|u-v|,\,L-|u-v|)$:
\[
\begin{aligned}
x&=\{0.4,\;4.1,\;4.5,\;8.0\} &&\text{(blue circles, sources)},\\
y&=\{1.6,\;5.0,\;5.4,\;11.4\} &&\text{(orange circles, targets)}.
\end{aligned}
\]
The cyclic type pattern $x\,y\,x\,x\,y\,y\,x\,y$ matches the PAWL walkthrough --- but the
wrap-around gap $(y_4,x_1)$ across $0$ is now short ($1.0$), and this changes everything:
the optimal partial plan will \emph{use} the wrap at $k=2,3$ and \emph{abandon} it at
$k=4$. This shifting circulation is precisely why the circle does not reduce to a single
line problem, and what PAWC's machinery (cells, free gaps, the simultaneous cut) is built
to tame. The instance is chosen so that every mechanism fires: a wrap-around selection, a
merged-cell selection with matching revocation, four lazy heap skips, free-gap
inheritance (including consumption of a selected cell's own representative), and the
two-directed-cell endgame.

\subsection*{Step 0 --- Input on $S^1$, and the naive baseline}

Cutting the circle open at any of the $N=8$ inter-atom gaps and unrolling gives a line
instance; the circular optimum is the lower envelope over cuts,
$C^\circ_k=\min_{r} C^{\text{line}}_{k,r}$ (the cut envelope). Running PAWL once per cut
costs $O(N^2\log N)$. PAWC produces the same profile --- and a single cut $\theta^\star$
valid for \emph{all} $k$ at once (\cref{thm:simultaneous-cut}) --- in one $O(N\log N)$ sweep. (Atom
spacing in the circle diagrams is schematic, spread for legibility; the gray coordinates
are the truth.)

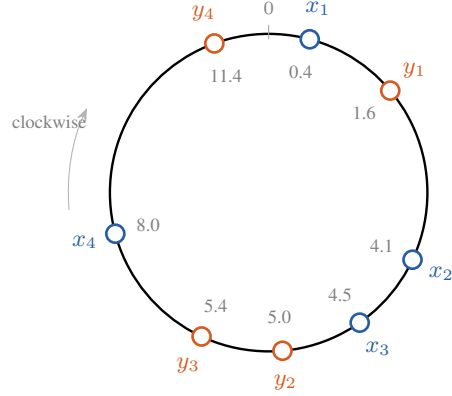
\begin{center}
\begin{tikzpicture}[scale=1.0]
 \basecircle
 \atomn{srcin}{\aXi}{src}{$x_1$}{0.42}
 \atomn{tgtin}{\aYi}{tgt}{$y_1$}{0.42}
 \atomn{srcin}{\aXii}{src}{$x_2$}{0.42}
 \atomn{srcin}{\aXiii}{src}{$x_3$}{0.42}
 \atomn{tgtin}{\aYii}{tgt}{$y_2$}{0.42}
 \atomn{tgtin}{\aYiii}{tgt}{$y_3$}{0.42}
 \atomn{srcin}{\aXiv}{src}{$x_4$}{0.42}
 \atomn{tgtin}{\aYiv}{tgt}{$y_4$}{0.42}
 \foreach \a/\v in {\aXi/0.4,\aYi/1.6,\aXii/4.1,\aXiii/4.5,\aYii/5.0,\aYiii/5.4,\aXiv/8.0,\aYiv/11.4}{
   \node[font=\scriptsize,gray] at (\a:\demoR-0.45) {\v};}
 \draw[-{Stealth[length=1.8mm]},gray!60] ([shift={(185:\demoR+0.55)}]0,0)
   arc[start angle=185,end angle=155,radius=\demoR+0.55];
 \node[gray,font=\scriptsize] at (162:\demoR+0.95) {clockwise};
\end{tikzpicture}
\end{center}

\noindent Atoms are read clockwise from $0$; hollow markers are \emph{inactive} atoms
(none transported yet). Note $y_4$ and $x_1$ sit on opposite sides of $0$, only $1.0$
apart along the wrap.

\subsection*{Step 1 --- Precomputation on the doubled sequence (\cref{sec:constant-time})}

To make every clockwise cell an ordinary interval, the sorted cyclic sequence is doubled:
$\tilde z_{t+8}=\tilde z_t+L$, $t=1,\dots,8$, giving $\tilde z_1,\dots,\tilde z_{16}$. One sweep computes the prefix
\emph{differential ranks} $R_t$ ($+1$ at a source, $-1$ at a target), the signed prefix
sums $S_t$, the minimal-chain predecessors $p_t$ (last previous occurrence of the same
rank), and the maximal-chain prefix costs $Q_t$ --- exactly PAWL's Algorithm~1 run on
the doubled line. The rank walk repeats with period $N$ (total balance $n=m$):

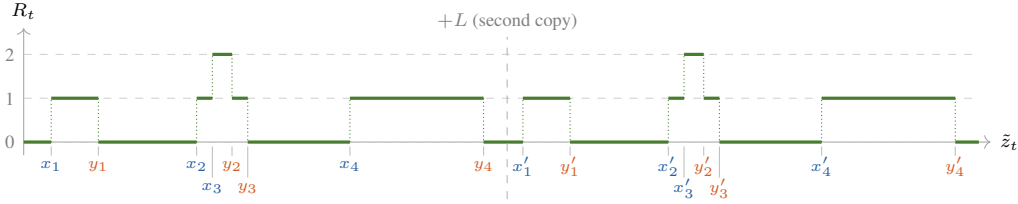
\begin{center}
\begin{tikzpicture}[xscale=0.52,yscale=0.58]
  \draw[->,gray!70] (-0.3,0) -- (24.3,0) node[right,font=\scriptsize,black] {$\tilde z_t$};
  \draw[->,gray!70] (-0.3,-0.3) -- (-0.3,2.6) node[above,font=\scriptsize,black] {$R_t$};
  \foreach \yy/\ll in {0/0,1/1,2/2}{\draw[gray!35,dashed] (-0.3,\yy)--(24.0,\yy);
    \node[left,font=\scriptsize,gray] at (-0.3,\yy) {\ll};}
  \draw[very thick,candidate!80!black]
    (-0.3,0)--(0.4,0) (0.4,1)--(1.6,1) (1.6,0)--(4.1,0) (4.1,1)--(4.5,1)
    (4.5,2)--(5.0,2) (5.0,1)--(5.4,1) (5.4,0)--(8.0,0) (8.0,1)--(11.4,1)
    (11.4,0)--(12.4,0) (12.4,1)--(13.6,1) (13.6,0)--(16.1,0) (16.1,1)--(16.5,1)
    (16.5,2)--(17.0,2) (17.0,1)--(17.4,1) (17.4,0)--(20.0,0) (20.0,1)--(23.4,1)
    (23.4,0)--(24.0,0);
  \foreach \xx/\ya/\yb in {0.4/0/1,1.6/1/0,4.1/0/1,4.5/1/2,5.0/2/1,5.4/1/0,8.0/0/1,
    11.4/1/0,12.4/0/1,13.6/1/0,16.1/0/1,16.5/1/2,17.0/2/1,17.4/1/0,20.0/0/1,23.4/1/0}{
    \draw[candidate!80!black,densely dotted] (\xx,\ya)--(\xx,\yb);}
  \foreach \xx/\nm/\cc/\hh in {0.4/x_1/src/-0.55,1.6/y_1/tgt/-0.55,
    4.1/x_2/src/-0.55,4.5/x_3/src/-1.05,5.0/y_2/tgt/-0.55,5.4/y_3/tgt/-1.05,
    8.0/x_4/src/-0.55,11.4/y_4/tgt/-0.55,12.4/x_1'/src/-0.55,13.6/y_1'/tgt/-0.55,
    16.1/x_2'/src/-0.55,16.5/x_3'/src/-1.05,17.0/y_2'/tgt/-0.55,17.4/y_3'/tgt/-1.05,
    20.0/x_4'/src/-0.55,23.4/y_4'/tgt/-0.55}{
    \draw[gray!50,line width=.4pt] (\xx,-0.10)--(\xx,\hh+0.24);
    \node[font=\tiny,text=\cc] at (\xx,\hh) {$\nm$};}
  \draw[gray!60,dashed] (12.0,-1.3) -- (12.0,2.4);
  \node[gray,font=\scriptsize] at (12.0,2.75) {$+L$ (second copy)};
\end{tikzpicture}
\end{center}

\begin{center}
\renewcommand{\arraystretch}{1.15}
\begin{tabular}{l|c|cccccccc}
\toprule
$t$          & 0   & 1     & 2     & 3     & 4     & 5     & 6     & 7     & 8\\
\midrule
atom         & --- & $x_1$ & $y_1$ & $x_2$ & $x_3$ & $y_2$ & $y_3$ & $x_4$ & $y_4$\\
$\tilde z_t$ & --- & 0.4   & 1.6   & 4.1   & 4.5   & 5.0   & 5.4   & 8.0   & 11.4\\
$R_t$        & 0   & 1     & 0     & 1     & 2     & 1     & 0     & 1     & 0\\
$S_t$        & 0   & 0.4   & $-1.2$& 2.9   & 7.4   & 2.4   & $-3.0$& 5.0   & $-6.4$\\
$p_t$        & --- & ---   & 0     & 1     & ---   & 3     & 2     & 5     & 6\\
$Q_t$        & 0   & 0     & 1.2   & 2.5   & 0     & 3.0   & 3.0   & 5.6   & 6.4\\
\bottomrule
\end{tabular}

\medskip
\begin{tabular}{l|cccccccc}
\toprule
$t$          & 9      & 10     & 11     & 12     & 13     & 14     & 15     & 16\\
\midrule
atom         & $x_1'$ & $y_1'$ & $x_2'$ & $x_3'$ & $y_2'$ & $y_3'$ & $x_4'$ & $y_4'$\\
$\tilde z_t$ & 12.4   & 13.6   & 16.1   & 16.5   & 17.0   & 17.4   & 20.0   & 23.4\\
$R_t$        & 1      & 0      & 1      & 2      & 1      & 0      & 1      & 0\\
$S_t$        & 6.0    & $-7.6$ & 8.5    & 25.0   & 8.0    & $-9.4$ & 10.6   & $-12.8$\\
$p_t$        & 7      & 8      & 9      & 4      & 11     & 10     & 13     & 14\\
$Q_t$        & 6.6    & 7.6    & 9.1    & 17.6   & 9.6    & 9.4    & 12.2   & 12.8\\
\bottomrule
\end{tabular}
\end{center}

\noindent These rows are the whole oracle; $S_t$ is listed so the $Q$ recursion can be
checked by eye. Note the index convention: $p_t$ is the \emph{predecessor}, the last
$s<t$ with $R_s=R_t$, so the minimal chain ending at $t$ is $[p_t+1,t]$ --- one position
to the right of $p_t$. For instance $p_2=0$ and the minimal chain ending at $y_1$ is
$[1,2]=\{x_1,y_1\}$. The recursion is $Q_t=Q_{p_t}+w\lvert S_t-S_{p_t}\rvert$, e.g.\
$Q_6=Q_2+\lvert S_6-S_2\rvert=1.2+\lvert-3.0-(-1.2)\rvert=3.0$, and $Q_t=0$ when $p_t$
does not exist. The $t=0$ column is the empty prefix, needed whenever $a=1$.

For any balanced interval $[a,b]$ with
$b-a<N$ (\cref{prop:balanced-query}): $c_{\mathbb R}([a,b])=Q_b-Q_{a-1}$, and the marginal of a cell with
endpoints $a,b$ is one line of arithmetic \eqref{eq:constant-marginal}:
\[
m([a,b])=\bigl(Q_b-Q_{a-1}\bigr)-\bigl(Q_{b-1}-Q_{a}\bigr).
\]
Worked example: the cell $[x_2,y_3]=[3,6]$ has
$m = (Q_6-Q_2)-(Q_5-Q_3) = (3.0-1.2)-(3.0-2.5) = 1.8-0.5 = \mathbf{1.3}$. This cell does
not exist yet --- it is created when the three cells around $[x_3,y_2]$ merge at $k{=}1$,
and it is selected at $k{=}3$ (Step 5).

\subsection*{Step 2 --- Cells, free gaps, and the initial heap (\cref{sec:cells,sec:freegaps})}

The inactive atoms partition $S^1$ into \emph{cells}. Initially every cell is a single
gap between cyclic neighbors; the six with opposite-type endpoints are \emph{candidate
cells} (green arcs) with marginal equal to their clockwise arc length; $(x_2,x_3)$ and
$(y_2,y_3)$ are not candidates. Every cell contains a \emph{free gap} --- an original gap
never absorbed by a selected cell (\cref{lem:free-gap}) --- and stores one boxed
\emph{representative}; initially each cell's representative is its own gap.

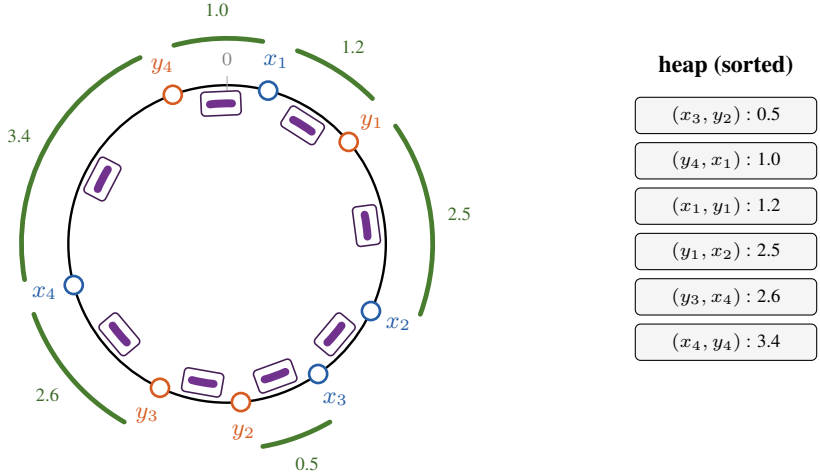
\begin{center}
\begin{tikzpicture}[scale=1.0]
 \basecircle
 \foreach \a in {\gXiYi,\gYiXii,\gXiiXiii,\gXiiiYii,\gYiiYiii,\gYiiiXiv,\gXivYiv,\gYivXi}{
   \gapmark{gfree}{\a}\repbox{\a}}
 \cellarc{cell}{\aXi}{\aYi}   \node[candidate!70!black,font=\scriptsize] at (\gXiYi:\demoR+1.0) {1.2};
 \cellarc{cell}{\aYi}{\aXii}  \node[candidate!70!black,font=\scriptsize] at (\gYiXii:\demoR+1.0) {2.5};
 \cellarc{cell}{\aXiii}{\aYii} \node[candidate!70!black,font=\scriptsize] at (\gXiiiYii:\demoR+1.0) {0.5};
 \cellarc{cell}{\aYiii}{\aXiv}\node[candidate!70!black,font=\scriptsize] at (\gYiiiXiv:\demoR+1.0) {2.6};
 \cellarc{cell}{\aXiv}{\aYivc}\node[candidate!70!black,font=\scriptsize] at (\gXivYiv:\demoR+1.0) {3.4};
 \cellarc{cell}{\aYiv}{\aXi}  \node[candidate!70!black,font=\scriptsize] at (\gYivXi:\demoR+1.0) {1.0};
 \atomn{srcin}{\aXi}{src}{$x_1$}{0.40}
 \atomn{tgtin}{\aYi}{tgt}{$y_1$}{0.40}
 \atomn{srcin}{\aXii}{src}{$x_2$}{0.40}
 \atomn{srcin}{\aXiii}{src}{$x_3$}{0.40}
 \atomn{tgtin}{\aYii}{tgt}{$y_2$}{0.40}
 \atomn{tgtin}{\aYiii}{tgt}{$y_3$}{0.40}
 \atomn{srcin}{\aXiv}{src}{$x_4$}{0.40}
 \atomn{tgtin}{\aYiv}{tgt}{$y_4$}{0.40}
 \begin{scope}[xshift=6.6cm,yshift=1.7cm]
   \node[font=\small\bfseries] at (0,0.65) {heap (sorted)};
   \foreach \i/\lab/\cst in {0/{(x_3,y_2)}/0.5, 1/{(y_4,x_1)}/1.0, 2/{(x_1,y_1)}/1.2,
                             3/{(y_1,x_2)}/2.5, 4/{(y_3,x_4)}/2.6, 5/{(x_4,y_4)}/3.4}{
     \node[draw,rounded corners=2pt,minimum width=2.4cm,minimum height=0.48cm,
           font=\scriptsize,fill=gray!8] at (0,-0.60*\i) {$\lab$ : \cst};}
 \end{scope}
\end{tikzpicture}
\end{center}

\noindent Note the wrap-around candidate $(y_4,x_1)$ at the top: on the circle it is a
perfectly ordinary cell. The two same-type gaps carry free gaps and representatives too
--- the invariant is about \emph{cells}, not candidates.

\subsection*{Step 3 --- Iteration $k=1$: pop $(x_3,y_2)$; three cells merge into one}

The cheapest cell $(x_3,y_2)$, marginal $0.5$, is selected: $x_3\!\to\!y_2$,
$C^\circ_1=0.5$. Cells $[x_2,x_3]$, $[x_3,y_2]$, $[y_2,y_3]$ merge into $[x_2,y_3]$,
whose endpoints have opposite types --- a new candidate with marginal $1.3$ (computed in
Step~1). The selected cell's gap turns \emph{used}; the merged cell inherits the boxed
representative of its \emph{left} unselected neighbor $[x_2,x_3]$; the right neighbor's
gap $(y_2,y_3)$ stays free but unstored.

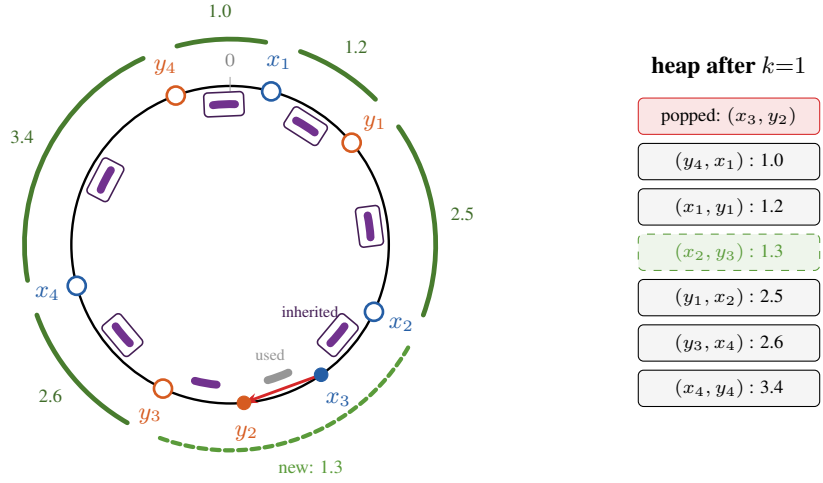
\begin{center}
\begin{tikzpicture}[scale=1.0]
 \basecircle
 \foreach \a in {\gXiYi,\gYiXii,\gYiiiXiv,\gXivYiv,\gYivXi}{\gapmark{gfree}{\a}\repbox{\a}}
 \gapmark{gfree}{\gXiiXiii}\repbox{\gXiiXiii}   
 \gapmark{gused}{\gXiiiYii}                     
 \gapmark{gfree}{\gYiiYiii}                     
 \node[freegap!60!black,font=\tiny] at (\gXiiXiii:\demoR-0.72) {inherited};
 \node[usedg,font=\tiny] at (\gXiiiYii:\demoR-0.55) {used};
 \cellarc{cell}{\aXi}{\aYi}   \node[candidate!70!black,font=\scriptsize] at (\gXiYi:\demoR+1.0) {1.2};
 \cellarc{cell}{\aYi}{\aXii}  \node[candidate!70!black,font=\scriptsize] at (\gYiXii:\demoR+1.0) {2.5};
 \cellarc{cell}{\aYiii}{\aXiv}\node[candidate!70!black,font=\scriptsize] at (\gYiiiXiv:\demoR+1.0) {2.6};
 \cellarc{cell}{\aXiv}{\aYivc}\node[candidate!70!black,font=\scriptsize] at (\gXivYiv:\demoR+1.0) {3.4};
 \cellarc{cell}{\aYiv}{\aXi}  \node[candidate!70!black,font=\scriptsize] at (\gYivXi:\demoR+1.0) {1.0};
 \cellarc{cellnew}{\aXii}{\aYiii}
 \node[candidate,font=\scriptsize] at (\gXiiiYii:\demoR+1.02) {new: 1.3};
 \draw[mnew] (\aXiii:\demoR) -- (\aYii:\demoR);
 \atomn{srcin}{\aXi}{src}{$x_1$}{0.40}
 \atomn{tgtin}{\aYi}{tgt}{$y_1$}{0.40}
 \atomn{srcin}{\aXii}{src}{$x_2$}{0.40}
 \atomn{srcpt}{\aXiii}{src}{$x_3$}{0.40}
 \atomn{tgtpt}{\aYii}{tgt}{$y_2$}{0.40}
 \atomn{tgtin}{\aYiii}{tgt}{$y_3$}{0.40}
 \atomn{srcin}{\aXiv}{src}{$x_4$}{0.40}
 \atomn{tgtin}{\aYiv}{tgt}{$y_4$}{0.40}
 \begin{scope}[xshift=6.6cm,yshift=1.7cm]
   \node[font=\small\bfseries] at (0,0.65) {heap after $k{=}1$};
   \node[draw,rounded corners=2pt,minimum width=2.4cm,minimum height=0.48cm,
         font=\scriptsize,fill=popcol!12,draw=popcol] at (0,0) {popped: $(x_3,y_2)$};
   \foreach \i/\lab/\cst in {1/{(y_4,x_1)}/1.0, 2/{(x_1,y_1)}/1.2}{
     \node[draw,rounded corners=2pt,minimum width=2.4cm,minimum height=0.48cm,
           font=\scriptsize,fill=gray!8] at (0,-0.60*\i) {$\lab$ : \cst};}
   \node[draw,dashed,rounded corners=2pt,minimum width=2.4cm,minimum height=0.48cm,
         font=\scriptsize,fill=candidate!10,draw=candidate,text=candidate] at (0,-0.60*3)
         {$(x_2,y_3)$ : 1.3};
   \foreach \i/\lab/\cst in {4/{(y_1,x_2)}/2.5, 5/{(y_3,x_4)}/2.6, 6/{(x_4,y_4)}/3.4}{
     \node[draw,rounded corners=2pt,minimum width=2.4cm,minimum height=0.48cm,
           font=\scriptsize,fill=gray!8] at (0,-0.60*\i) {$\lab$ : \cst};}
 \end{scope}
\end{tikzpicture}
\end{center}

\subsection*{Step 4 --- Iteration $k=2$: the wrap-around cell is selected}

Next cheapest is the wrap cell $(y_4,x_1)$, marginal $1.0$: the match $x_1\!\to\!y_4$
crosses $0$. $C^\circ_2=1.5$. No cut placed at the top could ever find this plan --- this
is the step where the circle genuinely departs from every line relaxation cut near $0$.
Cells $[x_4,y_4]$, $[y_4,x_1]$, $[x_1,y_1]$ merge into $[x_4,y_1]$ (opposite types):
new candidate, marginal $m=(Q_{10}-Q_6)-(Q_9-Q_7)=4.6-1.0=\mathbf{3.6}$, inheriting the
representative of $[x_4,y_4]$.

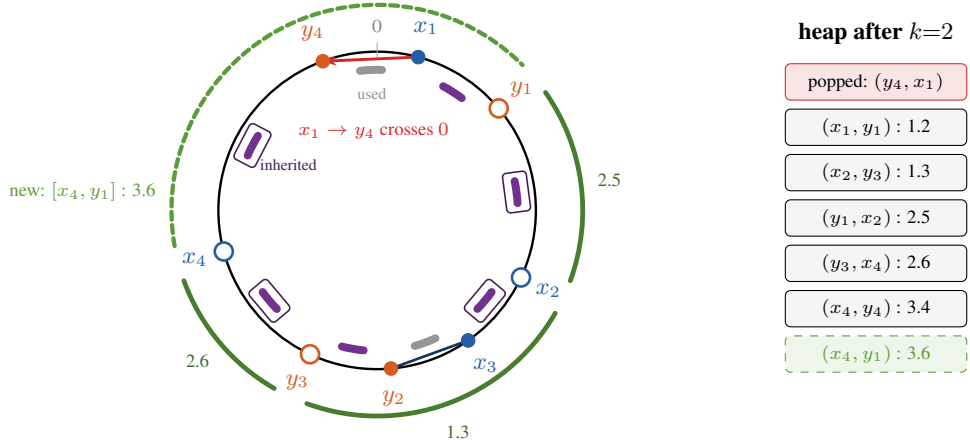
\begin{center}
\begin{tikzpicture}[scale=1.0]
 \basecircle
 \foreach \a in {\gYiXii,\gYiiiXiv}{\gapmark{gfree}{\a}\repbox{\a}}
 \gapmark{gfree}{\gXiiXiii}\repbox{\gXiiXiii}
 \gapmark{gused}{\gXiiiYii}\gapmark{gfree}{\gYiiYiii}
 \gapmark{gfree}{\gXivYiv}\repbox{\gXivYiv}  
 \gapmark{gused}{\gYivXi}                    
 \gapmark{gfree}{\gXiYi}                     
 \node[freegap!60!black,font=\tiny] at (\gXivYiv:\demoR-0.78) {inherited};
 \node[usedg,font=\tiny] at (\gYivXi:\demoR-0.55) {used};
 \cellarc{cell}{\aYi}{\aXii}  \node[candidate!70!black,font=\scriptsize] at (\gYiXii:\demoR+1.0) {2.5};
 \cellarc{cell}{\aYiii}{\aXiv}\node[candidate!70!black,font=\scriptsize] at (\gYiiiXiv:\demoR+1.0) {2.6};
 \cellarc{cell}{\aXii}{\aYiii} \node[candidate!70!black,font=\scriptsize] at (\gXiiiYii:\demoR+1.02) {1.3};
 \cellarc{cellnew}{\aXiv}{\aYic}
 \node[candidate,font=\scriptsize,anchor=east] at (-185:\demoR+0.72) {new: $[x_4,y_1]$ : 3.6};
 \draw[mold] (\aXiii:\demoR) -- (\aYii:\demoR);
 \draw[mnew] (\aXi:\demoR) -- (\aYiv:\demoR);
 \node[popcol,font=\scriptsize] at (\gYivXi:\demoR-1.05) {$x_1\to y_4$ crosses $0$};
 \atomn{srcpt}{\aXi}{src}{$x_1$}{0.40}
 \atomn{tgtin}{\aYi}{tgt}{$y_1$}{0.40}
 \atomn{srcin}{\aXii}{src}{$x_2$}{0.40}
 \atomn{srcpt}{\aXiii}{src}{$x_3$}{0.40}
 \atomn{tgtpt}{\aYii}{tgt}{$y_2$}{0.40}
 \atomn{tgtin}{\aYiii}{tgt}{$y_3$}{0.40}
 \atomn{srcin}{\aXiv}{src}{$x_4$}{0.40}
 \atomn{tgtpt}{\aYiv}{tgt}{$y_4$}{0.40}
 \begin{scope}[xshift=6.6cm,yshift=1.7cm]
   \node[font=\small\bfseries] at (0,0.65) {heap after $k{=}2$};
   \node[draw,rounded corners=2pt,minimum width=2.4cm,minimum height=0.48cm,
         font=\scriptsize,fill=popcol!12,draw=popcol] at (0,0) {popped: $(y_4,x_1)$};
   \foreach \i/\lab/\cst in {1/{(x_1,y_1)}/1.2, 2/{(x_2,y_3)}/1.3, 3/{(y_1,x_2)}/2.5,
                             4/{(y_3,x_4)}/2.6, 5/{(x_4,y_4)}/3.4}{
     \node[draw,rounded corners=2pt,minimum width=2.4cm,minimum height=0.48cm,
           font=\scriptsize,fill=gray!8] at (0,-0.60*\i) {$\lab$ : \cst};}
   \node[draw,dashed,rounded corners=2pt,minimum width=2.4cm,minimum height=0.48cm,
         font=\scriptsize,fill=candidate!10,draw=candidate,text=candidate] at (0,-0.60*6)
         {$(x_4,y_1)$ : 3.6};
 \end{scope}
\end{tikzpicture}
\end{center}

\subsection*{Step 5 --- Iteration $k=3$: a lazy skip, then the merged cell fires}

The heap's top, $(x_1,y_1):1.2$, is \emph{stale}: $x_1$ was activated by the wrap
selection. Lazy deletion discards it on pop (validity check: both endpoints inactive,
cyclic successors, opposite types). Next is the merged cell $[x_2,y_3]:1.3$ --- valid.
Its endpoints activate and the cell's interior is rematched in unwrapped order:
$x_2\!\to\!y_2$, $x_3\!\to\!y_3$, \emph{revoking} $x_3\!\to\!y_2$. $C^\circ_3=2.8$.
Merging $[y_1,x_2]+[x_2,y_3]+[y_3,x_4]\to[y_1,x_4]$ gives a new candidate, marginal
$(Q_7-Q_1)-(Q_6-Q_2)=5.6-1.8=\mathbf{3.8}$. Crucially, the selected cell's own boxed
representative $(x_2,x_3)$ is \emph{consumed} along with all its gaps --- this is exactly
why \cref{lem:free-gap} makes the merged cell inherit from an \emph{unselected} neighbor: here the
left one, $[y_1,x_2]$, donating gap $(y_1,x_2)$.

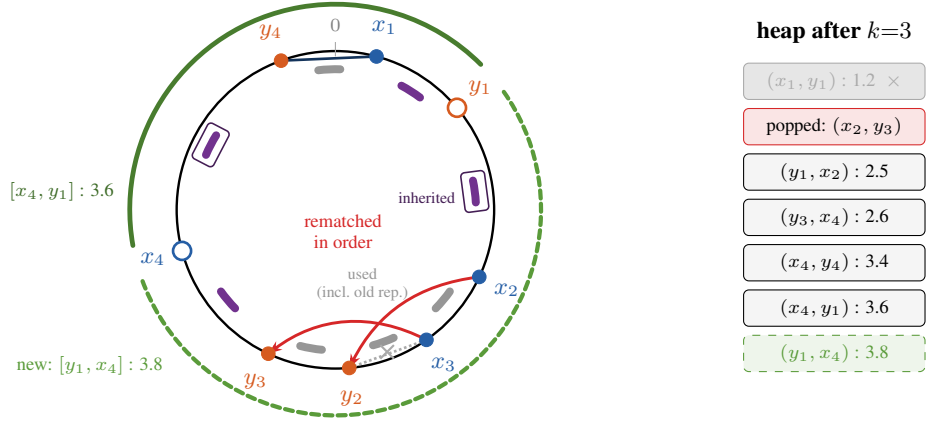
\begin{center}
\begin{tikzpicture}[scale=1.0]
 \basecircle
 \gapmark{gfree}{\gYiXii}\repbox{\gYiXii}      
 \gapmark{gused}{\gXiiXiii}\gapmark{gused}{\gXiiiYii}\gapmark{gused}{\gYiiYiii}
 \gapmark{gfree}{\gYiiiXiv}                    
 \gapmark{gfree}{\gXivYiv}\repbox{\gXivYiv}
 \gapmark{gused}{\gYivXi}\gapmark{gfree}{\gXiYi}
 \node[freegap!60!black,font=\tiny] at (\gYiXii:\demoR-0.88) {inherited};
 \node[usedg,font=\tiny,align=center] at (\gXiiiYii:\demoR-1.05) {used\\(incl.\ old rep.)};
 \cellarc{cell}{\aXiv}{\aYic}
 \node[candidate!70!black,font=\scriptsize,anchor=east] at (-185:\demoR+0.72) {$[x_4,y_1]$ : 3.6};
 \cellarc{cellnew}{\aYi}{\aXiv}
 \node[candidate,font=\scriptsize,anchor=north east] at (-140:\demoR+0.72) {new: $[y_1,x_4]$ : 3.8};
 \draw[mold] (\aXi:\demoR) -- (\aYiv:\demoR);
 \draw[mrev] (\aXiii:\demoR) -- (\aYii:\demoR);
 \draw[inact,thick] ($(\gXiiiYii:\demoR-0.1)+(-0.09,-0.09)$) -- ($(\gXiiiYii:\demoR-0.1)+(0.09,0.09)$);
 \draw[inact,thick] ($(\gXiiiYii:\demoR-0.1)+(-0.09,0.09)$) -- ($(\gXiiiYii:\demoR-0.1)+(0.09,-0.09)$);
 \draw[mnew] (\aXii:\demoR) to[bend right=28] (\aYii:\demoR);
 \draw[mnew] (\aXiii:\demoR) to[bend right=32] (\aYiii:\demoR);
 \node[popcol,font=\scriptsize,align=center] at (-70:0.30) {rematched\\in order};
 \atomn{srcpt}{\aXi}{src}{$x_1$}{0.40}
 \atomn{tgtin}{\aYi}{tgt}{$y_1$}{0.40}
 \atomn{srcpt}{\aXii}{src}{$x_2$}{0.40}
 \atomn{srcpt}{\aXiii}{src}{$x_3$}{0.40}
 \atomn{tgtpt}{\aYii}{tgt}{$y_2$}{0.40}
 \atomn{tgtpt}{\aYiii}{tgt}{$y_3$}{0.40}
 \atomn{srcin}{\aXiv}{src}{$x_4$}{0.40}
 \atomn{tgtpt}{\aYiv}{tgt}{$y_4$}{0.40}
 \begin{scope}[xshift=6.6cm,yshift=1.7cm]
   \node[font=\small\bfseries] at (0,0.65) {heap after $k{=}3$};
   \node[draw,rounded corners=2pt,minimum width=2.4cm,minimum height=0.48cm,
         font=\scriptsize,fill=gray!20,draw=inact,text=inact] at (0,0)
         {$(x_1,y_1)$ : 1.2 \;$\times$};
   \node[draw,rounded corners=2pt,minimum width=2.4cm,minimum height=0.48cm,
         font=\scriptsize,fill=popcol!12,draw=popcol] at (0,-0.60) {popped: $(x_2,y_3)$};
   \foreach \i/\lab/\cst in {2/{(y_1,x_2)}/2.5, 3/{(y_3,x_4)}/2.6, 4/{(x_4,y_4)}/3.4,
                             5/{(x_4,y_1)}/3.6}{
     \node[draw,rounded corners=2pt,minimum width=2.4cm,minimum height=0.48cm,
           font=\scriptsize,fill=gray!8] at (0,-0.60*\i) {$\lab$ : \cst};}
   \node[draw,dashed,rounded corners=2pt,minimum width=2.4cm,minimum height=0.48cm,
         font=\scriptsize,fill=candidate!10,draw=candidate,text=candidate] at (0,-0.60*6)
         {$(y_1,x_4)$ : 3.8};
 \end{scope}
\end{tikzpicture}
\end{center}

\subsection*{Step 6 --- Iteration $k=4$: three skips, the endgame, and a circulation shift}

Three stale entries are popped and discarded in a row: $(y_1,x_2)$, $(y_3,x_4)$,
$(x_4,y_4)$ --- each has an active endpoint. Only two inactive atoms remain, $y_1$ and
$x_4$, of opposite types: the endgame. They bound \emph{two} directed cells, $[x_4,y_1]$
through the top (marginal $3.6$) and its complement $[y_1,x_4]$ through the bottom
($3.8$); both sit in the heap, and the pop order selects the cheaper automatically.
Selecting $[x_4,y_1]$ activates $x_4,y_1$ and rematches its interior $\{y_4,x_1\}$ in
unwrapped order: $x_4\!\to\!y_4$, $x_1\!\to\!y_1$ --- \emph{revoking the wrap match}
$x_1\!\to\!y_4$. The circulation that was optimal at $k=2,3$ is abandoned at $k=4$:
$C^\circ_4=2.8+3.6=6.4$.

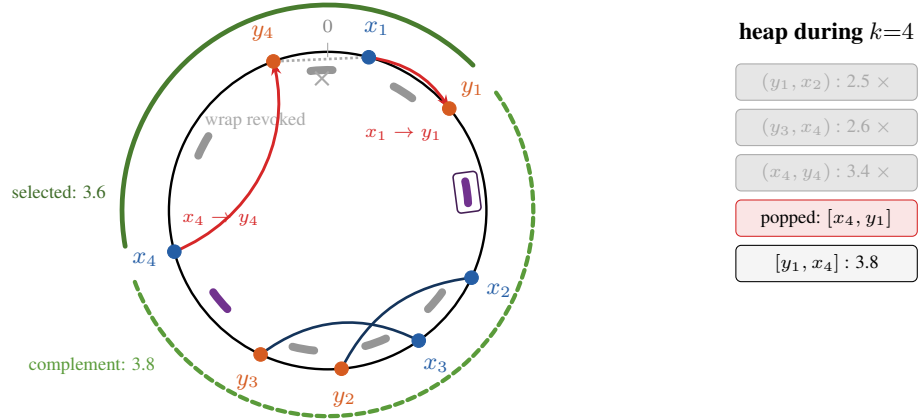
\begin{center}
\begin{tikzpicture}[scale=1.0]
 \basecircle
 \gapmark{gfree}{\gYiXii}\repbox{\gYiXii}
 \gapmark{gused}{\gXiiXiii}\gapmark{gused}{\gXiiiYii}\gapmark{gused}{\gYiiYiii}
 \gapmark{gfree}{\gYiiiXiv}
 \gapmark{gused}{\gXivYiv}\gapmark{gused}{\gYivXi}\gapmark{gused}{\gXiYi}
 \draw[mold] (\aXii:\demoR) to[bend right=28] (\aYii:\demoR);
 \draw[mold] (\aXiii:\demoR) to[bend right=32] (\aYiii:\demoR);
 \draw[mrev] (\aXi:\demoR) -- (\aYiv:\demoR);
 \draw[inact,thick] ($(\gYivXi:\demoR-0.35)+(-0.09,-0.09)$) -- ($(\gYivXi:\demoR-0.35)+(0.09,0.09)$);
 \draw[inact,thick] ($(\gYivXi:\demoR-0.35)+(-0.09,0.09)$) -- ($(\gYivXi:\demoR-0.35)+(0.09,-0.09)$);
 \node[inact,font=\scriptsize,anchor=east] at (98:1.20) {wrap revoked};
 \draw[mnew] (\aXiv:\demoR) to[bend right=40] (\aYiv:\demoR);
 \draw[mnew] (\aXi:\demoR) to[bend left=24] (\aYi:\demoR);
 \node[popcol,font=\scriptsize] at (185:1.42) {$x_4\to y_4$};
 \node[popcol,font=\scriptsize] at (45:1.42) {$x_1\to y_1$};
 \cellarc{cell}{\aXiv}{\aYic}
 \node[candidate!70!black,font=\scriptsize,anchor=east] at (-185:\demoR+0.72) {selected: 3.6};
 \cellarc{cellnew}{\aYi}{\aXiv}
 \node[candidate,font=\scriptsize,anchor=north east] at (-140:\demoR+0.72) {complement: 3.8};
 \atomn{srcpt}{\aXi}{src}{$x_1$}{0.40}
 \atomn{tgtpt}{\aYi}{tgt}{$y_1$}{0.40}
 \atomn{srcpt}{\aXii}{src}{$x_2$}{0.40}
 \atomn{srcpt}{\aXiii}{src}{$x_3$}{0.40}
 \atomn{tgtpt}{\aYii}{tgt}{$y_2$}{0.40}
 \atomn{tgtpt}{\aYiii}{tgt}{$y_3$}{0.40}
 \atomn{srcpt}{\aXiv}{src}{$x_4$}{0.40}
 \atomn{tgtpt}{\aYiv}{tgt}{$y_4$}{0.40}
 \begin{scope}[xshift=6.6cm,yshift=1.7cm]
   \node[font=\small\bfseries] at (0,0.65) {heap during $k{=}4$};
   \foreach \i/\lab/\cst in {0/{(y_1,x_2)}/2.5, 1/{(y_3,x_4)}/2.6, 2/{(x_4,y_4)}/3.4}{
     \node[draw,rounded corners=2pt,minimum width=2.4cm,minimum height=0.48cm,
           font=\scriptsize,fill=gray!20,draw=inact,text=inact] at (0,-0.60*\i)
           {$\lab$ : \cst \;$\times$};}
   \node[draw,rounded corners=2pt,minimum width=2.4cm,minimum height=0.48cm,
         font=\scriptsize,fill=popcol!12,draw=popcol] at (0,-0.60*3)
         {popped: $[x_4,y_1]$};
   \node[draw,rounded corners=2pt,minimum width=2.4cm,minimum height=0.48cm,
         font=\scriptsize,fill=gray!8] at (0,-0.60*4) {$[y_1,x_4]$ : 3.8};
 \end{scope}
\end{tikzpicture}
\end{center}

\subsection*{Step 7 --- \cref{thm:simultaneous-cut}: the simultaneous cut $\theta^\star$}

Two gaps survive as free: $(y_1,x_2)$ and $(y_3,x_4)$. The representative maintained
through every merge --- last held by the complementary cell $[y_1,x_4]$ --- is
$\theta^\star=(y_1,x_2)$: cutting there and running plain PAWL on the unrolled line
reproduces $C^\circ_k$ for \emph{every} $k$ at once. The table below (each row an
exhaustively computed line profile) shows why this is not automatic: a gap can be free
\emph{at some} $k$ yet fail at another. In particular $(x_4,y_4)$ is only consumed at the
last step and matches $k\le3$ but fails at $k=4$; and the wrap gap $(y_4,x_1)$, used at
$k=2$, happens to carry no flow again at $k=4$ --- ``used'' is a \emph{for-all-$k$}
certificate, which per-$k$ freeness cannot replace.

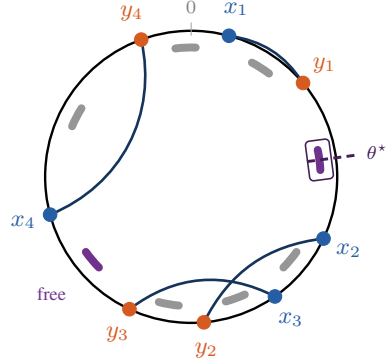
\begin{center}
\renewcommand{\arraystretch}{1.25}
\begin{tabular}{l l ccccc l}
\toprule
cut in gap & status & $k{=}0$ & 1 & 2 & 3 & 4 & verdict\\
\midrule
$(y_1,x_2)=\theta^\star$ & free (representative) & 0 & 0.5 & 1.5 & 2.8 & 6.4
  & matches all $k$\\
$(y_3,x_4)$ & free & 0 & 0.5 & 1.5 & 2.8 & 6.4 & matches all $k$\\
$(x_4,y_4)$ & used at $k{=}4$ & 0 & 0.5 & 1.5 & 2.8 & \textbf{6.6} & fails at $k{=}4$\\
$(y_4,x_1)$ & used at $k{=}2$ & 0 & 0.5 & \textbf{1.7} & \textbf{3.0} & 6.4
  & fails at $k{=}2,3$\\
$(x_3,y_2)$ & used at $k{=}1$ & 0 & \textbf{1.0} & \textbf{3.5} & \textbf{6.1} &
  \textbf{17.6} & fails from $k{=}1$\\
\bottomrule
\end{tabular}
\end{center}

\begin{center}
\begin{tikzpicture}[scale=0.92]
 \basecircle
 \gapmark{gfree}{\gYiXii}\repbox{\gYiXii}
 \node[freegap!60!black,font=\scriptsize] at (\gYiXii:\demoR+0.60) {$\theta^\star$}; 
 \gapmark{gfree}{\gYiiiXiv}
 \node[freegap,font=\scriptsize] at (\gYiiiXiv:\demoR+0.52) {free};
 \foreach \a in {\gXiYi,\gXiiXiii,\gXiiiYii,\gYiiYiii,\gXivYiv,\gYivXi}{\gapmark{gused}{\a}}
 \draw[freegap!60!black,densely dashed,line width=1.1pt] (\gYiXii:\demoR-0.42) -- (\gYiXii:\demoR+0.30);
 \draw[mold] (\aXii:\demoR) to[bend right=28] (\aYii:\demoR);
 \draw[mold] (\aXiii:\demoR) to[bend right=32] (\aYiii:\demoR);
 \draw[mold] (\aXiv:\demoR) to[bend right=40] (\aYiv:\demoR);
 \draw[mold] (\aXi:\demoR) to[bend left=24] (\aYi:\demoR);
 \atomn{srcpt}{\aXi}{src}{$x_1$}{0.40}
 \atomn{tgtpt}{\aYi}{tgt}{$y_1$}{0.40}
 \atomn{srcpt}{\aXii}{src}{$x_2$}{0.40}
 \atomn{srcpt}{\aXiii}{src}{$x_3$}{0.40}
 \atomn{tgtpt}{\aYii}{tgt}{$y_2$}{0.40}
 \atomn{tgtpt}{\aYiii}{tgt}{$y_3$}{0.40}
 \atomn{srcpt}{\aXiv}{src}{$x_4$}{0.40}
 \atomn{tgtpt}{\aYiv}{tgt}{$y_4$}{0.40}
 \node[font=\scriptsize,gray,align=left] at (6.2,0.9)
   {cut at $\theta^\star$, unroll from $x_2$,\\
    keep atoms with activation\\ rank $\le k$, match in order\\
    $\Rightarrow$ optimal $\pi^k$ for every $k$\\ (plan recovery, \cref{sec:plan-recovery})};
\end{tikzpicture}
\end{center}

\subsection*{Step 8 --- The output: the full circular profile}

One sweep has produced $C^\circ_k$ for every $k$; fractional mass follows by linear
interpolation (convexity of the min-cost-flow value function). The increments
$0.5,\,1.0,\,1.3,\,3.6$ are non-decreasing, so elbow-based selection of the transported
mass works exactly as in PAWL --- note the pronounced elbow at $k=3$, where the next
sample would cost nearly three times more to transport.

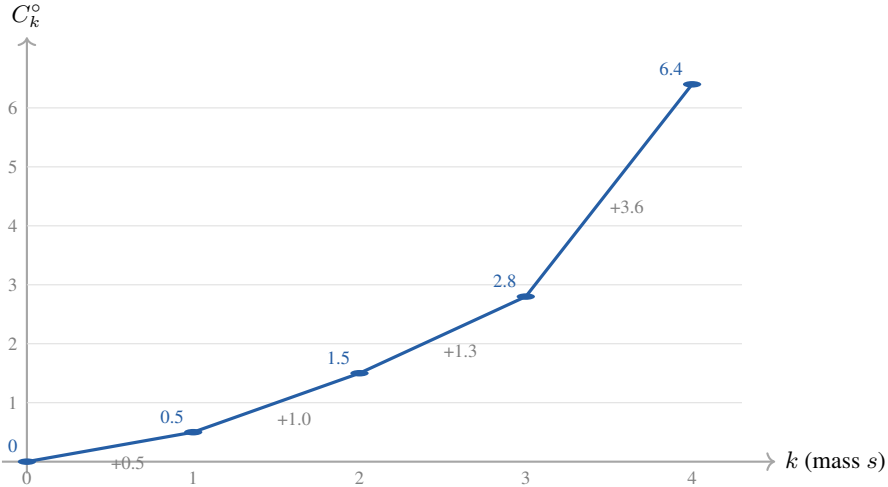
\begin{center}
\begin{tikzpicture}[xscale=2.2,yscale=0.78]
  \draw[->,thick,gray!70] (-0.15,0) -- (4.5,0) node[right,font=\small,black] {$k$ (mass $s$)};
  \draw[->,thick,gray!70] (0,-0.2) -- (0,7.2) node[above,font=\small,black] {$C^\circ_k$};
  \foreach \yy in {1,2,3,4,5,6}{\draw[gray!25] (0,\yy)--(4.3,\yy);
    \node[left,font=\scriptsize,gray] at (0,\yy) {\yy};}
  \foreach \kk in {0,...,4}{\node[below,font=\scriptsize,gray] at (\kk,0) {\kk};}
  \draw[very thick,src] (0,0) -- (1,0.5) -- (2,1.5) -- (3,2.8) -- (4,6.4);
  \foreach \kk/\vv in {0/0,1/0.5,2/1.5,3/2.8,4/6.4}{
    \fill[src] (\kk,\vv) circle (0.055);
    \node[above left=0pt,font=\scriptsize,src] at (\kk,\vv) {\vv};}
  \foreach \kk/\vc/\sl in {0/0.25/{+0.5},1/1.0/{+1.0},2/2.15/{+1.3},3/4.6/{+3.6}}{
    \node[font=\scriptsize,gray,below right] at ({\kk+0.45},\vc) {\sl};}
\end{tikzpicture}
\end{center}

\medskip
\noindent\textbf{Summary of the run.}
\begin{center}
\renewcommand{\arraystretch}{1.25}
\small
\begin{tabular}{c l l l l c}
\toprule
$k$ & popped (skipped) & action & new candidate & rep.\ inherited & $C^\circ_k$\\
\midrule
1 & $(x_3,y_2){:}0.5$ & $x_3\to y_2$ & $[x_2,y_3]{:}1.3$ & $(x_2,x_3)$ & 0.5\\
2 & $(y_4,x_1){:}1.0$ & $x_1\to y_4$ \emph{(wrap)} & $[x_4,y_1]{:}3.6$ & $(x_4,y_4)$ & 1.5\\
3 & (\,$(x_1,y_1)$\,), $[x_2,y_3]{:}1.3$ & rematch $x_2{\to}y_2$, $x_3{\to}y_3$
  & $[y_1,x_4]{:}3.8$ & $(y_1,x_2)$ & 2.8\\
4 & (\,3 stale\,), $[x_4,y_1]{:}3.6$ & rematch; wrap revoked & --- & --- & 6.4\\
\bottomrule
\end{tabular}
\end{center}

\medskip
\noindent Every marginal was a constant-time lookup in the $Q$-table of Step 1; every
pop, skip, and insert a heap operation on at most $N{+}k$ entries; the free-gap
representatives cost $O(1)$ bookkeeping per merge and delivered $\theta^\star$ for free.
Total: $O(N\log N)$ --- against $O(N^2\log N)$ for the cut-enumeration baseline whose
entire output (the envelope table of Step 7) this single sweep reproduces.

\section{Notation and problem formulation}
\label{app:setup}

Let the circle of circumference $L>0$ be
\[
\Sone = \RR/L\mathbb{Z},
\]
and represent its points by coordinates in $[0,L)$.  Its geodesic distance is
\begin{equation}
\dcirc(x,y)=\min\{|x-y|,\,L-|x-y|\}.
\label{eq:circle-distance}
\end{equation}
Consider two empirical measures
\begin{equation}
\mu=w\sum_{i=1}^{n}\delta_{x_i},
\qquad
\nu=w\sum_{j=1}^{m}\delta_{y_j},
\qquad w>0.
\label{eq:measures}
\end{equation}

\begin{assumption}[Standing assumptions]
\label{ass:standing}
All atoms have the same mass $w$; the points in
$\{x_1,\dots,x_n,y_1,\dots,y_m\}$ are pairwise distinct; and the ground cost is the geodesic distance \eqref{eq:circle-distance}.  Let $N=n+m$ and $K=\min(n,m)$.
\end{assumption}

For a transported mass $s\in[0,Kw]$, define
\begin{equation}
\PW_{\circ}(s)
=
\min_{\pi\in\RR_+^{n\times m}}
\sum_{i=1}^{n}\sum_{j=1}^{m}\dcirc(x_i,y_j)\pi_{ij}
\label{eq:partial-lp}
\end{equation}
subject to
\begin{equation}
\sum_j\pi_{ij}\leq w,
\qquad
\sum_i\pi_{ij}\leq w,
\qquad
\sum_{i,j}\pi_{ij}=s.
\label{eq:partial-constraints}
\end{equation}
At an integer transported mass \(s=kw\), define the rescaled transport plan
\[
\gamma_{ij} \coloneqq \frac{\pi_{ij}}{w}.
\]
The feasible set then becomes
\[
\mathcal{P}_k
=
\left\{
\gamma \in \mathbb{R}_{+}^{n\times m}
\;\middle|\;
\sum_{j=1}^{m}\gamma_{ij}\leq 1,\ \forall i,
\quad
\sum_{i=1}^{n}\gamma_{ij}\leq 1,\ \forall j,
\quad
\sum_{i=1}^{n}\sum_{j=1}^{m}\gamma_{ij}=k
\right\}.
\]
This is the cardinality-\(k\) bipartite matching polytope. Its extreme
points are integral, meaning that
\[
\gamma_{ij}\in\{0,1\}.
\]
Since the objective is linear, an optimum is attained at an extreme
point of \(\mathcal{P}_k\). Therefore, there exists an optimal transport
plan satisfying
\[
\pi_{ij}\in\{0,w\}.
\]
Moreover, the constraint
\[
\sum_{i,j}\gamma_{ij}=k
\]
implies that exactly \(k\) entries of \(\gamma\) are equal to one, while
the row and column constraints ensure that no two selected entries share
the same source or target. Hence, an optimal solution can be represented
by a cardinality-\(k\) matching between \(k\) distinct source atoms and
\(k\) distinct target atoms. We therefore write
\begin{equation}
C_k^{\circ}
=
\min_{M:\,|M|=k}
w\sum_{(x_i,y_j)\in M}\dcirc(x_i,y_j),
\qquad k=0,\dots,K,
\label{eq:integer-problem}
\end{equation}
where no source or target atom occurs in more than one edge of $M$.

The \emph{active set} of a matching is
\begin{equation}
A(M)=\{x_i:\exists j,(x_i,y_j)\in M\}
\cup
\{y_j:\exists i,(x_i,y_j)\in M\}.
\end{equation}
When only the active set matters, we write $A_k$ for an optimal set containing $k$ source and $k$ target atoms.

\paragraph{Mass accounting.}
\Cref{fig:problem-accounting} provides an accounting interpretation
of the fixed-cardinality partial transport problem. At cardinality \(k\),
exactly \(kw\) units of mass are transported, while \((n-k)w\) units of
source mass remain unmatched and \((m-k)w\) units of target capacity remain
unfilled. These latter quantities may be interpreted as destroyed and
created mass, respectively, and represented by arcs to dummy reservoirs.
However, because \(k\) is fixed in~\eqref{eq:integer-problem}, the total
amounts of destruction and creation are prescribed. Consequently, constant
per-unit destruction and creation penalties contribute only an additive
constant, and PAWC needs to optimize only the geodesic transportation cost.

\section{Background: the line structure and the cut envelope}
\label{app:background}

\subsection{The line structure inherited from PAWL}

On the line, uniformly weighted partial \(W_1\) admits a particularly rigid combinatorial structure that makes it possible to compute the entire transported-mass profile efficiently. In particular, optimal active sets can be chosen to grow monotonically with the transported cardinality, and the transition from a cardinality-\(k\) solution to a cardinality-\((k+1)\) solution activates exactly one source atom and one target atom. Moreover, these newly activated atoms must be adjacent after removing the current active set, so each admissible update is represented by a balanced contiguous chain. For the absolute-value ground cost, the marginal cost of activating such a
chain can be recovered from precomputed chain costs in constant time. These structural and computational properties form the basis of PAWL
\citep{chapel2025one} and are summarized below, as they will serve as the starting point for our extension from the line to the circle.

Throughout, all line costs are written with a single symbol.  Let $P$ be a finite
set of support points carrying equally many source and target atoms, all of common
weight $w$, and let $x_{(1)}<\dots<x_{(j)}$ and $y_{(1)}<\dots<y_{(j)}$ be its
source and target atoms in increasing order.  Such a $P$ is called \emph{balanced},
and we write
\begin{equation}
c_{\mathbb R}(P)
=
w\sum_{i=1}^{j}\bigl|x_{(i)}-y_{(i)}\bigr|
\label{eq:line-cost}
\end{equation}
for its one-dimensional $W_1$ transport cost, the increasing-order matching being
optimal on the line.  We apply $c_{\mathbb R}$ to whatever set of atoms is under
discussion and abbreviate accordingly: $c_{\mathbb R}(I)$ for the atoms lying in an
interval $I$, $c_{\mathbb R}(A)$ for those of an active set $A$,
$c_{\mathbb R}(\mathcal C)$ for those of a circular cell $\mathcal C$ once a cut
outside it has been fixed, and $c_{\mathbb R}([a,b])$ for those at positions $a$
through $b$ of the doubled sequence of \cref{sec:doubled}.  In every case the
argument is balanced, so \eqref{eq:line-cost} applies verbatim; only the description
of which atoms are meant changes.

\begin{lemma}[Universal line-neighbor and locality principle]
\label{lem:line-locality}
Consider two optimal active sets for uniformly weighted partial
\(W_1\) on the line satisfying
\[
A_{k+1}=A_k\cup\{u,v\},
\]
where \(u\) and \(v\) belong to opposite measures. Then \(u\) and \(v\)
are consecutive in the ordered inactive set \(A_k^c\). Consequently,
the open interval between \(u\) and \(v\) contains only active atoms and
is balanced. Moreover, passing from \(A_k\) to \(A_{k+1}\) changes the
increasing-order optimal matching only inside the closed interval
\([u,v]\); every matching edge outside this interval is preserved.
\end{lemma}

\begin{proof}[Proof sketch and citation map]
This lemma packages three ingredients from
\citet[Proposition~2, Lemma~1 in Appendix~A.1, and
Section~3.3]{chapel2025one}. Their Lemma~1 shows that every optimal
active set on the line can be decomposed into disjoint contiguous
balanced chains. Their Proposition~2 then proves the neighbor property,
while the discussion in Section~3.3 identifies the resulting local
rematching and its chain marginal.

We recall the exchange argument because the universal quantifier in the
present statement is important. Let \(A_k\) be any optimal
cardinality-\(k\) active set, and let
\[
A_{k+1}=A_k\cup\{u,v\}
\]
be any optimal nested cardinality-\((k+1)\) extension. Assume without
loss of generality that \(u<v\). Suppose, toward a contradiction, that
\(u\) and \(v\) are not consecutive in the ordered inactive set
\(A_k^c\). Then there exists an inactive atom in the open interval
\((u,v)\). Let \(z\) be the first such atom encountered when moving from
\(u\) toward \(v\). By symmetry, we may assume that \(z\) belongs to the
same measure as \(v\).

Because \(z\) is the first inactive atom after \(u\), every support point
strictly between \(u\) and \(z\) is active. The chain decomposition of
\(A_k\) therefore implies that the active atoms in \((u,z)\), together
with the opposite-type endpoints \(u\) and \(z\), form a balanced
contiguous chain. Hence
\[
A_k\cup\{u,z\}
\]
is another feasible cardinality-\((k+1)\) active set.

The exchange calculation in the proof of
\citet[Lemma~1, equations~(3)--(6)]{chapel2025one} compares this
competitor with \(A_k\cup\{u,v\}\). In the increasing-order matching,
replacing the nearer atom \(z\) by the strictly more distant atom \(v\)
forces one local matched endpoint to move strictly to the right. The
cost of the affected left chain therefore increases strictly. The
remaining part of the matching cannot compensate for this increase:
otherwise one could replace the corresponding active atom in \(A_k\)
by \(v\) and obtain a cardinality-\(k\) solution cheaper than \(A_k\),
contradicting its optimality. Thus,
\[
c_{\mathbb R}\bigl(A_k\cup\{u,v\}\bigr)
>
c_{\mathbb R}\bigl(A_k\cup\{u,z\}\bigr),
\]
contradicting the assumed optimality of the extension
\(A_k\cup\{u,v\}\).

The inequality is strict because all support locations are distinct.
Consequently, a non-neighbor extension cannot even tie an optimal
neighbor extension. This is what upgrades the existential statement
that \emph{some} optimal extension uses neighboring inactive atoms to
the universal statement used here: \emph{every} optimal nested
extension must do so.

It remains to establish balance and locality. Since \(u\) and \(v\) are
consecutive in \(A_k^c\), every atom in \((u,v)\) belongs to \(A_k\).
By the chain decomposition, this interior is a union of contiguous
balanced chains and is therefore itself balanced. On the line, the
optimal matching between uniformly weighted active measures pairs
sources and targets in increasing order. Adding the opposite-type
endpoints \(u\) and \(v\) may revoke and replace matches within
\([u,v]\), but the balanced interior ensures that the source and target
ranks agree again upon leaving the interval. Hence all increasing-order
matches strictly to the left and strictly to the right of \([u,v]\)
remain unchanged. This is the locality property described in
\citet[Section~3.3, cases (S1)--(S2)]{chapel2025one}.
\end{proof}

\begin{remark}[Why the universal form matters]
\label{rem:universal-line-neighbor}
The strict exchange argument above rules out every non-neighbor optimal
nested extension; the lemma is therefore stronger than the existence of
one suitably chosen nested sequence. This universal form is used in
\cref{thm:cyclic-neighbor}, where the line-neighbor property must be
applied to the particular nested pair obtained after choosing a common
cut.
\end{remark}

For a balanced line interval $I$ and its interior $I^{\circ}$, define
\begin{equation}
\mathfrak m(I)=c_{\mathbb R}(I)-c_{\mathbb R}(I^{\circ}),
\label{eq:line-marginal}
\end{equation}
with $c_{\mathbb R}$ as in \eqref{eq:line-cost}.  By \cref{lem:line-locality}, \eqref{eq:line-marginal} is the exact line cost increment produced by activating the endpoints of $I$.  This is the chain marginal used by \PAWL{} \citep{chapel2025one}.

\subsection{Circular flow, circulation, and cuts}
This subsection concerns the balanced transport problem induced by a fixed
active set. Specifically, let \(A\) contain \(k\) source atoms and \(k\)
target atoms. The restricted measures
\[
\mu_A
=
w\sum_{x_i\in A\cap\supp\mu}\delta_{x_i},
\qquad
\nu_A
=
w\sum_{y_j\in A\cap\supp\nu}\delta_{y_j}
\]
have equal total mass \(kw\), so their transport cost is an ordinary
balanced \(1\)-Wasserstein problem on the circle. The partial nature of the
overall problem enters through the subsequent optimization over the choice
of the active set \(A\).

To characterize one-dimensional transport on the circle, we first introduce the elementary arcs between consecutive support points and the cumulative mass imbalance associated with them. Let
\[
z_1,\ldots,z_N,
\qquad N=n+m,
\]
denote the cyclically ordered union of the source and target support points.
Throughout the paper, indices are understood cyclically, so that
\(z_{N+1}=z_1\). The \emph{original support gaps} are the open circular
arcs
\[
G_i \coloneqq (z_i,z_{i+1})_{\circ},
\qquad i=1,\ldots,N,
\]
and their lengths are denoted by
\[
\ell_i \coloneqq |G_i|
=d_{\mathrm{cw}}(z_i,z_{i+1}),
\]
where \(d_{\mathrm{cw}}\) is the clockwise arc length. By construction,
no source or target atom lies in the interior of a support gap. Under our
distinct-support assumption, every gap has positive length and
\[
\sum_{i=1}^{N}\ell_i=L.
\]
These original support gaps should not be confused with the
\emph{current cells} introduced later: a current cell is an arc between
consecutive inactive atoms and may contain several original support gaps
as well as active atoms.

Now fix a balanced active set \(A\), containing the same number of source
and target atoms. Choose an arbitrary origin \(\theta\) in one of the
support gaps and unwrap the circle at \(\theta\), thereby identifying
\(\mathbb S_L^1\) with the interval \([0,L)\). Define the cumulative
imbalance
\[
H_A(t)
=
\#\bigl\{
x\in A\cap\supp\mu:\ x\in[0,t)
\bigr\}
-
\#\bigl\{
y\in A\cap\supp\nu:\ y\in[0,t)
\bigr\}.
\]
Thus, \(H_A(t)\) records the net number of active source atoms minus
active target atoms encountered while moving clockwise from the cut
\(\theta\) to \(t\). It increases by one when an active source atom is
crossed, decreases by one when an active target atom is crossed, and is
unchanged when an inactive atom is crossed. Consequently, \(H_A\) is
integer-valued and constant on the interior of every original support
gap \(G_i\). We denote this constant value by
\[
H_A(G_i).
\]

Unlike on the line, the cumulative imbalance on the circle depends on
where the circle is cut. Equivalently, a feasible circular flow is
determined only up to an arbitrary constant circulation around the
cycle. Accounting for this degree of freedom gives the standard
one-dimensional circular \(W_1\) representation
\begin{equation}
W_{1,\circ}(A)
=
w\min_{a\in\RR}
\int_0^L
\left|H_A(t)-a\right|\,dt
=
w\min_{a\in\RR}
\sum_{i=1}^{N}
\ell_i\left|H_A(G_i)-a\right|.
\label{eq:circular-median}
\end{equation}
The scalar \(a\) represents the circulation level. The second equality
follows because \(H_A\) is constant on each support gap. Hence, a
minimizer \(a^\star\) is a weighted median of the values
\(\{H_A(G_i)\}_{i=1}^{N}\), with each value weighted by the corresponding
gap length \(\ell_i\) \citep{delon2010fast,rabin2011transportation}.
Since all values \(H_A(G_i)\) are integers, an integer-valued minimizing
median may be chosen. 

The same representation admits a direct interpretation as a minimum-cost
flow problem on the cycle. Orient each original support gap \(G_i\)
clockwise, and let \(f_i\) denote the signed amount of mass flowing through
that gap, with \(f_i>0\) corresponding to clockwise flow. Let
\[
b_i
=
\begin{cases}
+1, & z_i\in A\cap\supp\mu,\\
-1, & z_i\in A\cap\supp\nu,\\
0,  & z_i\notin A,
\end{cases}
\]
denote the signed activity at the \(i\)th support point. Conservation of
mass imposes the discrete divergence constraints
\[
f_i-f_{i-1}=wb_i,
\]
with cyclic indexing. These constraints determine the differences between
successive edge flows, but they do not determine their common additive
offset: adding the same constant circulation to every \(f_i\) preserves all
divergence constraints. Consequently, every feasible circular flow can be
written as
\[
f_i(a)=w\bigl(H_A(G_i)-a\bigr)
\]
for some scalar circulation level \(a\). Since transporting an amount
\(\lvert f_i(a)\rvert\) across a gap of length \(\ell_i\) incurs cost
\(\ell_i\lvert f_i(a)\rvert\), the minimum flow cost is
\[
\min_{a\in\RR}\sum_{i=1}^{N}\ell_i\lvert f_i(a)\rvert
=
w\min_{a\in\RR}
\sum_{i=1}^{N}\ell_i
\left|H_A(G_i)-a\right|,
\]
which recovers~\eqref{eq:circular-median}.

This flow formulation makes the freedom in the
choice of cut explicit: the circulation level \(a\) is a free primal variable,
and every value of it gives a feasible flow. Because the minimizing circulation \(a^\star\) may
be chosen as a weighted median of the finitely many values
\(\{H_A(G_i)\}_{i=1}^{N}\), it can be chosen equal to
\(H_A(G_r)\) for at least one support gap \(G_r\). The corresponding optimal
flow then satisfies
\[
f_r=w\bigl(H_A(G_r)-a^\star\bigr)=0.
\]
No mass crosses \(G_r\), so cutting the circle anywhere inside this gap
does not interrupt the optimal transport flow and reduces the circular
problem to an equivalent transport problem on the line
\citep{delon2010fast,rabin2011transportation}.

\begin{proposition}[Cut representation for a fixed active set]
\label{prop:fixed-cut}
For each original support gap \(G_r\), choose any
\(\theta_r\in G_r\), cut the circle at \(\theta_r\), and unwrap it onto
an interval of length \(L\). Let \(W_{1,r}^{\mathrm{line}}(A)\) denote
the line \(W_1\) cost of the active set \(A\) in this unwrapped
representation. Then, for every balanced active set \(A\),
\begin{equation}
W_{1,\circ}(A)
=
\min_{r=1,\ldots,N}
W_{1,r}^{\mathrm{line}}(A).
\label{eq:fixed-active-cut}
\end{equation}
The value \(W_{1,r}^{\mathrm{line}}(A)\) does not depend on the precise
location of \(\theta_r\) inside \(G_r\), because moving the cut within a
gap changes all unwrapped coordinates by the same translation and does
not alter their linear order.
\end{proposition}

\begin{proof}
A cut in gap $G_r$ fixes the flow through that gap to zero, which corresponds in \eqref{eq:circular-median} to choosing $a=H_A(G_r)$.  Conversely, an attained minimizing median $a^\star$ equals $H_A$ on at least one positive-length gap; cutting there realizes the circular optimum as a line optimum.
\end{proof}

\begin{corollary}[Exact cut envelope]
\label{cor:cut-envelope}
For each support gap \(G_r\), let
\(C_{k,r}^{\mathrm{line}}\) denote the minimum transportation cost,
with the absolute-value ground cost on the unwrapped line, among all
matchings that pair exactly \(k\) source atoms with \(k\) target atoms.
Then
\begin{equation}
C_k^{\circ}
=
\min_{r=1,\dots,N}
C_{k,r}^{\mathrm{line}}.
\label{eq:cut-envelope}
\end{equation}
Consequently, running \PAWL{} for all \(N\) possible cuts yields an exact
\(O(N^2\log N)\) reference solver for the complete circular
partial-transport profile.
\end{corollary}

\begin{proof}
Using \cref{prop:fixed-cut} and exchanging two finite minima,
\[
C_k^{\circ}
=
\min_{A:|A\cap\supp\mu|=|A\cap\supp\nu|=k}
\min_r W_{1,r}^{\mathrm{line}}(A)
=
\min_r C_{k,r}^{\mathrm{line}}.
\]
Each line run returns all $k$ in $O(N\log N)$ time \citep{chapel2025one}.
\end{proof}

\section{Nested optimal active sets and cyclic neighbours}
\label{app:nested}

The purpose of this section is to establish the two structural properties that
make an efficient incremental algorithm possible on the circle. First, we show
that optimal solutions can be selected consistently across transported
cardinalities: starting from an optimal cardinality-\(k\) solution, one can
construct an optimal cardinality-\((k+1)\) solution by activating exactly one
additional source atom and one additional target atom. Thus, the associated
active sets may be chosen to form a nested sequence. Second, we exploit the
one-dimensional geometry of the circle to sharply restrict which inactive
source--target pair can be activated next. We show that the two
newly activated atoms may be chosen to be consecutive in the cyclic ordering
of the inactive support points.

These results replace the quadratic search over all inactive
source--target pairs by a linear-size family of local candidates. The
nestedness property is derived first from the successive-augmentation
structure of min-cost flow, while the cyclic-neighbor property is obtained by
cutting the circle at a suitable gap and reducing the transition to the
corresponding line problem studied by \PAWL{}.

\subsection{Nested optimal active sets}
\label{sec:nested-active-sets}

A priori, optimal matchings at two consecutive cardinalities need not appear
compatible: an optimal cardinality-\(k\) matching and an independently
computed optimal cardinality-\((k+1)\) matching may involve substantially
different source and target atoms. For an incremental algorithm, however, we
need a stronger statement. Namely, we need to know that the optima can be
\emph{chosen} so that increasing the transported cardinality preserves every
previously active atom and activates exactly one additional source atom and
one additional target atom. This nestedness property is the analogue of the
monotone active-set structure exploited by PAWL on the line
\citep{chapel2025one}.

For a matching \(M\), recall that its active set is
\[
A(M)
\coloneqq
\left\{x_i:\exists\,y_j\text{ such that }(x_i,y_j)\in M\right\}
\cup
\left\{y_j:\exists\,x_i\text{ such that }(x_i,y_j)\in M\right\}.
\]
Thus, if \(M\) has cardinality \(k\), then \(A(M)\) contains exactly \(k\)
source atoms and \(k\) target atoms.

\begin{theorem}[Nested optimal extension]
\label{thm:nested}
Let \(M_k\) be any optimal cardinality-\(k\) matching, with \(k<K\).
Then there exists an optimal cardinality-\((k+1)\) matching
\(M_{k+1}\) such that
\begin{equation}
A(M_{k+1})
=
A(M_k)\cup\{x,y\},
\label{eq:nested-extension}
\end{equation}
where \(x\in A(M_k)^{c}\) is a source atom and
\(y\in A(M_k)^{c}\) is a target atom.
\end{theorem}
\begin{proof}
We represent the cardinality-constrained matching problem
\eqref{eq:integer-problem} as a unit-capacity minimum-cost flow problem.
Introduce a super-source node \(s\), a super-sink node \(t\), one node for
each source atom \(x_i\), and one node for each target atom \(y_j\). The
network contains the arcs
\[
s\longrightarrow x_i,
\qquad
x_i\longrightarrow y_j,
\qquad
y_j\longrightarrow t,
\]
all with unit capacity. The arcs adjacent to \(s\) and \(t\) have zero
cost, while the source-to-target arc \(x_i\to y_j\) has cost
\[
c_{ij}=w\,\dcirc(x_i,y_j).
\]
Because every arc has integral capacity, an integral flow of value \(k\)
selects exactly \(k\) source-to-target arcs, with at most one selected arc
incident to each source or target atom. It therefore corresponds exactly to
a cardinality-\(k\) matching. Conversely, every cardinality-\(k\) matching
defines an integral \(s\)-\(t\) flow of value \(k\). Hence the minimum-cost
flow problem of value \(k\) is equivalent to
\eqref{eq:integer-problem}, and the given optimal matching \(M_k\)
corresponds to an optimal integral flow \(f_k\) of value \(k\).

Consider the residual network of \(f_k\). For every unused
source-to-target arc \(x_i\to y_j\), the residual network contains a forward
arc of cost \(c_{ij}\). For every matched pair
\((x_i,y_j)\in M_k\), it contains the reverse residual arc
\[
y_j\longrightarrow x_i
\]
of cost \(-c_{ij}\). The latter arc represents the possibility of removing
the currently selected matching edge during an augmentation.

Since \(f_k\) is optimal, its residual network contains no
negative-cost directed cycle. Moreover, because \(k<K\), there exist an
unmatched source atom \(x_i\) and an unmatched target atom \(y_j\); the
path \(s\to x_i\to y_j\to t\) is then residual, so an \(s\)--\(t\)
residual path exists and a minimum-cost one is well defined.
Equivalently, one may introduce feasible node potentials and compute the
path using nonnegative reduced costs. By the successive shortest
augmenting-path theorem, augmenting one unit of flow along such a path
produces an optimal flow \(f_{k+1}\) of value \(k+1\)
\citep{ahuja1993network}. Importantly, this conclusion requires only
that \(f_k\) be optimal, and not that it have been obtained through any
particular sequence of previous augmentations. We may choose the
shortest augmenting path to be simple, so it does not revisit the
super-source or super-sink.

The first atom visited after \(s\) must be an unmatched source atom. Indeed,
the arc \(s\to x_i\) has forward residual capacity only when \(x_i\) is not
already matched. Similarly, the final atom visited before \(t\) must be an
unmatched target atom. Denote these two endpoints by \(x\) and \(y\),
respectively.

Between these endpoints, the augmenting path alternates between unused
forward matching edges and reversed currently matched edges. Its
source--target portion therefore has the form
\[
x
\longrightarrow y_1
\longrightarrow x_1
\longrightarrow y_2
\longrightarrow x_2
\longrightarrow
\cdots
\longrightarrow y,
\]
where each forward edge \(x_r\to y_{r+1}\) is inserted into the matching and
each reverse edge \(y_r\to x_r\) removes an edge of \(M_k\). Augmenting along
the path is thus equivalent to toggling the matching status of all edges on
this alternating path.

Every internal source or target atom on the path loses one matched edge but
simultaneously gains another. Consequently, it remains matched after the
augmentation. Atoms that are not on the path are unaffected. Only the two
endpoints \(x\) and \(y\) gain a matched edge without losing one, because
they were unmatched before the augmentation. Therefore, the resulting
optimal matching \(M_{k+1}\) satisfies
\[
A(M_{k+1})
=
A(M_k)\cup\{x,y\},
\]
which proves \eqref{eq:nested-extension}.
\end{proof}

\begin{corollary}[Nested optimal sequence]
\label{cor:nested-sequence}
There exist optimal matchings
\[
M_0,M_1,\ldots,M_K,
\qquad
K=\min\{n,m\},
\]
such that
\[
A(M_k)\subset A(M_{k+1})
\qquad
\text{for every }k<K,
\]
where \(A(M_{k+1})\setminus A(M_k)\) consists of exactly one source atom
and one target atom.
\end{corollary}
\begin{proof}
The empty matching \(M_0\) is the unique, hence optimal,
cardinality-\(0\) matching. Applying \cref{thm:nested} successively for
\(k=0,\ldots,K-1\) constructs the desired sequence.
\end{proof}

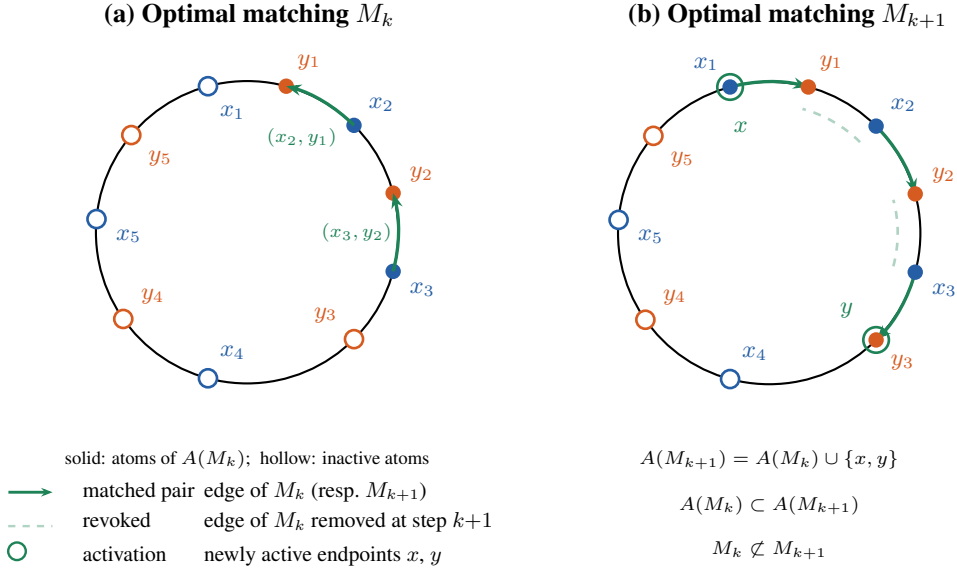
\begin{figure}[t]
\centering
\begin{minipage}[t]{0.46\linewidth}
\centering
\textbf{(a) Optimal matching $M_k$}\par\medskip
\begin{tikzpicture}[
  scale=1.0,font=\small,
  srcpt/.style={circle,draw=src,fill=src,inner sep=1.8pt,line width=.7pt},
  srcinactive/.style={circle,draw=src,fill=white,inner sep=2.4pt,line width=1pt},
  tgtpt/.style={circle,draw=tgt,fill=tgt,inner sep=1.8pt,line width=.7pt},
  tgtinactive/.style={circle,draw=tgt,fill=white,inner sep=2.4pt,line width=1pt},
  matcharc/.style={draw=transportcol,line width=1.25pt,-{Stealth[length=2mm]}},
  revoked/.style={draw=transportcol!35,dashed,line width=1pt},
  newactive/.style={circle,draw=transportcol,line width=1pt,inner sep=3.4pt}
]
  \def\R{2.0}
  \draw[line width=.8pt] (0,0) circle (\R);
  \node[tgtpt,label={[tgt]15:$y_2$}]  at (15:\R)  {};
  \node[srcpt,label={[src]45:$x_2$}]  at (45:\R)  {};
  \node[tgtpt,label={[tgt]75:$y_1$}]  at (75:\R)  {};
  \node[srcpt,label={[src]345:$x_3$}] at (345:\R) {};
  \node[srcinactive,label={[src]285:$x_1$}] at (105:\R) {};
  \node[tgtinactive,label={[tgt]320:$y_5$}] at (140:\R) {};
  \node[srcinactive,label={[src]355:$x_5$}] at (175:\R) {};
  \node[tgtinactive,label={[tgt]35:$y_4$}]  at (215:\R) {};
  \node[srcinactive,label={[src]75:$x_4$}]  at (255:\R) {};
  \node[tgtinactive,label={[tgt]135:$y_3$}] at (315:\R) {};
  \draw[matcharc] (45:\R)  arc[start angle=45, end angle=75, radius=\R];
  \draw[matcharc] (345:\R) arc[start angle=345,end angle=375,radius=\R];
  \node[transportcol,font=\scriptsize,fill=white,inner sep=1pt] at (60:1.45)  {$(x_2,y_1)$};
  \node[transportcol,font=\scriptsize,fill=white,inner sep=1pt] at (0:1.45)   {$(x_3,y_2)$};
  \node[font=\scriptsize,align=center] at (0,-3.0)
    {solid: atoms of $A(M_k)$; \ hollow: inactive atoms};
\end{tikzpicture}
\medskip
{\small
\resizebox{\linewidth}{!}{%
\begin{tabular}{@{}ll@{\hspace{4pt}}l@{}}
\tikz\draw[transportcol,line width=1.2pt,-{Stealth[length=1.7mm]}] (0,0)--(.65,0); & matched pair & edge of $M_k$ (resp.\ $M_{k+1}$)\\[1mm]
\tikz\draw[transportcol!35,dashed,line width=1pt] (0,0)--(.65,0); & revoked & edge of $M_k$ removed at step $k{+}1$\\[1mm]
\tikz\node[circle,draw=transportcol,line width=1pt,inner sep=2.6pt]{}; & activation & newly active endpoints $x$, $y$
\end{tabular}}}
\end{minipage}\hfill
\begin{minipage}[t]{0.51\linewidth}
\centering
\textbf{(b) Optimal matching $M_{k+1}$}\par\medskip
\begin{tikzpicture}[
  scale=1.0,font=\small,
  srcpt/.style={circle,draw=src,fill=src,inner sep=1.8pt,line width=.7pt},
  srcinactive/.style={circle,draw=src,fill=white,inner sep=2.4pt,line width=1pt},
  tgtpt/.style={circle,draw=tgt,fill=tgt,inner sep=1.8pt,line width=.7pt},
  tgtinactive/.style={circle,draw=tgt,fill=white,inner sep=2.4pt,line width=1pt},
  matcharc/.style={draw=transportcol,line width=1.25pt,-{Stealth[length=2mm]}},
  revoked/.style={draw=transportcol!35,dashed,line width=1pt},
  newactive/.style={circle,draw=transportcol,line width=1pt,inner sep=3.4pt}
]
  \def\R{2.0}
  \draw[line width=.8pt] (0,0) circle (\R);
  \draw[revoked] (45:\R-0.3)  arc[start angle=45, end angle=75, radius=\R-0.3];
  \draw[revoked] (345:\R-0.3) arc[start angle=345,end angle=375,radius=\R-0.3];
  \draw[matcharc] (105:\R) arc[start angle=105,end angle=75, radius=\R];
  \draw[matcharc] (45:\R)  arc[start angle=45, end angle=15, radius=\R];
  \draw[matcharc] (345:\R) arc[start angle=345,end angle=315,radius=\R];
  \node[newactive] at (105:\R) {};
  \node[newactive] at (315:\R) {};
  \node[transportcol,font=\small] at (105:1.45) {$x$};
  \node[transportcol,font=\small] at (315:1.45) {$y$};
  \node[tgtpt,label={[tgt]15:$y_2$}]   at (15:\R)  {};
  \node[srcpt,label={[src]45:$x_2$}]   at (45:\R)  {};
  \node[tgtpt,label={[tgt]75:$y_1$}]   at (75:\R)  {};
  \node[srcpt,label={[src]105:$x_1$}]  at (105:\R) {};
  \node[tgtpt,label={[tgt]315:$y_3$}]  at (315:\R) {};
  \node[srcpt,label={[src]345:$x_3$}]  at (345:\R) {};
  \node[tgtinactive,label={[tgt]320:$y_5$}] at (140:\R) {};
  \node[srcinactive,label={[src]355:$x_5$}] at (175:\R) {};
  \node[tgtinactive,label={[tgt]35:$y_4$}]  at (215:\R) {};
  \node[srcinactive,label={[src]75:$x_4$}]  at (255:\R) {};
  \node[font=\scriptsize,align=center] at (0,-3.0)
    {$A(M_{k+1})=A(M_k)\cup\{x,y\}$};
  \node[font=\scriptsize,align=center] at (0,-3.6)
    {$A(M_k)\subset A(M_{k+1})$};
    \node[font=\scriptsize,align=center] at (0,-4.2)
    {$M_k\not\subset M_{k+1}$};
\end{tikzpicture}
\end{minipage}
\caption{Nested optimal active sets (Theorem~\ref{thm:nested}).
\textbf{(a)}~An optimal cardinality-$k$ matching $M_k$ on the circle, with
matched pairs $(x_2,y_1)$ and $(x_3,y_2)$ transported along geodesic arcs;
solid atoms form the active set $A(M_k)$ and hollow atoms are inactive.
\textbf{(b)}~An optimal cardinality-$(k{+}1)$ matching obtained by augmenting
along a minimum-cost alternating path: exactly one source atom $x=x_1$ and one
target atom $y=y_3$ are activated (circled), and they are cyclic neighbors
among the inactive atoms since every atom between them is already active.
The augmentation revokes the edges $(x_2,y_1)$ and $(x_3,y_2)$ of $M_k$
(dashed) and rematches the previously active atoms as
$(x_1,y_1)$, $(x_2,y_2)$, $(x_3,y_3)$, so that
$A(M_{k+1})=A(M_k)\cup\{x,y\}$ while $M_k\not\subset M_{k+1}$
(Remark~\ref{rem:edge-nestedness}).}
\label{fig:nested-active-sets}
\end{figure}

\begin{remark}[Active-set nestedness versus edge nestedness]
\label{rem:edge-nestedness}
The nested-extension property concerns the active atoms, not necessarily
the matching edges. Although
\[
A(M_k)\subset A(M_{k+1}),
\]
one need not have
\[
M_k\subset M_{k+1}.
\]
Indeed, the augmenting path in the proof of \cref{thm:nested} may remove
several edges of \(M_k\) and replace them with different edges; \cref{fig:nested-active-sets}
shows an instance in which two edges of \(M_k\) are revoked and three new ones
created, while every previously active atom remains active. Every
internal atom on the augmenting path nevertheless remains matched, since
it loses one incident matching edge and gains another. Consequently, all
previously active source and target atoms remain active, while only the
two endpoints of the augmenting path become newly active.

This distinction is important for the developments below: the local
rematching of previously active atoms is precisely what gives rise to the
chain and cell marginal costs. Moreover, \cref{thm:nested} is stronger
than the mere existence of one nested optimal sequence, since it applies
to any prescribed optimal cardinality-\(k\) matching. This stronger
extension form is essential for the greedy analysis, where the theorem
must be applied to the particular optimal active set produced by the
algorithm at iteration \(k\).
\end{remark}

\begin{remark}[Independence from the circular geometry]
\label{rem:nested-cost-independent}
The proof of \cref{thm:nested} uses only the bipartite matching and
minimum-cost flow structure; it does not rely on any special property of
the geodesic cost or of the circle. Hence, the nested optimal-extension
property holds for an arbitrary bipartite cost matrix. The circular
geometry becomes essential only in the next step, where we show that the
two newly activated atoms are cyclic neighbors among the inactive atoms
and that the resulting cost increment can be evaluated locally.
\end{remark}

\subsection{A pairwise common-cut lemma}

The next lemma is the bridge between nested circular optima and the line neighbor theorem.

\begin{lemma}[Common optimal cut]
\label{lem:common-cut}
Let $A$ be balanced and let
\[
A'=A\cup\{x,y\},
\]
where $x$ is a source and $y$ is a target.  There exists a support gap $G$ that carries zero flow in an optimal circular $W_1$ solution for both $A$ and $A'$.  Equivalently, cutting in $G$ realizes both circular costs as line costs.
\end{lemma}

\begin{proof}
Choose an orientation from $x$ to $y$ and let $I$ be the corresponding circular arc.  Up to adding a constant to the cumulative imbalance of $A'$, which is absorbed by the minimization in \eqref{eq:circular-median}, we may write
\begin{equation}
H_{A'}=H_A+\ind_I.
\label{eq:one-interval-update}
\end{equation}
Let $a$ be an attained weighted median of $H_A$, so
\begin{equation}
|\{H_A<a\}|\leq L/2,
\qquad
|\{H_A>a\}|\leq L/2,
\label{eq:median-conditions}
\end{equation}
where $|\cdot|$ denotes circular arc length.  Set $E=\{H_A=a\}$.

If $a$ is also a median of $H_{A'}$ and $|E\setminus I|>0$, choose a gap in $E\setminus I$.  On that gap, both cumulative functions equal $a$.

Otherwise, we show that $a+1$ is a median of $H_{A'}$ and that $|E\cap I|>0$.  If $a$ is not a median of $H_{A'}$, then because $H_{A'}\geq H_A$, only the upper median condition can fail.  Hence $|\{H_{A'}>a\}|>L/2$, and the excess over $\{H_A>a\}$ must come from $E\cap I$, which therefore has positive length.  Since $H_{A'}$ is integer-valued,
\[
\{H_{A'}<a+1\}=\{H_{A'}\leq a\}
\]
has length below $L/2$, while
\[
\{H_{A'}>a+1\}\subseteq\{H_A>a\}
\]
has length at most $L/2$.  Thus $a+1$ is a median.

The remaining case is that $a$ is a median of $H_{A'}$ but $|E\setminus I|=0$.  Then $|E\cap I|>0$.  Moreover,
\[
\{H_{A'}<a+1\}\subseteq\{H_A<a\}\cup(E\setminus I),
\qquad
\{H_{A'}>a+1\}\subseteq\{H_A>a\},
\]
so \eqref{eq:median-conditions} again shows that $a+1$ is a median of $H_{A'}$.

Choose a positive-length gap in $E\cap I$.  On it, $H_A=a$ and $H_{A'}=a+1$.  Centering each cumulative imbalance at its respective median gives zero flow for both active sets on the chosen gap.  \Cref{fig:common-cut} shows this second case: the two cumulative imbalances sit at different medians, and the gap chosen from $E\cap I$ is flow-free for both.
\end{proof}

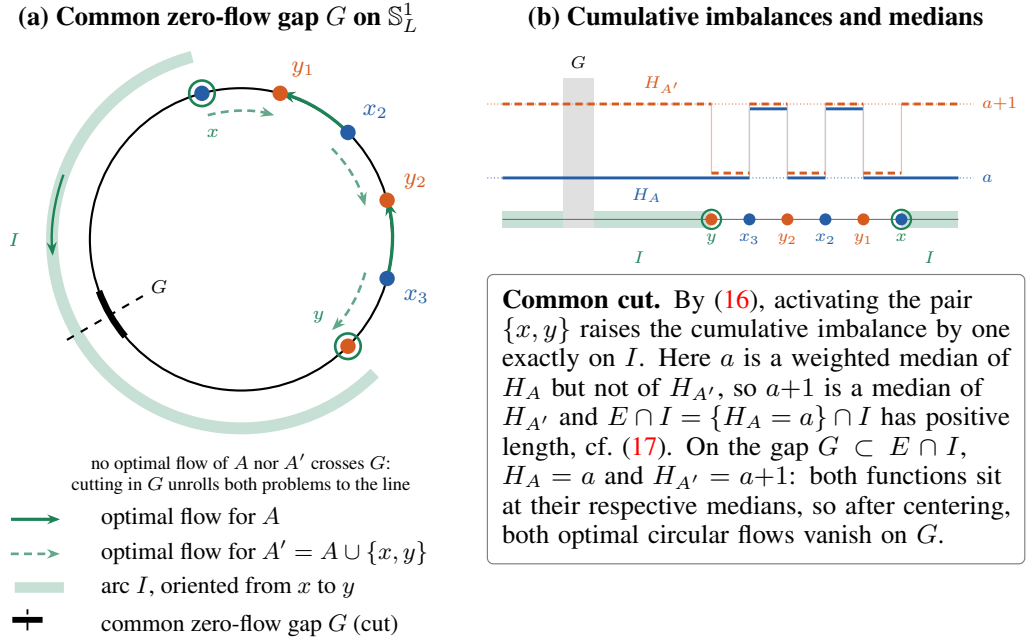
\begin{figure}[t!]
\centering
\begin{minipage}[t]{0.46\linewidth}
\centering
\textbf{(a) Common zero-flow gap $G$ on $\mathbb S^1_L$}\par\medskip
\begin{tikzpicture}[
  scale=1.0,font=\small,
  srcpt/.style={circle,draw=src,fill=src,inner sep=1.8pt,line width=.7pt},
  tgtpt/.style={circle,draw=tgt,fill=tgt,inner sep=1.8pt,line width=.7pt},
  flowA/.style={draw=transportcol,line width=1.25pt,-{Stealth[length=2mm]}},
  flowAp/.style={draw=transportcol!70,densely dashed,line width=1pt,-{Stealth[length=1.8mm]}},
  newactive/.style={circle,draw=transportcol,line width=1pt,inner sep=3.4pt}
]
  \def\R{2.0}
  \draw[transportcol!25,line width=4.5pt]
    (105:\R+0.5) arc[start angle=105,end angle=315,radius=\R+0.5];
  \draw[transportcol,line width=.9pt,-{Stealth[length=1.8mm]}]
    (160:\R+0.5) arc[start angle=160,end angle=185,radius=\R+0.5];
  \node[transportcol,font=\scriptsize] at (180:\R+1.) {$I$};
  \draw[line width=.8pt] (0,0) circle (\R);
  \draw[black,line width=2.5pt] (200:\R) arc[start angle=200,end angle=220,radius=\R];
  \draw[black,dashed,line width=.8pt] (210:1.5) -- (210:2.75);
  \node[font=\scriptsize] at (210:1.25) {$G$};
  \draw[flowA] (45:\R)  arc[start angle=45, end angle=75, radius=\R];
  \draw[flowA] (345:\R) arc[start angle=345,end angle=375,radius=\R];
  \draw[flowAp] (105:\R-0.3) arc[start angle=105,end angle=75, radius=\R-0.3];
  \draw[flowAp] (45:\R-0.3)  arc[start angle=45, end angle=15, radius=\R-0.3];
  \draw[flowAp] (345:\R-0.3) arc[start angle=345,end angle=315,radius=\R-0.3];
  \node[tgtpt,label={[tgt]15:$y_2$}]  at (15:\R)  {};
  \node[srcpt,label={[src]45:$x_2$}]  at (45:\R)  {};
  \node[tgtpt,label={[tgt]75:$y_1$}]  at (75:\R)  {};
  \node[srcpt,label={[src]345:$x_3$}] at (345:\R) {};
  \node[newactive] at (105:\R) {};
  \node[newactive] at (315:\R) {};
  \node[srcpt] at (105:\R) {};
  \node[tgtpt] at (315:\R) {};
  \node[transportcol,font=\scriptsize] at (105:1.45) {$x$};
  \node[transportcol,font=\scriptsize] at (315:1.45) {$y$};
  \node[font=\scriptsize,align=center] at (0,-3.1)
    {no optimal flow of $A$ nor $A'$ crosses $G$:\\ cutting in $G$ unrolls both problems to the line};
\end{tikzpicture}
\medskip
{\small
\begin{tabular}{@{}ll@{}}
\tikz\draw[transportcol,line width=1.2pt,-{Stealth[length=1.7mm]}] (0,0)--(.65,0); & optimal flow for $A$\\[1mm]
\tikz\draw[transportcol!70,densely dashed,line width=1pt,-{Stealth[length=1.7mm]}] (0,0)--(.65,0); & optimal flow for $A'=A\cup\{x,y\}$\\[1mm]
\tikz\draw[transportcol!25,line width=4pt] (0,0)--(.65,0); & arc $I$, oriented from $x$ to $y$\\[1mm]
\tikz{\draw[black,line width=2.2pt] (0,0)--(.4,0);\draw[black,dashed,line width=.7pt] (.2,-.14)--(.2,.14);} & common zero-flow gap $G$ (cut)
\end{tabular}}
\end{minipage}\hfill
\begin{minipage}[t]{0.51\linewidth}
\centering
\textbf{(b) Cumulative imbalances and medians}\par\medskip
\resizebox{\linewidth}{!}{%
\begin{tikzpicture}[x=0.55cm,y=1.0cm,font=\small,
  srcpt/.style={circle,draw=src,fill=src,inner sep=1.5pt,line width=.7pt},
  tgtpt/.style={circle,draw=tgt,fill=tgt,inner sep=1.5pt,line width=.7pt},
  newactive/.style={circle,draw=transportcol,line width=.9pt,inner sep=2.8pt}
]
  \fill[transportcol!25] (0,-0.72) rectangle (5.5,-0.48);
  \fill[transportcol!25] (10.5,-0.72) rectangle (12,-0.48);
  \node[transportcol,font=\scriptsize] at (3.6,-1.15) {$I$};
  \node[transportcol,font=\scriptsize] at (11.25,-1.15) {$I$};
  \fill[black!12] (1.6,-0.72) rectangle (2.4,1.45);
  \node[font=\scriptsize] at (2.0,1.68) {$G$};
  \draw[src,densely dotted] (-0.4,0) -- (12.4,0)
    node[right,font=\scriptsize] {$a$};
  \draw[tgt,densely dotted] (-0.4,1.07) -- (12.4,1.07)
    node[right,font=\scriptsize] {$a{+}1$};
  \draw[src!40,line width=.5pt] (6.5,0)--(6.5,1) (7.5,0)--(7.5,1)
                                (8.5,0)--(8.5,1) (9.5,0)--(9.5,1);
  \draw[src,line width=1.2pt]
    (0,0)--(6.5,0) (6.5,1)--(7.5,1) (7.5,0)--(8.5,0) (8.5,1)--(9.5,1) (9.5,0)--(12,0);
  \draw[tgt!40,line width=.5pt] (5.5,0.07)--(5.5,1.07) (6.5,0.07)--(6.5,1.07)
    (7.5,0.07)--(7.5,1.07) (8.5,0.07)--(8.5,1.07) (9.5,0.07)--(9.5,1.07)
    (10.5,0.07)--(10.5,1.07);
  \draw[tgt,densely dashed,line width=1.2pt]
    (0,1.07)--(5.5,1.07) (5.5,0.07)--(6.5,0.07) (6.5,1.07)--(7.5,1.07)
    (7.5,0.07)--(8.5,0.07) (8.5,1.07)--(9.5,1.07) (9.5,0.07)--(10.5,0.07)
    (10.5,1.07)--(12,1.07);
  \draw[black!60] (0,-0.6) -- (12,-0.6);
  \node[newactive] at (5.5,-0.6) {};
  \node[newactive] at (10.5,-0.6) {};
  \node[tgtpt,label={[transportcol,font=\scriptsize]below:$y$}]   at (5.5,-0.6)  {};
  \node[srcpt,label={[src,font=\scriptsize]below:$x_3$}]          at (6.5,-0.6)  {};
  \node[tgtpt,label={[tgt,font=\scriptsize]below:$y_2$}]          at (7.5,-0.6)  {};
  \node[srcpt,label={[src,font=\scriptsize]below:$x_2$}]          at (8.5,-0.6)  {};
  \node[tgtpt,label={[tgt,font=\scriptsize]below:$y_1$}]          at (9.5,-0.6)  {};
  \node[srcpt,label={[transportcol,font=\scriptsize]below:$x$}]   at (10.5,-0.6) {};
  \node[src,font=\scriptsize] at (3.8,-0.25) {$H_A$};
  \node[tgt,font=\scriptsize] at (4.2,1.35)  {$H_{A'}$};
\end{tikzpicture}}
\medskip
\resizebox{\linewidth}{!}{%
\begin{tikzpicture}[font=\small]
\node[draw=black!50,rounded corners=2pt,align=left,inner sep=5pt,text width=.85\linewidth] {
\textbf{Common cut.}  By \eqref{eq:one-interval-update}, activating the pair $\{x,y\}$ raises the cumulative imbalance by one exactly on $I$.  Here $a$ is a weighted median of $H_A$ but not of $H_{A'}$, so $a{+}1$ is a median of $H_{A'}$ and $E\cap I=\{H_A=a\}\cap I$ has positive length, cf.\ \eqref{eq:median-conditions}.  On the gap $G\subset E\cap I$, $H_A=a$ and $H_{A'}=a{+}1$: both functions sit at their respective medians, so after centering, both optimal circular flows vanish on $G$.};
\end{tikzpicture}}
\end{minipage}
\caption{Common optimal cut (Lemma~\ref{lem:common-cut}).
\textbf{(a)}~A balanced active set $A=\{x_2,y_1,x_3,y_2\}$ and its extension
$A'=A\cup\{x,y\}$ on the circle.  The arc $I$, oriented from $x$ to $y$,
supports the one-interval update $H_{A'}=H_A+\ind_I$.  The support gap $G$
carries zero flow in an optimal circular $W_1$ solution for both $A$ (solid
arcs) and $A'$ (dashed arcs), so cutting in $G$ realizes both circular costs
as line costs.
\textbf{(b)}~The cumulative imbalances, unrolled at a point of the large
support gap.  $H_A$ is flat at its median $a$ on $E=\{H_A=a\}$, while
$H_{A'}=H_A+\ind_I$ is flat at its median $a{+}1$ on $E\cap I$; the gap
$G\subset E\cap I$ is at median level for both, hence flow-free for both.}
\label{fig:common-cut}
\end{figure}

\subsection{Cyclic neighbor addition}
Nestedness alone does not yet yield an efficient update rule.  Although
\cref{thm:nested} guarantees that an optimal cardinality-\((k+1)\) active set
can be obtained from an optimal cardinality-\(k\) active set by adding one
source atom and one target atom, a direct search would still have to examine
all opposite-type pairs in \(A_k^c\), resulting in a quadratic number of
candidates.  The circular order provides a much stronger restriction: the two
new atoms may be chosen adjacent after the currently active atoms are removed.
Thus, at iteration \(k\), it is sufficient to consider only opposite-type
pairs that are consecutive in the cyclic ordering of \(A_k^c\).  This reduces
the candidate family to linear size and, more importantly, identifies each
candidate with a single circular cell whose interior contains only active
atoms.  The result is the circular counterpart of the locality property
underlying \PAWL{} on the line.

\begin{theorem}[Cyclic-neighbor property of optimal extensions]
\label{thm:cyclic-neighbor}
Let \(M_k\) be any circularly optimal cardinality-\(k\) matching, with
\(k<K\), and write \(A_k\coloneqq A(M_k)\). Let \(M_{k+1}\) be any
circularly optimal cardinality-\((k+1)\) matching satisfying
\begin{equation}
A(M_{k+1})=A_k\cup\{x,y\},
\label{eq:neighbor-extension}
\end{equation}
where \(x\in A_k^c\) is a source atom and \(y\in A_k^c\) is a target
atom. Then \(x\) and \(y\) are consecutive in the cyclic ordering of
the inactive set \(A_k^c\).

In particular, by \cref{thm:nested}, every circularly optimal
cardinality-\(k\) matching admits a circularly optimal
cardinality-\((k+1)\) extension whose two newly active atoms are cyclic
neighbors.
\end{theorem}

\begin{proof}
Let \(M_k\) and \(M_{k+1}\) be as in the statement, and set
\[
A_{k+1}\coloneqq A(M_{k+1})
=
A_k\cup\{x,y\}.
\]
It remains to show that \(x\) and \(y\) are cyclically adjacent among
the atoms that are inactive at stage \(k\).

By \cref{lem:common-cut}, there exists an original support gap \(G_r\)
such that cutting the circle at any point \(\theta_r\in G_r\) and
unwrapping it onto an interval of length \(L\) realizes the circular
transport costs of both active sets as line transport costs:
\[
W^{\mathrm{line}}_{1,r}(A_k)
=
W_{1,\circ}(A_k),
\qquad
W^{\mathrm{line}}_{1,r}(A_{k+1})
=
W_{1,\circ}(A_{k+1}).
\]
Since \(M_k\) and \(M_{k+1}\) are circularly optimal,
\[
W_{1,\circ}(A_k)=C_k^\circ,
\qquad
W_{1,\circ}(A_{k+1})=C_{k+1}^\circ,
\]
and therefore
\[
W^{\mathrm{line}}_{1,r}(A_k)=C_k^\circ,
\qquad
W^{\mathrm{line}}_{1,r}(A_{k+1})=C_{k+1}^\circ.
\]

We next verify that \(A_k\) and \(A_{k+1}\) are optimal for the
corresponding cardinality-constrained partial-transport problems on the
unwrapped line. For clarity, let
\[
\operatorname{cost}^{\circ}(M)
=
w\sum_{(x_i,y_j)\in M}\dcirc(x_i,y_j)
\]
and let \(\operatorname{cost}^{\mathrm{line}}_r(M)\) denote the
analogous matching cost computed using the distances in the
representation obtained by cutting at \(G_r\).

Consider any cardinality-\(k\) matching \(M'\). For each matched pair,
the circular geodesic distance is no larger than its distance in the
fixed unwrapped representation. Hence
\[
\operatorname{cost}^{\mathrm{line}}_r(M')
\geq
\operatorname{cost}^{\circ}(M').
\]
Circular optimality of \(M_k\), together with the fact that the common
cut realizes its circular cost, gives
\[
\operatorname{cost}^{\mathrm{line}}_r(M')
\geq
\operatorname{cost}^{\circ}(M')
\geq
C_k^\circ
=
W^{\mathrm{line}}_{1,r}(A_k).
\]
Thus no cardinality-\(k\) line matching is cheaper than the optimal line
matching supported on \(A_k\). Therefore \(A_k\) is optimal for the
cardinality-\(k\) partial-transport problem induced by the cut. The same
argument, with \(k+1\) in place of \(k\), proves that \(A_{k+1}\) is
optimal for the corresponding cardinality-\((k+1)\) line problem.

We may now apply the universal line-neighbor property of
\cref{lem:line-locality} to the nested pair of line-optimal active sets
\[
A_k\subset A_{k+1}.
\]
It follows that the two newly activated atoms \(x\) and \(y\) are
consecutive in the linear ordering of the atoms in \(A_k^c\) induced by
the cut.

Finally, the linear inactive order is obtained by breaking the cyclic
inactive order at a single location. This operation can remove the one
cyclic adjacency crossing the cut, but it cannot create any new
adjacency. Hence every pair that is consecutive in the linear inactive
order is also consecutive in the original cyclic inactive order.
Therefore \(x\) and \(y\) are cyclic neighbors in \(A_k^c\).
\end{proof}

\section{Circular cells and local marginals}
\label{app:cells}

\label{sec:cells}

The previous section established the combinatorial structure of an optimal
cardinality-by-cardinality construction. In particular, optimal active sets can
be chosen to form a nested sequence, and the transition from \(A_k\) to
\(A_{k+1}\) may be realized by activating a source atom and a target atom that
are consecutive in the cyclic ordering of the inactive set \(A_k^c\). This
reduces the admissible updates from all opposite-type inactive pairs to a
linear-size family of cyclic-neighbor candidates.

This structural reduction, however, is not yet sufficient for an efficient
algorithm. For each candidate pair, one must still determine the increase in
the optimal transportation cost caused by activating its two endpoints.
Recomputing the circular optimal transport problem from scratch for every
candidate would destroy the desired near-linear complexity. The key remaining
question is therefore whether this cost increment can be expressed using only
the portion of the circle lying between the two candidate endpoints.

In this section, we show that the answer is affirmative. Consecutive inactive
points partition the circle into \emph{current cells}, whose interiors contain
only active atoms. A cyclic-neighbor pair of opposite type defines a candidate
cell, and activating its endpoints modifies the optimal matching only within
that cell after a suitable unwrapping onto the line. This leads to an exact
local marginal formula: the increase in circular transport cost is the
difference between the line transport cost of the full balanced cell and that
of its active interior. We then relate these cell costs to the chain
decomposition used by \PAWL{}, which will later permit constant-time marginal
evaluation after preprocessing.

Let $U_k=A_k^c$ be the current inactive set, written in cyclic order as
\[
u_1,u_2,\dots,u_r.
\]
Each directed pair of consecutive inactive points defines a \emph{cell}
\begin{equation}
\mathcal C_i=\arc{u_i}{u_{i+1}},
\qquad u_{r+1}=u_1,
\label{eq:cell-def}
\end{equation}
where the arc is traversed clockwise and includes both endpoints.  The open cell interior contains only active points.

\begin{definition}[Candidate cell]
A current cell is a \emph{candidate} if its two inactive endpoints come from opposite measures.
\end{definition}

\begin{lemma}[Balanced cell interiors]
\label{lem:balanced-interiors}
At every iteration of the algorithm, the active atoms in the interior of each current cell contain the same number of source and target atoms.
\end{lemma}

\begin{proof}
Initially all cell interiors are empty.  Suppose the property holds and a candidate cell $\arc{u}{v}$ is selected.  Let $p$ be the previous inactive point before $u$ and $q$ the next inactive point after $v$.  Removing $u,v$ merges
\[
\arc{p}{u},\qquad \arc{u}{v},\qquad \arc{v}{q}
\]
into $\arc{p}{q}$.  The old interiors are balanced by induction, and the newly activated endpoints $u,v$ contribute one source and one target.  Therefore the new interior is balanced.  Unaffected cells are unchanged.
\end{proof}

For a current cell
\[
\mathcal C=\arc{u}{v},
\]
the endpoints \(u\) and \(v\) are consecutive inactive atoms in the
clockwise order. Hence, by definition of a cell, there is no inactive atom
in the open arc \(\mathcal C^\circ=\arcopen{u}{v}\); every support point in
its interior already belongs to the current active set \(A_k\). If
\(\mathcal C\) is a candidate cell, then \(u\) and \(v\) come from opposite
distributions. Moreover, \cref{lem:balanced-interiors} implies that the
active atoms in \(\mathcal C^\circ\) contain the same number of sources and
targets. It follows that both the interior \(\mathcal C^\circ\) and the
closed cell
\[
\mathcal C=\mathcal C^\circ\cup\{u,v\}
\]
are balanced: the former represents the local transport configuration
before activating \(u\) and \(v\), while the latter represents the local
configuration after their activation.

To compare these two configurations, choose a cut outside
\(\mathcal C\) and unwrap the circle onto an interval of length \(L\).
Because the cut does not intersect \(\mathcal C\), the clockwise arc
\(\mathcal C\) becomes a contiguous interval on the real line. Let
\[
\mu_{\mathcal C^\circ}
\quad\text{and}\quad
\nu_{\mathcal C^\circ}
\]
denote the source and target measures supported on the active interior
atoms, and let
\[
\mu_{\mathcal C}
\quad\text{and}\quad
\nu_{\mathcal C}
\]
denote the corresponding measures after including the two endpoints.
Both are balanced, so \eqref{eq:line-cost} applies to each: $c_{\mathbb R}(\mathcal C^\circ)$
is the cost of matching $\mu_{\mathcal C^\circ}$ against $\nu_{\mathcal C^\circ}$ in
increasing unwrapped order, and $c_{\mathbb R}(\mathcal C)$ the same for
$\mu_{\mathcal C}$ against $\nu_{\mathcal C}$.

\begin{definition}[Circular cell marginal]
\label{def:cell-marginal}
The \emph{local marginal cost} associated with a candidate cell
\(\mathcal C\) is
\begin{equation}
\mathfrak m(\mathcal C)
=
c_{\mathbb R}(\mathcal C)
-
c_{\mathbb R}(\mathcal C^\circ).
\label{eq:cell-marginal}
\end{equation}
Thus, \(\mathfrak m(\mathcal C)\) measures the increase in the local
transportation cost caused by activating the two inactive endpoints of
\(\mathcal C\) and recomputing the optimal matching inside the cell.
\end{definition}

The distance entering both terms in
\eqref{eq:cell-marginal} is the ordinary absolute distance in the
unwrapped interval, equivalently the clockwise arc length measured inside
\(\mathcal C\). It is not the circular geodesic distance
\(\dcirc\). Indeed, once a cut is fixed outside the cell,
\(\mathcal C\) is treated as a contiguous line interval, and the local
matching is constrained to remain within that interval. Using the
geodesic distance at this stage could allow a pair of atoms in
\(\mathcal C\) to be connected through the complementary arc, thereby
crossing the cut and destroying the desired locality.

The significance of \(\mathfrak m(\mathcal C)\) is that it is also the
increment in the global line transport cost. Before activation, the atoms
in \(\mathcal C^\circ\) are matched optimally among themselves. After
activating \(u\) and \(v\), those local matching edges may be revoked and
replaced by the optimal sorted matching on the larger balanced interval
\(\mathcal C\). By the one-dimensional locality principle, all matching
edges outside \(\mathcal C\) remain unchanged. Consequently, whenever the
chosen cut lies outside \(\mathcal C\),
\[
c_{\mathbb R}(A_k\cup\{u,v\})
-
c_{\mathbb R}(A_k)
=
c_{\mathbb R}(\mathcal C)
-
c_{\mathbb R}(\mathcal C^\circ)
=
\mathfrak m(\mathcal C).
\]
Hence, the effect of activating a candidate pair can be evaluated entirely
from the atoms contained in its cell.  \Cref{fig:cell-locality} shows the
rematching this describes: the sorted matching of the active interior is revoked
and replaced by the sorted matching of the closed cell, while every edge outside
the cell is untouched.

\section{Free gaps, the exact greedy theorem and the simultaneous cut}
\label{app:greedy}

\label{sec:freegaps}

Let the \emph{original gaps} be the $N$ arcs between consecutive points in the fully sorted union support.  A gap is called \emph{used} once it lies inside a selected candidate cell, and \emph{free} otherwise.  Since the endpoints of current cells are inactive atoms, hence support points, every original gap lies inside exactly one current cell, and which cell that is changes as cells merge.

\begin{lemma}[Free-gap invariant]
\label{lem:free-gap}
Before termination, every current cell contains at least one free original gap.
\end{lemma}

\begin{proof}
The labels \emph{free} and \emph{used} attach to the $N$ original gaps, which are
fixed once and for all, and not to the cells, which are not: a gap acquires the
label \emph{used} when it first lies inside a selected cell and never loses it.
Merging changes which cell a gap lies in; it never changes the gap's label.  This
is the point on which the induction turns, since a gap stays free while the cell
containing it grows.

Initially, each current cell is exactly one original gap, hence free.  Suppose a
candidate $\arc{u}{v}$ is selected.  This marks used exactly the gaps lying inside
$\arc{u}{v}$, and removes $u$ and $v$ from the inactive list, so the three cells
$\arc{p}{u}$, $\arc{u}{v}$, $\arc{v}{q}$ merge into the single cell $\arc{p}{q}$.
If the algorithm does not terminate, the arcs $\arc{p}{u}$ and $\arc{v}{q}$ are
contained in $\arc{p}{q}$, and each contained a free gap by induction.  Those gaps
lie outside $\arc{u}{v}$, so the selection did not mark them, and they are still
free afterwards --- the merged cell inherits two independent free gaps, of which
the algorithm need only store one (\cref{fig:merge-freegap}).  Every unaffected
cell retains its free gap.
\end{proof}

\begin{lemma}[Cut compatibility]
\label{lem:cut-compatibility}
Let $\mathcal C_1,\dots,\mathcal C_k$ be a sequence of selected cells, and let $\theta$ be an original gap contained in none of them.  After cutting at $\theta$, the line cost of the active set $A_k$ generated by these selections is
\begin{equation}
W_{1,\theta}^{\mathrm{line}}(A_k)
=
\sum_{r=1}^{k}\mathfrak m(\mathcal C_r).
\label{eq:cut-compatible-sum}
\end{equation}
If a current candidate cell $\mathcal C$ also avoids $\theta$, then activating its endpoints gives a line matching of cost
\begin{equation}
W_{1,\theta}^{\mathrm{line}}(A_k)+\mathfrak m(\mathcal C).
\end{equation}
\end{lemma}

\begin{proof}
No selected cell crosses the cut, so every selected circular arc becomes an ordinary line interval.  The claim follows inductively from \cref{lem:line-locality}: each activation changes the sorted line matching only inside its interval and adds exactly the local marginal.  The same argument applies to the additional candidate.
\end{proof}

\begin{theorem}[Exact greedy choice]
\label{thm:greedy}
Assume that the current active set \(A_k\) is circularly optimal and was
generated by successive cell selections. Then
\begin{equation}
C_{k+1}^{\circ}-C_k^{\circ}
=
\min_{\mathcal C\in\mathfrak C_k}
\mathfrak m(\mathcal C),
\label{eq:greedy-choice}
\end{equation}
where \(\mathfrak C_k\) is the set of current candidate cells whose
endpoints have opposite types. Moreover, activating the endpoints of any
minimizing cell yields a circularly optimal active set \(A_{k+1}\).
\end{theorem}

\begin{proof}
Write
\[
\delta_{k+1}
\coloneqq
C_{k+1}^{\circ}-C_k^{\circ}.
\]

We first show that every candidate marginal is at least
\(\delta_{k+1}\). Fix an arbitrary current candidate cell
\(\mathcal C\in\mathfrak C_k\). Choose a current cell different from
\(\mathcal C\) and, by \cref{lem:free-gap}, select a free support gap
\(\theta\) contained in that cell. If only two inactive atoms remain,
there are two complementary directed cells between them, and we choose
the cell different from \(\mathcal C\). Thus, in either case,
\(\theta\) lies outside \(\mathcal C\) and belongs to none of the cells
selected during the preceding iterations.

By \cref{lem:cut-compatibility}, cutting the circle at \(\theta\) is
compatible with the entire sequence of cell selections that generated
\(A_k\). Since \(A_k\) is circularly optimal by assumption, its line cost
under this cut satisfies
\[
W_{1,\theta}^{\mathrm{line}}(A_k)
=
C_k^\circ.
\]
Activating the two endpoints of \(\mathcal C\) changes only the matching
inside that cell. By the definition of the local cell marginal and the
line-locality property, this produces a feasible cardinality-\((k+1)\)
line matching of cost
\[
C_k^\circ+\mathfrak m(\mathcal C).
\]
The same matching is feasible on the circle, and its circular geodesic
cost is no larger than its cost in the fixed unwrapped representation.
Consequently,
\[
C_{k+1}^\circ
\leq
C_k^\circ+\mathfrak m(\mathcal C),
\]
or equivalently,
\begin{equation}
\delta_{k+1}
\leq
\mathfrak m(\mathcal C).
\label{eq:greedy-lower-bound}
\end{equation}
Since \(\mathcal C\) was arbitrary, this inequality holds for every
candidate cell in \(\mathfrak C_k\).

We next establish the reverse inequality by exhibiting a candidate cell
whose marginal equals \(\delta_{k+1}\). Let \(M_k\) be an optimal
cardinality-\(k\) matching supported on the particular active set
\(A_k\) produced by the greedy construction. Applying
\cref{thm:nested} to \(M_k\) yields a circularly optimal
cardinality-\((k+1)\) matching \(M_{k+1}^{\star}\) with active set
\[
A_{k+1}^{\star}
=
A_k\cup\{x^{\star},y^{\star}\},
\]
where \(x^{\star}\) is a previously inactive source atom and
\(y^{\star}\) is a previously inactive target atom. By
\cref{thm:cyclic-neighbor}, these two newly activated atoms are
consecutive in the cyclic ordering of \(A_k^c\). They therefore delimit
a current opposite-endpoint candidate cell, which we denote by
\(\mathcal C^{\star}\).

By \cref{lem:common-cut}, there exists an original support gap at which
the circle can be cut so that the circular costs of both \(A_k\) and
\(A_{k+1}^{\star}\) are realized as line costs. As established in the
proof of \cref{thm:cyclic-neighbor}, both active sets are optimal for the
corresponding cardinality-constrained line problems under this common
cut.

When more than two inactive atoms remain, the common cut cannot lie in
the interior of \(\mathcal C^{\star}\). Indeed, such a cut would place
\(x^{\star}\) and \(y^{\star}\) at the two opposite ends of the linear
inactive ordering, whereas the line-neighbor property requires them to
be consecutive. If only two inactive atoms remain, the two complementary
directed cells have the same endpoints, and we take
\(\mathcal C^{\star}\) to be the one whose interior does not contain the
common cut.

Thus, in all cases, the common cut lies outside
\(\mathcal C^{\star}\). The one-dimensional locality principle then
shows that the difference between the two optimal line costs is exactly
the local marginal of this cell:
\[
\mathfrak m(\mathcal C^{\star})
=
W_{1}^{\mathrm{line}}(A_{k+1}^{\star})
-
W_{1}^{\mathrm{line}}(A_k)
=
C_{k+1}^{\circ}-C_k^{\circ}
=
\delta_{k+1}.
\]
It follows that
\[
\min_{\mathcal C\in\mathfrak C_k}
\mathfrak m(\mathcal C)
\leq
\mathfrak m(\mathcal C^{\star})
=
\delta_{k+1}.
\]
Combining this inequality with \eqref{eq:greedy-lower-bound} proves
\eqref{eq:greedy-choice}.

Finally, suppose that \(\mathcal C\) attains the minimum in
\eqref{eq:greedy-choice}. The construction in the first part of the
proof gives a feasible circular cardinality-\((k+1)\) matching whose
cost is at most
\[
C_k^\circ+\mathfrak m(\mathcal C)
=
C_k^\circ+\delta_{k+1}
=
C_{k+1}^\circ.
\]
Since \(C_{k+1}^\circ\) is the minimum cost among all feasible
cardinality-\((k+1)\) matchings, this matching must attain
\(C_{k+1}^\circ\). Therefore, activating the endpoints of any minimizing
candidate cell produces a circularly optimal active set \(A_{k+1}\).
\end{proof}

\begin{corollary}[Optimal nested sequence]
\label{cor:greedy-sequence}
Starting from $A_0=\varnothing$ and repeatedly selecting a minimum-marginal candidate cell produces an optimal active set for every cardinality $k=0,\dots,K$.
\end{corollary}

\subsection{A constructive simultaneous cut}

\begin{theorem}[Simultaneous optimal cut]
\label{thm:simultaneous-cut}
There exists an original support gap $\theta^\star$ such that, for every $k=0,\dots,K$,
\begin{equation}
C_{k,\theta^\star}^{\mathrm{line}}=C_k^{\circ}.
\label{eq:simultaneous-cut}
\end{equation}
Moreover, the greedy algorithm constructs such a gap.
\end{theorem}

\begin{proof}
For every current cell, maintain one representative free gap.  Initially this is the cell's unique original gap.  When a selected cell and its two neighbors merge, let the new cell inherit a free-gap representative from either unselected neighboring cell, as guaranteed by \cref{lem:free-gap}.  \Cref{fig:simultaneous-cut} traces the resulting gap through a complete run: one original gap is avoided by every selection, so the same cut serves all cardinalities.

If $n\neq m$, the algorithm terminates with inactive atoms from only the larger measure.  Choose a representative free gap from any remaining current cell.  If $n=m$, the final step starts with two inactive atoms and two complementary cells.  Select one cell and retain the free-gap representative of the other.  In both cases, the resulting gap $\theta^\star$ belongs to no selected cell.

By \cref{lem:cut-compatibility}, cutting at $\theta^\star$ gives
\[
W_{1,\theta^\star}^{\mathrm{line}}(A_k)
=
\sum_{r=1}^{k}\mathfrak m(\mathcal C_r)
=
C_k^{\circ}
\]
for every greedy active set $A_k$, where the last equality follows from \cref{thm:greedy}.  Therefore
\[
C_{k,\theta^\star}^{\mathrm{line}}
\leq C_k^{\circ}.
\]
Conversely, any line matching is feasible on the circle with no larger geodesic cost, so
\[
C_k^{\circ}\leq C_{k,\theta^\star}^{\mathrm{line}}.
\]
The two inequalities imply \eqref{eq:simultaneous-cut}.
\end{proof}

\begin{figure}[t]
\centering
\begin{minipage}[t]{0.34\linewidth}
\centering
\textbf{(a) One free gap survives every selection}\par\medskip
\begin{tikzpicture}[
 scale=.9,font=\small,
 srcpt/.style={circle,draw=src,fill=src,inner sep=1.6pt},
 tgtpt/.style={circle,draw=tgt,fill=tgt,inner sep=1.6pt},
 sel/.style={draw=candidate!75!black,line width=2.2pt,line cap=round},
 cut/.style={draw=freegap,densely dashed,line width=1.4pt}
]
\def\R{2.0}
\draw[line width=.8pt] (0,0) circle (\R);
\draw[cut] (32:\R) arc[start angle=32,end angle=48,radius=\R];
\node[freegap] at (40:2.42) {$\theta^\star$};
\draw[sel] (120:\R) arc[start angle=120,end angle=160,radius=\R];
\draw[sel] (200:\R) arc[start angle=200,end angle=240,radius=\R];
\draw[sel] (280:\R) arc[start angle=280,end angle=320,radius=\R];
\node[srcpt,label={[src]120:$x_1$}] at (120:\R) {};
\node[tgtpt,label={[tgt]160:$y_1$}] at (160:\R) {};
\node[srcpt,label={[src]200:$x_2$}] at (200:\R) {};
\node[tgtpt,label={[tgt]240:$y_2$}] at (240:\R) {};
\node[srcpt,label={[src]280:$x_3$}] at (280:\R) {};
\node[tgtpt,label={[tgt]320:$y_3$}] at (320:\R) {};
\node[align=center,font=\scriptsize,draw=black!45,rounded corners=2pt,inner sep=3.5pt] at (0,0)
 {no selected cell\\contains $\theta^\star$};
\end{tikzpicture}
\par\scriptsize All three selections avoided the gap $\theta^\star$, so the same cut is valid for every $k$.
\end{minipage}\hfill
\begin{minipage}[t]{0.62\linewidth}
\centering
\textbf{(b) Nested optimal plans after cutting at $\theta^\star$}\par\medskip
\resizebox{\linewidth}{!}{%
\begin{tikzpicture}[
 x=.88cm,y=1cm,font=\small,
 srcpt/.style={circle,draw=src,fill=src,inner sep=1.6pt},
 srcinactive/.style={circle,draw=src,fill=white,inner sep=2.1pt,line width=.9pt},
 tgtpt/.style={circle,draw=tgt,fill=tgt,inner sep=1.6pt},
 tgtinactive/.style={circle,draw=tgt,fill=white,inner sep=2.1pt,line width=.9pt},
 match/.style={draw=black!65,line width=.8pt,-{Stealth[length=1.5mm]}},
 cut/.style={draw=freegap,densely dashed,line width=1.1pt}
]
\foreach \yy/\kk in {1.25/1,0/2,-1.25/3}{
  \draw[line width=.7pt] (0,\yy)--(8.1,\yy);
  \draw[cut] (0,\yy-.22)--(0,\yy+.22);
  \draw[cut] (8.1,\yy-.22)--(8.1,\yy+.22);
  \node[anchor=east] at (-.35,\yy) {$k=\kk$};
  \node[freegap,font=\scriptsize] at (0,\yy-.42) {$\theta^\star$};
  \node[freegap,font=\scriptsize] at (8.1,\yy-.42) {$\theta^\star{+}L$};
}
\node[srcpt] at (1.0,1.25) {};  \node[tgtpt] at (2.0,1.25) {};
\node[srcinactive] at (3.0,1.25) {}; \node[tgtinactive] at (4.15,1.25) {};
\node[srcinactive] at (5.4,1.25) {}; \node[tgtinactive] at (6.4,1.25) {};
\draw[match] (1.0,1.25) to[bend left=35] (2.0,1.25);
\node[srcpt] at (1.0,0) {};  \node[tgtpt] at (2.0,0) {};
\node[srcpt] at (3.0,0) {};  \node[tgtpt] at (4.15,0) {};
\node[srcinactive] at (5.4,0) {}; \node[tgtinactive] at (6.4,0) {};
\draw[match] (1.0,0) to[bend left=35] (2.0,0);
\draw[match] (3.0,0) to[bend left=35] (4.15,0);
\node[srcpt] at (1.0,-1.25) {};  \node[tgtpt] at (2.0,-1.25) {};
\node[srcpt] at (3.0,-1.25) {};  \node[tgtpt] at (4.15,-1.25) {};
\node[srcpt] at (5.4,-1.25) {};  \node[tgtpt] at (6.4,-1.25) {};
\draw[match] (1.0,-1.25) to[bend left=35] (2.0,-1.25);
\draw[match] (3.0,-1.25) to[bend left=35] (4.15,-1.25);
\draw[match] (5.4,-1.25) to[bend left=35] (6.4,-1.25);
\node[align=center,font=\scriptsize,draw=black!45,rounded corners=2pt,inner sep=4pt] at (4.05,-2.35)
 {$C^{\rm line}_{k,\theta^\star}=C^\circ_k$ for every $k$; each plan is recovered by sorted matching};
\end{tikzpicture}}
\par\scriptsize The circle of (a), unwrapped at $\theta^\star$ into the interval $[\theta^\star,\theta^\star{+}L]$.  Atoms appear in the same cyclic order ($x_1$,\,$y_1$,\,$x_2$,\,$y_2$,\,$x_3$,\,$y_3$); one admissible nested activation order is shown.
\end{minipage}
\caption{Constructive simultaneous cut.  The free gap $\theta^\star$ is contained in no selected cell, so the complete nested sequence can be cut and unwrapped at the same location.  On the resulting line, the active sets grow by local sorted-matching updates and $C_{k,\theta^\star}^{\mathrm{line}}=C_k^\circ$ for every transported cardinality $k$.}
\label{fig:simultaneous-cut}
\end{figure}
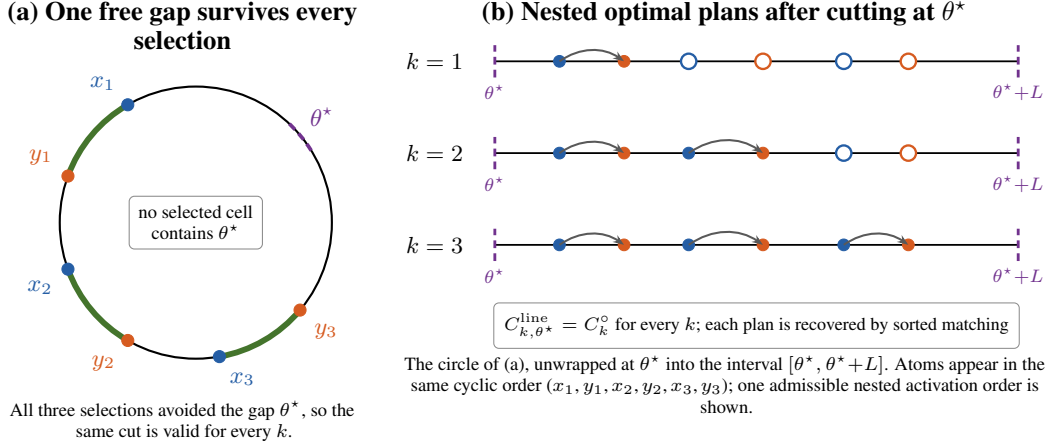

\begin{remark}[Existence is not uniqueness]
\label{rem:not-every-cut}
The theorem asserts existence of at least one simultaneous cut.  It does not claim
that every cut optimal for the full-mass problem is optimal for all smaller masses.
The distinction is not academic, and it is what rules out the obvious shortcut:
cut once at a gap optimal for $k=K$, run line \PAWL{} there, and read off the whole
profile.  That procedure is $O(N\log N)$ and uses none of the structure above.
\Cref{ex:one-cut-unbounded} shows that its error is unbounded.
\end{remark}

\begin{example}[A cut optimal at $k=K$ and unbounded at $k=1$]
\label{ex:one-cut-unbounded}
Take $n=m=2$ on a circle of circumference $L$, with $0<\delta<L/4$:
\begin{equation}
x=\{0,\;L/4\},\qquad y=\{\delta,\;3L/4\}.
\label{eq:one-cut-counterexample}
\end{equation}
Write $G_0=(0,\delta)$ and $G_2=(L/4,3L/4)$ for the two gaps whose cumulative
imbalance sits one level above their neighbours.  Together they carry
$\delta+L/2>L/2$ of the circumference, so that level is the weighted median in
\eqref{eq:circular-median} and, by \cref{prop:fixed-cut},
\begin{equation}
G_0,G_2\in \Theta_K:=\bigl\{r:\ C_{K,r}^{\mathrm{line}}=C_K^{\circ}\bigr\},
\qquad K=2 .
\end{equation}
Both are therefore admissible choices for a rule that cuts where the full-mass
problem is solved, and $G_0$ is the first of them in index order.

At $k=1$ the two disagree.  The closest opposite-type pair is $(0,\delta)$, so
$C_1^{\circ}=w\delta$, and cutting in $G_2$ recovers it.  Cutting in $G_0$ separates
that pair: their unwrapped distance becomes $L-\delta$, and the sorted line matching
falls back on the runner-up $d_2=L/4-\delta$.  Hence
\begin{equation}
\frac{C_{1,G_0}^{\mathrm{line}}-C_1^{\circ}}{C_1^{\circ}}
=\frac{d_2-\delta}{\delta}
=\frac{L/4-2\delta}{\delta}
\;\xrightarrow[\ \delta\to0\ ]{}\;\infty .
\label{eq:one-cut-blowup}
\end{equation}
The mechanism is that $C_1^{\circ}$ is a \emph{geodesic} nearest-pair distance
whereas $C_{1,r}^{\mathrm{line}}$ is an \emph{unwrapped} one, and nothing couples the
location of the closest pair to the median that fixes $\Theta_K$.  So no bound of the
form $C_{k,r}^{\mathrm{line}}\le c\,C_k^{\circ}$ holds uniformly over $r\in \Theta_K$, for
any constant $c$: the full-mass problem carries no information about the small-mass
ones.  This is what \cref{thm:simultaneous-cut} supplies and a single cut cannot.
\end{example}

\section{Constant-time cell-marginal queries}
\label{app:constant}

\label{sec:constant-time}

We now adapt the minimal-chain preprocessing of \PAWL{} \citep{chapel2025one} to cyclic intervals.

\subsection{Doubled cyclic sequence}
\label{sec:doubled}

Let
\[
z_1<z_2<\cdots<z_N<z_1+L
\]
be the sorted union support, and define labels
\begin{equation}
\sigma_i=
\begin{cases}
+1,&z_i\in\supp\mu,\\
-1,&z_i\in\supp\nu.
\end{cases}
\label{eq:labels}
\end{equation}
Form the doubled sequence
\begin{equation}
\widetilde z_i=z_i,
\quad
\widetilde z_{i+N}=z_i+L,
\quad
\widetilde\sigma_i=\widetilde\sigma_{i+N}=\sigma_i,
\qquad i=1,\dots,N.
\label{eq:doubled}
\end{equation}
Every clockwise circular cell can be represented by a standard interval $[a,b]\subset\{1,\dots,2N\}$ with $b-a<N$.

Define prefix differential ranks and signed coordinate sums
\begin{equation}
R_0=0,
\qquad
R_t=\sum_{i=1}^{t}\widetilde\sigma_i,
\label{eq:rank-prefix}
\end{equation}
\begin{equation}
S_0=0,
\qquad
S_t=\sum_{i=1}^{t}\widetilde\sigma_i\widetilde z_i.
\label{eq:signed-prefix}
\end{equation}
An interval $[a,b]$ is balanced if and only if
\begin{equation}
R_b=R_{a-1}.
\label{eq:balanced-rank}
\end{equation}

For each $t$, let
\begin{equation}
p_t=\max\{s<t:R_s=R_t\},
\label{eq:previous-rank}
\end{equation}
when this set is nonempty.  Then $[p_t+1,t]$ is the minimal balanced interval ending at $t$.

\begin{lemma}[Minimal-chain cost]
\label{lem:minimal-chain-cost}
If $p_t$ exists, then
\begin{equation}
c_{\mathbb R}([p_t+1,t])
=
w\left|S_t-S_{p_t}\right|.
\label{eq:minimal-cost}
\end{equation}
\end{lemma}

\begin{proof}
By the maximality of $p_t$, the relative rank $R_s-R_{p_t}$ does not return to zero for $p_t<s<t$.  It therefore has a constant strict sign throughout the interval.  Equivalently, every pair in the increasing-order line matching has the same orientation.  The sum of absolute pair displacements is consequently the absolute difference between the sum of source coordinates and the sum of target coordinates, which is $|S_t-S_{p_t}|$.  Multiplying by $w$ gives \eqref{eq:minimal-cost}.  This is the minimal-chain identity used by \PAWL{} for the Manhattan cost \citep{chapel2025one}.
\end{proof}

\Cref{fig:minimal-chains} illustrates the two ingredients: a balanced interval decomposing into adjacent minimal chains, and the differential-rank walk whose first returns identify them.  Define the maximal-chain prefix cost $Q_t$ recursively by
\begin{equation}
Q_0=0,
\qquad
Q_t=
\begin{cases}
Q_{p_t}+w|S_t-S_{p_t}|,&p_t\text{ exists},\\
0,&p_t\text{ does not exist}.
\end{cases}
\label{eq:Q-recurrence}
\end{equation}

\begin{proposition}[Balanced interval query]
\label{prop:balanced-query}
For every balanced interval $[a,b]$ in the doubled sequence,
\begin{equation}
c_{\mathbb R}([a,b])=Q_b-Q_{a-1}.
\label{eq:balanced-query}
\end{equation}
\end{proposition}

\begin{proof}
Repeatedly applying $p_t$ decomposes every balanced interval into a unique sequence of adjacent minimal balanced intervals.  The increasing-order line matching never crosses a boundary between adjacent balanced blocks, so their costs add.  Because $R_b=R_{a-1}$, the maximal-chain decomposition ending at $a-1$ is exactly the prefix of the maximal-chain decomposition ending at $b$.  Subtracting their accumulated costs gives \eqref{eq:balanced-query}.  This is the same dynamic-programming principle as the maximal-chain cost identity in \PAWL{} \citep{chapel2025one}.
\end{proof}

\begin{corollary}[Constant-time candidate marginal]
\label{cor:constant-marginal}
Let a candidate cell be represented by the balanced doubled interval $[a,b]$.  Then
\begin{equation}
\mathfrak m([a,b])
=
\bigl(Q_b-Q_{a-1}\bigr)
-
\bigl(Q_{b-1}-Q_a\bigr),
\label{eq:constant-marginal}
\end{equation}
with the second term interpreted as zero when the interior is empty.  Hence a candidate marginal is evaluated in $O(1)$ time after preprocessing.
\end{corollary}

\begin{proof}
The outer interval $[a,b]$ and the interior $[a+1,b-1]$ are balanced by \cref{lem:balanced-interiors}.  Apply \cref{prop:balanced-query} to both and use \eqref{eq:cell-marginal}.
\end{proof}

\begin{figure}[t]
\centering
\centering
\textbf{(a) A balanced interval in the doubled sequence decomposes into minimal balanced chains}\par\medskip
\begin{tikzpicture}[
 x=.92cm,y=1cm,font=\small,
 srcpt/.style={circle,draw=src,fill=src,inner sep=1.7pt},
 tgtpt/.style={circle,draw=tgt,fill=tgt,inner sep=1.7pt},
 match/.style={draw=black!65,line width=.8pt}
]
\draw[line width=.7pt] (0,0)--(10.2,0);
\node at (.25,0) {$\cdots$}; \node at (9.9,0) {$\cdots$};
\node[tgtpt,label={[tgt]below:$y_1$}] (y1) at (1,0) {};
\node[srcpt,label={[src]below:$x_1$}] (x1) at (2,0) {};
\node[srcpt,label={[src]below:$x_2$}] (x2) at (4,0) {};
\node[srcpt,label={[src]below:$x_3$}] (x3) at (5,0) {};
\node[tgtpt,label={[tgt]below:$y_2$}] (y2) at (6,0) {};
\node[tgtpt,label={[tgt]below:$y_3$}] (y3) at (7,0) {};
\node[tgtpt,label={[tgt]below:$y_4$}] (y4) at (8.2,0) {};
\node[srcpt,label={[src]below:$x_4$}] (x4) at (9.2,0) {};
\draw[match] (y1) to[bend left=40] (x1);
\draw[match] (x2) to[bend left=45] (y2);
\draw[match] (x3) to[bend left=45] (y3);
\draw[match] (x4) to[bend right=40] (y4);
\draw[decorate,decoration={brace,amplitude=5pt},draw=magenta!75!black] (.75,-.6)--(2.25,-.6) node[midway,below=1pt,text=magenta!75!black] {$\mathcal C_1^\star$};
\draw[decorate,decoration={brace,amplitude=5pt},draw=magenta!75!black] (3.75,-.6)--(7.25,-.6) node[midway,below=1pt,text=magenta!75!black] {$\mathcal C_2^\star$};
\draw[decorate,decoration={brace,amplitude=5pt},draw=magenta!75!black] (7.95,-.6)--(9.45,-.6) node[midway,below=1pt,text=magenta!75!black] {$\mathcal C_3^\star$};
\draw[decorate,decoration={brace,amplitude=5pt},draw=candidate!70!black,line width=1pt] (.5,-1.3)--(9.7,-1.3) node[midway,below=6pt,text=candidate!45!black] {$[a,b]=\mathcal C_1^\star\cup\mathcal C_2^\star\cup\mathcal C_3^\star$};
\end{tikzpicture}
\[
\widetilde z_{i+N}=\widetilde z_i+L,
\qquad
 c_{\mathbb R}([a,b])=\sum_r c_{\mathbb R}(\mathcal C_r^\star)=Q_b-Q_{a-1}.
\]
\vspace{1mm}
\begin{minipage}[t]{0.57\linewidth}
\centering
\textbf{(b) Differential-rank walk}\par\medskip
\resizebox{\linewidth}{!}{%
\begin{tikzpicture}[x=.72cm,y=.6cm,font=\small]
\draw[->] (.3,0)--(10.6,0) node[right] {$t$};
\draw[->] (.3,-1.9)--(.3,2.6) node[above] {$R_t-R_{a-1}$};
\draw[densely dashed,black!35] (.3,0)--(10.4,0);
\draw[line width=.9pt]
 (.45,0)--(1,0)--(1,-1)--(2,-1)--(2,0)--(4,0)--(4,1)--(5,1)--(5,2)--(6,2)--(6,1)--(7,1)--(7,0)--(8.2,0)--(8.2,-1)--(9.2,-1)--(9.2,0)--(10.1,0);
\foreach \x in {.6,2,7,9.2}{\fill[magenta!75!black] (\x,0) circle (1.7pt);}
\node[magenta!75!black,font=\scriptsize] at (.85,.42) {$a{-}1$};
\node[magenta!75!black,font=\scriptsize] at (9.2,.42) {$b$};
\node[magenta!75!black,font=\scriptsize] at (1.5,-1.55) {first return};
\node[magenta!75!black,font=\scriptsize] at (5.5,-.55) {first return};
\node[magenta!75!black,font=\scriptsize] at (8.7,-1.55) {first return};
\end{tikzpicture}}
\end{minipage}\hfill
\begin{minipage}[t]{0.40\linewidth}
\small
A balanced interval satisfies $R_b=R_{a-1}$.  For each endpoint $t$, the most recent previous occurrence
\[
p_t=\max\{s<t:R_s=R_t\}
\]
identifies the minimal chain ending at $t$.  Its cost is
\[
c_{\mathbb R}([p_t+1,t])=w\lvert S_t-S_{p_t}\rvert.
\]
When $[a,b]$ represents a candidate cell, as above (opposite-type endpoints, balanced interior), its marginal
\[
\mathfrak m([a,b])=(Q_b-Q_{a-1})-(Q_{b-1}-Q_a)
\]
is obtained from four table lookups.
\end{minipage}
\caption{Constant-time cell-cost preprocessing on the doubled cyclic sequence.  A balanced interval decomposes uniquely into adjacent minimal balanced chains, analogous to the chain decomposition in \PAWL{} \citep{chapel2025one}.  Repeated differential-rank values identify the minimal chains; their costs are signed-coordinate prefix differences.  The accumulated values $Q_t$ then provide both the cell cost and its interior cost with constant-time table lookups.}
\label{fig:minimal-chains}
\end{figure}

All values $R_t,S_t,p_t,Q_t$ for $t=0,\dots,2N$ are computed in $O(N)$ time after sorting.  The predecessor $p_t$ is found by storing the most recent index of each integer differential-rank value in a hash table or array indexed over the rank range $[-2N,2N]$.

\section{Correctness and complexity}
\label{app:complexity}

The active/inactive cyclic order is maintained by a circular doubly linked list.  Each directed current cell stores its endpoints and a representative free original gap.  Candidate cells are stored in a min-heap keyed by \eqref{eq:constant-marginal}.  Heap deletion is lazy: an entry is valid only if both endpoints remain inactive and the second endpoint is still the clockwise successor of the first.

\begin{theorem}[Correctness and complexity]
\label{thm:complexity}
Under \cref{ass:standing}, \cref{alg:pawc-main} returns the exact values $C_k^{\circ}$ and an optimal active set for every $k=0,\dots,K$, together with a gap $\theta^\star$ satisfying \eqref{eq:simultaneous-cut}.  Its time and memory complexities are
\begin{equation}
\boxed{\text{time }=O(N\log N),\qquad \text{memory }=O(N).}
\label{eq:complexity}
\end{equation}
\end{theorem}

\begin{proof}
Correctness follows inductively from \cref{thm:greedy}; the simultaneous cut is guaranteed by \cref{thm:simultaneous-cut}.  Sorting costs $O(N\log N)$.  The doubled-sequence preprocessing and linked-list initialization cost $O(N)$.  There are at most $N$ initial heap entries and at most one new candidate insertion per iteration.  Thus only $O(N)$ heap entries are ever created.  Every entry is popped at most once, valid or stale, and each heap operation costs $O(\log N)$.  Candidate-key evaluation, linked-list updates, and free-gap inheritance are $O(1)$.  The heap, linked list, preprocessing arrays, and activation records all use $O(N)$ memory.
\end{proof}

\section{Plan recovery and arbitrary transported masses}
\label{app:plans}

\label{sec:plan-recovery}

\subsection{Integer cardinalities}

The activation rank of each atom determines the optimal active set $A_k$.  Cut the circle at the simultaneous gap $\theta^\star$, retain the atoms activated by iteration $k$, and match the retained source and target atoms in increasing unwrapped order.  By \cref{thm:simultaneous-cut}, this line matching has value $C_k^{\circ}$ and is therefore circularly optimal.

If the selected atoms are filtered from the already sorted circular order, the matching can be recovered in $O(N)$ time; sorting only the selected atoms gives the simpler $O(k\log k)$ implementation.

\subsection{Noninteger transported mass}

Let $s\in[0,Kw]$ and write
\begin{equation}
k=\left\lfloor\frac{s}{w}\right\rfloor,
\qquad
\lambda=\frac{s-kw}{w}\in[0,1].
\label{eq:fractional-lambda}
\end{equation}
Let $\pi_k$ and $\pi_{k+1}$ be consecutive optimal transport plans associated with the nested sequence.  Then
\begin{equation}
\pi_s=(1-\lambda)\pi_k+\lambda\pi_{k+1}
\label{eq:interpolated-plan}
\end{equation}
is feasible for transported mass $s$.

\begin{proposition}[Interpolation]
\label{prop:interpolation}
The plan \eqref{eq:interpolated-plan} is optimal, and
\begin{equation}
\PW_{\circ}(s)
=(1-\lambda)C_k^{\circ}+\lambda C_{k+1}^{\circ}.
\label{eq:interpolated-value}
\end{equation}
\end{proposition}

\begin{proof}
The partial matching problem is a min-cost flow problem with integral capacities.  Its optimal value as a function of the required flow is convex and piecewise linear, with breakpoints at integer flow values; successive shortest augmentations give nondecreasing marginal costs \citep{ahuja1993network}.  Therefore the value on $[k,k+1]$ is the linear interpolation of $C_k^{\circ}$ and $C_{k+1}^{\circ}$.  The feasible plan \eqref{eq:interpolated-plan} attains exactly that value.  Equivalently, after cutting at $\theta^\star$, this is the interpolation principle used by \PAWL{} on the line \citep{chapel2025one}.
\end{proof}

\section{Spherical slicing and partial transport on $\Sph$}
\label{app:sphere}

\label{sec:slicing}

The construction so far is exact but one-dimensional.  This section uses it as a
primitive: we recall the spherical slicing of \citet{bonet2023spherical} and
\citet{liu2025linear}, and then define a sliced discrepancy that inherits
\PAWC{}'s partial structure, so that the whole transported-mass profile on
$\Sph$ is available at once, and the two measures need not carry the same mass.

\subsection{Great circles and the geodesic projection}
\label{sec:slicing-projection}

Let $\Sph=\{x\in\RR^d:\|x\|_2=1\}$ with the geodesic distance
$\dsph(x,y)=\arccos\langle x,y\rangle$.  The geodesics of $\Sph$ are its great
circles, obtained by intersecting the sphere with a two-dimensional linear
subspace.  Slicing on $\Sph$ therefore replaces the lines of Euclidean
sliced transport by great circles, and the one-dimensional problem it leaves
behind is a transport problem on a circle rather than on a line -- which is
precisely the problem \cref{alg:pawc-main} solves.

Following \citet{bonet2023spherical}, subspaces are parameterised by the Stiefel
manifold
\begin{equation}
V_{d,2}=\{U\in\RR^{d\times 2}: U^{\!\top}U=I_2\},
\label{eq:stiefel}
\end{equation}
which covers the Grassmannian $G_{d,2}$ of two-planes through
$U\mapsto UU^{\!\top}$ \citep{bendokat2020grassmann}.  Each $U\in V_{d,2}$ gives
the great circle $\mathcal C_U=\mathrm{span}(U)\cap\Sph$, and the natural map onto
it is the geodesic projection
\begin{equation}
\Pi_U(x)=\operatorname*{arg\,min}_{y\in\mathcal C_U} \dsph(x,y)
=\frac{UU^{\!\top}x}{\|U^{\!\top}x\|_2}\;\in\;\mathcal C_U\subset\RR^{d},
\label{eq:geodesic-projection}
\end{equation}
defined except on $\{x: U^{\!\top}x=0\}$, a subsphere of codimension two that is
null for any measure charging no such set \citep{bardelli2017probability}, so the
pushforwards below are well defined.

The columns of $U$ form an orthonormal basis of $\mathrm{span}(U)$, so $z\mapsto Uz$
is an isometry from $\Sunit$ onto $\mathcal C_U$ and $\Pi_U=U\,P^{U}$, where
\begin{equation}
P^{U}(x)=U^{\!\top}\Pi_U(x)=\frac{U^{\!\top}x}{\|U^{\!\top}x\|_2}\;\in\;\Sunit
\label{eq:slice-coordinates}
\end{equation}
is the coordinate representation of \eqref{eq:geodesic-projection} in the basis $U$
\citep[Lemma~1]{bonet2023spherical}.  The two are kept apart deliberately: $\Pi_U$
lands on the great circle in $\RR^{d}$, whereas $P^{U}$ is its coordinate vector in
$\RR^{2}$, and it is $P^{U}$ that the one-dimensional solver consumes.  Concretely,
\cref{alg:pawc-main} is run on the circular coordinate
\begin{equation}
\theta_U(x)=\operatorname{atan2}\bigl((U^{\!\top}x)_2,\;(U^{\!\top}x)_1\bigr)\bmod 2\pi
\;\in\;[0,2\pi),
\label{eq:slice-angle}
\end{equation}
that is, on $\Sone$ with circumference $L=2\pi$; because $z\mapsto Uz$ is an
isometry, arc length along $\mathcal C_U$ agrees with arc length in the coordinate
$\theta_U$, so no rescaling is needed.  \Cref{fig:spherical-slicing}(a) illustrates
the projection.

Two features of this construction matter here.  First, the image of $P^{U}$ is a
circle of circumference $2\pi$, so the slice inherits the cyclic structure with no
distinguished boundary -- the difficulty that motivated everything above, now
arising for a second reason.  Second, $P^{U}$ is a pushforward of measures, so it
preserves total mass and maps atoms to atoms: an empirical measure with $n$ atoms
of common weight $w$ becomes an empirical measure on $\Sunit$ with $n$ atoms
of the same weight $w$.  The standing assumptions of \cref{ass:standing} are
therefore met on almost every slice, distinctness of the projected support failing
only on a $\sigma$-null set of $U$.

\subsection{Sliced spherical partial Wasserstein}
\label{sec:slicing-definition}

Let $\sigma$ be the uniform probability measure on $V_{d,2}$.  For probability
measures $\mu,\nu\in\mathcal P(\Sph)$, \citet{bonet2023spherical} define the
spherical sliced-Wasserstein discrepancy by averaging the circular Wasserstein
distance over slices,
\begin{equation}
\mathrm{SSW}_p^p(\mu,\nu)=\int_{V_{d,2}}W_{p,\circ}^p\bigl(P^{U}_{\#}\mu,\,P^{U}_{\#}\nu\bigr)\,d\sigma(U).
\label{eq:ssw}
\end{equation}
This presumes $\mu$ and $\nu$ are probability measures, hence of equal mass.  We
drop that requirement.

Let $\mu$ and $\nu$ be finite positive measures on $\Sph$ with uniformly weighted
atoms,
\begin{equation}
\mu=w\sum_{i=1}^{n}\delta_{x_i},
\qquad
\nu=w\sum_{j=1}^{m}\delta_{y_j},
\qquad x_i,y_j\in\Sph,\ w>0,
\label{eq:sphere-measures}
\end{equation}
with $n\neq m$ allowed, so that $\mu(\Sph)=nw$ and $\nu(\Sph)=mw$ may differ.
Because each $P^{U}$ preserves mass, the admissible transported masses
$s\in[0,Kw]$ with $K=\min\{n,m\}$ are the same on every slice; a single $s$ is thus
meaningful across all of $V_{d,2}$ simultaneously.

\begin{definition}[Sliced spherical partial Wasserstein]
\label{def:sspw}
For $s\in[0,Kw]$ and $\mu,\nu$ as in \eqref{eq:sphere-measures}, define
\begin{equation}
\mathrm{SSPW}(s;\mu,\nu)
=\int_{V_{d,2}}\PW_{\circ}\bigl(s;\,P^{U}_{\#}\mu,\,P^{U}_{\#}\nu\bigr)\,d\sigma(U),
\label{eq:sspw}
\end{equation}
where $\PW_{\circ}(s;\cdot,\cdot)$ is the partial transport cost \eqref{eq:partial-lp}
on the circle, i.e.\ the quantity \cref{alg:pawc-main} returns for every admissible $s$
in a single run.
\end{definition}

At integer cardinalities the integrand is $C_k^{\circ}$ of
\eqref{eq:integer-problem} evaluated on the projected measures; between them it is
the interpolation of \cref{prop:interpolation}.  Four remarks fix the scope of the
definition.

\begin{remark}[Scope]
\label{rem:sspw-scope}
Two standing restrictions are worth stating plainly.  First, $d\ge 2$, and $d=2$ is
degenerate: there $\mathrm{span}(U)=\RR^{2}$, so $\mathcal C_U$ is all of
$\Sunit$, $P^{U}$ is an isometry, and $\mathrm{SSPW}$ collapses to
\cref{alg:pawc-main} itself -- slicing does nothing.  The construction is of interest for
$d\ge 3$.  Second, \eqref{eq:sphere-measures} gives $\mu$ and $\nu$ a \emph{common}
atom weight $w$, inherited from \cref{ass:standing}.  What varies is the number of
atoms, hence the total mass; the atom weight does not.  In particular this does not
cover two empirical probability measures with different sample counts, whose atoms
would weigh $1/n$ and $1/m$: that is a nonuniform-weight problem, requiring the
event-driven treatment of saturating atom capacities noted in the discussion, and is
outside the present statement.
\end{remark}

\begin{remark}[What is inherited slice-wise]
\label{rem:sspw-inherited}
Several properties pass from the circle to the sphere because they hold for the
integrand at every $U$ and are preserved by integration.
\begin{enumerate}[leftmargin=1.8em,itemsep=2pt]
\item \emph{Symmetry and nonnegativity}: $\mathrm{SSPW}(s;\mu,\nu)=\mathrm{SSPW}(s;\nu,\mu)\ge 0$,
since $\dcirc$ is symmetric.
\item \emph{Shape in $s$}: $s\mapsto\mathrm{SSPW}(s;\mu,\nu)$ is nondecreasing, convex and
piecewise linear with breakpoints contained in $\{kw\}_{k=0}^{K}$, because each
slice profile is (\cref{prop:interpolation}) and these properties survive averaging.
\item \emph{Isometry invariance}: for $R\in O(d)$,
$\mathrm{SSPW}(s;R_{\#}\mu,R_{\#}\nu)=\mathrm{SSPW}(s;\mu,\nu)$, since $P^{U}\circ R=P^{R^{\!\top}U}$
and $\sigma$ is invariant under the action of $O(d)$ on $V_{d,2}$.
\item \emph{Reduction}: if $n=m$ and $s=nw$, every slice problem is balanced.
Since $W_{1,\circ}$ is positively homogeneous under a common rescaling of both
masses,
\begin{equation}
\mathrm{SSPW}(nw;\mu,\nu)=nw\;\mathrm{SSW}_1\!\Bigl(\tfrac{\mu}{nw},\,\tfrac{\nu}{nw}\Bigr),
\label{eq:sspw-reduction}
\end{equation}
so \eqref{eq:sspw} recovers $\mathrm{SSW}_1$ of \eqref{eq:ssw} exactly when $nw=1$,
and up to that mass factor otherwise.
\end{enumerate}
\end{remark}

\begin{remark}[What is not claimed]
\label{rem:sspw-open}
$\mathrm{SSPW}$ is a discrepancy, not a metric, and we claim no more here.  Two
obstructions are already present one slice at a time.  Partial transport at a
\emph{fixed} transported mass does not satisfy the triangle inequality even on the
line, so no metric property can be inherited slice-wise; and $\mathrm{SSPW}(s;\mu,\nu)=0$
for $s>0$ only forces the projected measures to share $s$ units of mass on
$\sigma$-almost every slice, which is a statement about $P^{U}_{\#}\mu$ and
$P^{U}_{\#}\nu$ rather than about $\mu$ and $\nu$.  The injectivity results that
\citet{bonet2023spherical} obtain for the spherical Radon transform, and the
metricity that \citet{liu2025linear} establish for the linear embedding, both
concern probability measures and equal-mass transport; whether an analogue holds
once mass is allowed to differ is open, and settling it is the natural next question.
Sample complexity, by contrast, is not open at the elementary level: the integrand of
\eqref{eq:sspw} is bounded, and \cref{prop:sspw-concentration} below gives a uniform
Monte-Carlo guarantee over the whole profile.  What remains open there is sharper
dimension- or distribution-dependent control, since that bound uses nothing about
$\mu$ and $\nu$ beyond the range of the integrand.
\end{remark}

\begin{remark}[Choice of $s$]
\label{rem:sspw-choice-of-s}
Because \cref{alg:pawc-main} returns the entire profile, $s$ need not be committed to in
advance: one run per slice yields $\mathrm{SSPW}(\cdot;\mu,\nu)$ on all of $[0,Kw]$.  This
is the practical payoff of the profile being free.  When a single number is wanted,
$s$ can be selected afterwards, without recomputation.

For comparisons across data sets of differing cardinality or total mass, the
transported \emph{fraction}
\begin{equation}
\rho=\frac{s}{w\min\{n,m\}}\in[0,1]
\label{eq:transported-fraction}
\end{equation}
is more interpretable than the absolute mass $s$, and puts profiles with different
$K$ on a common axis.  When an elbow is wanted, it is usually easier to locate on the
marginal profile $\Delta_k=C_k^{\circ}-C_{k-1}^{\circ}$ than on the cumulative
profile $C_k^{\circ}$: the latter is convex and increasing, so the onset of
expensive transport registers only as a change of curvature, whereas $\Delta_k$ is
nondecreasing and registers the same event as a jump.
\end{remark}

\subsection{Estimation and cost}
\label{sec:slicing-cost}

In practice \eqref{eq:sspw} is estimated by Monte Carlo over slices; we write $M$
for their number, reserving $L$ for the circumference as elsewhere.  Draw
$U_1,\dots,U_M$ i.i.d.\ from $\sigma$, for instance by QR factorisation of a
$d\times 2$ Gaussian matrix, and set
\begin{equation}
\widehat{\mathrm{SSPW}}(s;\mu,\nu)=\frac1M\sum_{\ell=1}^{M}
\PW_{\circ}\bigl(s;\,P^{U_\ell}_{\#}\mu,\,P^{U_\ell}_{\#}\nu\bigr).
\label{eq:sspw-mc}
\end{equation}
\begin{proposition}[Uniform concentration of the sliced estimate]
\label{prop:sspw-concentration}
Let $\mu,\nu$ be as in \eqref{eq:sphere-measures} and let $U_1,\dots,U_M$ be i.i.d.\
$\sigma$.  For every fixed $s\in[0,Kw]$ and every $\varepsilon>0$,
\begin{equation}
\Pr\Bigl(\bigl|\widehat{\mathrm{SSPW}}(s;\mu,\nu)-\mathrm{SSPW}(s;\mu,\nu)\bigr|\ge\varepsilon\Bigr)
\;\le\;2\exp\!\Bigl(-\frac{2M\varepsilon^{2}}{\pi^{2}s^{2}}\Bigr).
\label{eq:sspw-hoeffding}
\end{equation}
Uniformly over the profile,
\begin{equation}
\Pr\Bigl(\sup_{s\in[0,Kw]}\bigl|\widehat{\mathrm{SSPW}}(s;\mu,\nu)-\mathrm{SSPW}(s;\mu,\nu)\bigr|\ge\varepsilon\Bigr)
\;\le\;2(K+1)\exp\!\Bigl(-\frac{2M\varepsilon^{2}}{\pi^{2}K^{2}w^{2}}\Bigr),
\label{eq:sspw-uniform}
\end{equation}
so that
$M\ge \bigl(\pi^{2}K^{2}w^{2}/2\varepsilon^{2}\bigr)\log\bigl(2(K+1)/\delta\bigr)$
slices suffice for $\varepsilon$-accuracy on the entire profile with probability at
least $1-\delta$.
\end{proposition}

\begin{proof}
Fix $s$ and set $Z_\ell=\PW_{\circ}\bigl(s;P^{U_\ell}_{\#}\mu,P^{U_\ell}_{\#}\nu\bigr)$,
so the $Z_\ell$ are i.i.d.\ with mean $\mathrm{SSPW}(s;\mu,\nu)$.  A transport plan of
mass $s$ moves no mass further than the geodesic radius $\pi$ of a circle of
circumference $2\pi$, whence $0\le Z_\ell\le\pi s$; the upper end is attained when
every transported pair is antipodal, so the range cannot be tightened without further
assumptions on $\mu$ and $\nu$.  Hoeffding's inequality for i.i.d.\ variables confined
to an interval of length $\pi s$ gives \eqref{eq:sspw-hoeffding}.

For \eqref{eq:sspw-uniform}, apply \eqref{eq:sspw-hoeffding} at each breakpoint
$s=kw$, $k=0,\dots,K$, with the range relaxed to the uniform bound $\pi Kw$, and take
a union bound over the $K+1$ breakpoints.  By \cref{prop:interpolation} every slice
profile is linear on $[kw,(k+1)w]$; hence so are $\widehat{\mathrm{SSPW}}$ and
$\mathrm{SSPW}$, being an average and an integral of such profiles over a common
breakpoint set, and hence so is their difference.  A linear function attains its
maximum modulus on an interval at an endpoint, so control at the breakpoints controls
the supremum.
\end{proof}

Each term of \eqref{eq:sspw-mc} costs one projection of the $N=n+m$ atoms, $O(dN)$,
followed by a single call to \cref{alg:pawc-main}, $O(N\log N)$ by \cref{thm:complexity}.  The estimator
therefore costs
\begin{equation}
O\bigl(M\,N(d+\log N)\bigr)
\label{eq:sspw-cost}
\end{equation}
in time and $O(N)$ in memory beyond the input, \emph{for the entire profile in $s$}.
The slices are independent and embarrassingly parallel.

\begin{remark}[A relative-accuracy guarantee is uniform for free]
\label{rem:relative-accuracy}
The range $\pi s$ in \eqref{eq:sspw-hoeffding} scales with $s$, so an
\emph{absolute} accuracy target treats small and large transported masses very
differently, and \eqref{eq:sspw-uniform} -- which relaxes every breakpoint to the
largest range $\pi Kw$ -- is correspondingly weak at small $s$
(\cref{sec:numerics-sphere} measures by how much).  Asking instead for accuracy
\emph{relative} to $s$ removes the disparity: taking $\varepsilon=\eta s$ in
\eqref{eq:sspw-hoeffding} gives, for every $s=kw$ with $k\ge 1$,
\begin{equation}
\Pr\Bigl(\bigl|\widehat{\mathrm{SSPW}}(s;\mu,\nu)-\mathrm{SSPW}(s;\mu,\nu)\bigr|\ge\eta\,s\Bigr)
\;\le\;2\exp\!\Bigl(-\frac{2M\eta^{2}}{\pi^{2}}\Bigr),
\label{eq:sspw-relative}
\end{equation}
in which $s$ has cancelled.  The bound is the same at every breakpoint, so a union
bound over the $K$ of them costs a factor $K$ and nothing else:
$\Pr(\exists k\ge1:\ |\cdot|\ge\eta\,kw)\le 2K\exp(-2M\eta^{2}/\pi^{2})$.  A
relative guarantee over the whole profile is thus no more expensive than a pointwise
one, which is the practical way to state accuracy when the profile spans several
orders of magnitude in $s$.  Retaining $\pi kw$ per breakpoint in
\eqref{eq:sspw-uniform} rather than relaxing to $\pi Kw$ would likewise be strictly
sharper, though for the supremum itself it buys only the factor $K+1$, the $k=K$
term dominating the sum in any case.
\end{remark}

Two points are worth isolating.  First, the profile is free: obtaining
$\mathrm{SSPW}(s;\cdot,\cdot)$ at every admissible $s$ costs the same as obtaining it at one,
because \cref{alg:pawc-main} produces all cardinalities in a single sweep.  A
construction built instead on a per-$s$ circular solver would pay a further factor
of $K=\Theta(N)$, which is exactly the cut-enumeration penalty measured in
\cref{sec:numerics-complexity}.  Second, replacing the circular solver by the
$O(N^2\log N)$ cut-enumeration reference of \cref{cor:cut-envelope} would raise
\eqref{eq:sspw-cost} to $O(LN^2\log N)$; the sliced construction is what makes the
$O(N\log N)$ solver worth having, and the $O(N\log N)$ solver is what makes the
sliced construction affordable at the sizes reported in
\cref{sec:numerics-complexity}.



\section{Numerical validation of the solver}
\label{app:numerics}

\label{sec:numerical}

This section reports an implementation of \cref{alg:pawc-main} and the programme used
to verify it.  The cut envelope \eqref{eq:cut-envelope} supplies a direct exact
oracle -- enumerate every original support gap, cut and unwrap there, run line
\PAWL{} once to obtain all $C_{k,r}^{\mathrm{line}}$, and compare the \PAWC{}
output to $\min_r C_{k,r}^{\mathrm{line}}$ for every $k$ -- and a
cardinality-constrained assignment solver applied separately for each $k$ supplies
a second, slower oracle that does not use the circle-cut structure at all.
Everything below is reproducible from a fixed set of seeds; the timing data and
the fitted slopes are emitted directly by the benchmark script into a CSV whose
header records the environment, package versions, and machine.

\subsection{Solvers and the trust hierarchy}
\label{sec:numerics-solvers}

Four solvers are compared, ordered by how much structure they assume.  A
disagreement is always resolved in favour of the solver that assumes less.

\begin{description}[leftmargin=1.6em,itemsep=2pt]
\item[Baseline A (LP oracle).]  For each $k$ separately, the cardinality-constrained
problem \eqref{eq:integer-problem} is solved as a linear program over the cardinality-$k$ matching
polytope $\mathcal P_k$, using the full circular cost matrix
$\dcirc(x_i,y_j)$.  It uses no cut, no cell and no chain, so it is independent of
every structural result proved above.  Its extreme points are integral, and the
returned vertex is checked to be a matching.
\item[Baseline B (cut enumeration).]  The exact cut envelope
\eqref{eq:cut-envelope}: cut at each of the $N$ original support gaps, unwrap, run
line \PAWL{} once per cut, and minimise over cuts for every $k$.  This is the
$O(N^2\log N)$ reference solver of \cref{cor:cut-envelope}.
\item[\PAWC{} (reference).]  A transparent transcription of \cref{alg:pawc-main} that
mirrors the proof structure and carries the invariant assertions of
\cref{sec:numerics-invariants}.
\item[\PAWC{}.]  The optimised implementation, in two variants -- one in plain
array code and one just-in-time compiled.
\end{description}

Instances are drawn from seven families chosen to stress different parts of the
argument: uniform; clustered (long free gaps and long chains); near-antipodal
pairs, perturbed off the exact tie excluded by \cref{ass:standing}; all atoms
inside a short arc, where the circle should degenerate to the line; strongly
unbalanced $n\ne m$; near-regular spacing, which produces many near-equal
candidate marginals; and strict source/target alternation around the circle.

\subsection{Exactness}
\label{sec:numerics-exactness}

On $504$ random instances with $n,m\in[2,30]$ spanning all seven families, the
cut envelope and the LP oracle agree at every cardinality $k=0,\dots,K$ to
$10^{-9}$, confirming \cref{prop:fixed-cut,cor:cut-envelope}
numerically.  Both \PAWC{} variants reproduce these values on the same battery, and
agree with the cut envelope on larger instances up to $N=2000$.

Two checks guard against the agreement being an artefact of floating point.
First, the whole comparison is repeated in exact rational arithmetic: support
points are taken as distinct multiples of $1/d$, gap midpoints are then exactly
rational, and the line profile under each cut is recomputed by an
$O(nmK)$ dynamic program over non-crossing matchings, which is exact because
$\lvert u-v\rvert$ satisfies the Monge condition on the line.  Second, the two
optimised variants are required to reproduce the reference implementation
\emph{bit for bit} -- identical \texttt{float64} cost arrays and identical
activation orders, not merely costs within a tolerance.  This is a strong test:
it fails under any reassociation of the floating-point sums, and it is the reason
the recursion \eqref{eq:Q-recurrence} is evaluated sequentially rather than as a
segmented cumulative sum (\cref{rem:numerics-Q}).

A separate check confirms the premise of \cref{prop:fixed-cut}: cuts
placed at uniformly random positions on the circle never beat the minimum over
the $N$ inter-point gaps, and moving the cut within a gap leaves the unwrapped
profile unchanged.

\subsection{Structural invariants}
\label{sec:numerics-invariants}

Cost agreement alone is a weak test, because optima need not be unique: a solver
can return the right number from a structurally wrong state.  The following are
therefore asserted directly, at \emph{every} induction step $k$ of every
instance.

\begin{enumerate}[leftmargin=1.8em,itemsep=2pt]
\item \emph{Nestedness} (\cref{thm:nested,cor:nested-sequence}):
$A_k\subset A_{k+1}$, and the increment consists of exactly one source atom and
one target atom.
\item \emph{Cyclic-neighbor addition} (\cref{thm:cyclic-neighbor}): the two newly
activated atoms are consecutive in the cyclic ordering of the inactive set
$A_k^{c}$.
\item \emph{Balanced cell interiors} (\cref{lem:balanced-interiors}): the active
atoms strictly inside each current cell contain equally many sources and targets.
\item \emph{Free-gap invariant} (\cref{lem:free-gap}): the current cells are
recomputed from scratch, and every one of them is confirmed to contain a gap
belonging to no previously selected cell.
\item \emph{Cut compatibility} (\cref{lem:cut-compatibility}), in a strengthened
form: cutting at \emph{any} free gap of \emph{any} current cell reproduces
$C_k^{\circ}$ exactly, not merely bounds it.
\item \emph{Cost consistency}: the incrementally maintained cost is compared with
$W_{1,\circ}(A_k)$ recomputed from the circulation form
\eqref{eq:circular-median} -- a weighted median over gap lengths, which involves
no cut, cell or chain and is therefore independent of how the algorithm arrived
at its value.
\item \emph{Constant-time marginal} (\cref{cor:constant-marginal}): the
four-lookup formula \eqref{eq:constant-marginal} is checked against
\cref{def:cell-marginal} evaluated by explicit sorted matching inside the cell.
\end{enumerate}

At the level of whole instances we further verify that $\theta^\star$ realises
$C_k^{\circ}$ as a line cost simultaneously for every $k$
(\cref{thm:simultaneous-cut}) and lies in a gap contained in no selected cell;
that $k\mapsto C_k^{\circ}$ is nondecreasing and convex; that recovered plans are
feasible, with marginals at most $w$, total mass $kw$, and cost equal to the
reported cost when re-evaluated with $\dcirc$ directly; that
\cref{prop:interpolation} holds at noninteger transported mass; and that the whole
profile is invariant under rotation of the circle, under reflection, and under
exchanging $\mu$ and $\nu$.  Rotation invariance is the sharpest of these in
practice, since wrap-around handling is where circular implementations
characteristically fail.  Degenerate configurations are covered separately:
$k=0$; $k=1$, where $C_1^{\circ}$ must equal the smallest geodesic distance over
all pairs; $n=1$ or $m=1$; and all atoms confined to a short arc, where the
profile must coincide with a single line \PAWL{} call.

\begin{remark}[On evaluating \eqref{eq:Q-recurrence}]
\label{rem:numerics-Q}
The recursion $Q_t=Q_{p_t}+w\lvert S_t-S_{p_t}\rvert$ chains together exactly the
indices sharing a differential rank, so it looks like a segmented cumulative sum
and is tempting to vectorise as a global prefix sum minus its value at each
segment start.  That form is numerically inferior: it differences two partial
sums whose magnitude grows like $NL$ in order to recover an $O(1)$ chain cost.
Measured against the sequential recursion it loses about $1.6\cdot 10^{-9}$
absolute at $N=10^{6}$, which is enough to break exact agreement between
implementations.  Since the recursion is $O(N)$ either way, \cref{thm:complexity}
is indifferent to the choice, and the sequential form is preferable.
\end{remark}

\subsection{Complexity}
\label{sec:numerics-complexity}

\Cref{fig:timings} and \cref{tab:timings} report the cost of computing the
\emph{complete} profile, i.e.\ all cardinalities $k=0,\dots,K$, for
$N$ from $10^{2}$ to $10^{6}$, with five repetitions per size and just-in-time
compilation excluded by an explicit warm-up.  Baseline B is measured to
$N=3\cdot 10^{4}$ and extrapolated beyond that along its theoretical
$N^{2}\log N$ slope, the constant being fitted at the largest measured size.
Baseline A is measured to $N=10^{3}$, beyond which one linear program per
cardinality is no longer practical, and is extrapolated along $N^{3}$.

\begin{table}[t]
\centering
\small
\begin{tabular}{rrrrr}
\toprule
$N$ & \PAWC{} (array) & \PAWC{} (compiled) & Baseline B & Baseline A \\
\midrule
$10^{2}$          & $0.000094$\,s & $0.000042$\,s & $0.0042$\,s & $0.212$\,s \\
$3\cdot 10^{2}$   & $0.000177$\,s & $0.000052$\,s & $0.0315$\,s & $4.52$\,s \\
$10^{3}$          & $0.000555$\,s & $0.000134$\,s & $0.349$\,s & $226$\,s \\
$3\cdot 10^{3}$   & $0.00170$\,s & $0.00048$\,s & $3.35$\,s & --- \\
$10^{4}$          & $0.00643$\,s & $0.00178$\,s & $38.5$\,s & $2.26\cdot 10^{5}$\,s$^{\ddagger}$ \\
$3\cdot 10^{4}$   & $0.0217$\,s  & $0.00574$\,s & $357$\,s & $6.10\cdot 10^{6}$\,s$^{\ddagger}$ \\
$10^{5}$          & $0.108$\,s   & $0.0217$\,s  & $4.43\cdot 10^{3}$\,s$^{\dagger}$ & --- \\
$3\cdot 10^{5}$   & $0.470$\,s   & $0.0777$\,s  & $4.37\cdot 10^{4}$\,s$^{\dagger}$ & --- \\
$10^{6}$          & $2.13$\,s    & $\mathbf{0.365}$\,s & $5.32\cdot 10^{5}$\,s$^{\dagger}$ & --- \\
\bottomrule
\end{tabular}
\caption{Time to compute the complete partial-transport profile for all $k$, mean
of five repetitions (two for Baseline A at $N\ge 600$).
$^{\dagger}$extrapolated along $N^{2}\log N$, $^{\ddagger}$along $N^{3}$, in both
cases with the constant fitted at the largest measured size; a dash marks sizes
past which no extrapolation is attempted.  At $N=10^{6}$ the compiled implementation returns
every cardinality in $0.365$\,s, against roughly six days for cut enumeration.
Baseline A already needs almost four minutes at $N=10^{3}$, where \PAWC{} needs
$1.3\cdot 10^{-4}$\,s.}
\label{tab:timings}
\end{table}

Fitting $\log(\text{time})$ against $\log N$ gives a slope of $2.00$ over all
sizes and $2.03$ over the top decade for cut enumeration, matching its
$O(N^{2}\log N)$ cost.  Baseline A fits at $3.05$, matching the cubic
cost of the task it performs: computing the whole profile by successive shortest
augmenting paths takes $\Theta(N)$ augmentations, each a shortest-path
computation costing $O(N^{2})$ on the dense bipartite graph with node potentials
\citep{ahuja1993network}.  A pure $N^{3}$ law fitted at the largest measured size
reproduces every measured point to within $35\%$ across a decade, which is why
its curve is extrapolated on the same footing as Baseline B's.  For \PAWC{} the same fit gives $1.11$ and
$1.29$ (array code), and $1.03$ and $1.23$ (compiled).  The top-decade values
exceed what one might naively expect, and the discrepancy deserves comment
rather than rounding.

Three observations place it.  First, $N\log N$ itself does not have log-log
slope $1$ on this range: over the top decade its slope is $1.079$, so the
relevant target is $1.08$ and not $1.0$.  Second, restricted to
$N\le 3\cdot 10^{4}$, where the preprocessing tables, the heap and the linked
list still fit in cache, the same code fits at $0.97$ (array) and $0.91$
(compiled); the excess appears only
from $N\ge 10^{5}$ upward, which is the signature of a memory-hierarchy
crossover rather than of an algorithmic term, since an unaccounted factor would
inflate the slope uniformly.  Third, and most directly, the same measurement
applied to line \PAWL{} -- an established $O(N\log N)$ partial-transport solver,
compiled and unmodified -- gives $1.123$ over all sizes and $1.251$ over the top
decade on the same machine, so \PAWC{} is in fact slightly flatter over the top
decade than the line algorithm it extends.

The machine-independent form of \cref{thm:complexity} is cleaner to verify
directly, and is: the heap performs $0.81N$ deletions and $0.31N$ insertions at
$N=10^{4}$, $10^{5}$ and $10^{6}$ alike.  That is, the total number of entries
ever created is linear, each is removed at most once whether valid or stale, and
at most one new candidate is inserted per iteration, exactly as
\cref{alg:pawc-main} prescribes.  Profiling at $N=10^{5}$ shows no stage growing
faster than the others, and in particular no residual cost from candidate-list
cleaning, which lazy deletion removes entirely.

\subsection{Is the simultaneous cut necessary?}
\label{sec:numerics-cut}

\Cref{ex:one-cut-unbounded} shows that a single cut chosen at $k=K$ can be
arbitrarily wrong at $k=1$.  That is a worst case; this subsection measures the
typical one.  Everything below is a reduction of the table
$T[r][k]=C_{k,r}^{\mathrm{line}}$ that the cut envelope already computes, so no
additional solver is involved, and all of it is computed in exact integer
arithmetic: whether a cut is optimal is an equality test, and a tolerance would
otherwise decide the answer.

We compare four rules for choosing one cut and reusing it for every $k$: the
deterministic $\min \Theta_K$; the widest support gap; a uniformly random gap, whose
expectation is evaluated exactly rather than sampled; and a warm start that fixes
the cut at $k_0=K/2$.  \Cref{tab:a1} reports $5600$ instances, $800$ from each of
the seven families of \cref{sec:numerics-solvers}; failure rates are stable to
$1.1$ percentage points against an independent replication of the same size.
\Cref{fig:a1} breaks the deterministic rule down by family and places the
adversarial family of \cref{ex:one-cut-unbounded} against its closed form.

\begin{table}[t]
\centering
\small
\begin{tabular}{lrrrr}
\toprule
cut rule & instances failing & median & 90th pct & worst \\
\midrule
$\min \Theta_K$ (optimal at $k=K$)        & $13.5\%$  & $0.09$ & $1.0$  & $1.8\cdot10^{3}$ \\
widest support gap                   & $22.6\%$  & $0.21$ & $0.7$  & $2.2$ \\
warm start at $k_0=K/2$              & $38.9\%$  & $0.28$ & $28$   & $3.4\cdot10^{4}$ \\
uniformly random gap                 & $100\%$   & $0.26$ & $60$   & $8.3\cdot10^{3}$ \\
\bottomrule
\end{tabular}
\caption{One-cut rules against the exact envelope, $5600$ exact instances.  A rule
\emph{fails} on an instance when its cut is strictly suboptimal at some $k$; the
excess columns are $\max_k\varepsilon_k$ over the instances where it does, with
$\varepsilon_k=(C_{k,r}^{\mathrm{line}}-C_k^{\circ})/C_k^{\circ}$.  Every cut is a
valid upper bound (\cref{cor:cut-envelope}); only looseness is at issue.}
\label{tab:a1}
\end{table}

\begin{figure}[t]
\centering
\begin{tikzpicture}[font=\small,
  ax/.style={draw=black!70,line width=.7pt},
  grid/.style={draw=black!12,line width=.4pt}]
\draw[grid] (0,0.000) -- (4.720,0.000);
\draw[ax] (0,0.000) -- (-.08,0.000);
\node[left,font=\scriptsize] at (-.10,0.000) {0\%};
\draw[grid] (0,0.949) -- (4.720,0.949);
\draw[ax] (0,0.949) -- (-.08,0.949);
\node[left,font=\scriptsize] at (-.10,0.949) {10\%};
\draw[grid] (0,1.897) -- (4.720,1.897);
\draw[ax] (0,1.897) -- (-.08,1.897);
\node[left,font=\scriptsize] at (-.10,1.897) {20\%};
\draw[grid] (0,2.846) -- (4.720,2.846);
\draw[ax] (0,2.846) -- (-.08,2.846);
\node[left,font=\scriptsize] at (-.10,2.846) {30\%};
\filldraw[draw=src!80!black,fill=src!55,line width=.6pt] (0.389,0) rectangle (0.791,1.138);
\node[above,font=\scriptsize] at (0.590,1.138) {12\%};
\node[below,font=\scriptsize,rotate=35,anchor=north east] at (0.590,-.05) {clust};
\filldraw[draw=src!80!black,fill=src!55,line width=.6pt] (0.979,0) rectangle (1.381,2.988);
\node[above,font=\scriptsize] at (1.180,2.988) {32\%};
\node[below,font=\scriptsize,rotate=35,anchor=north east] at (1.180,-.05) {inter};
\filldraw[draw=src!80!black,fill=src!55,line width=.6pt] (1.569,0) rectangle (1.971,2.134);
\node[above,font=\scriptsize] at (1.770,2.134) {22\%};
\node[below,font=\scriptsize,rotate=35,anchor=north east] at (1.770,-.05) {antip};
\filldraw[draw=src!80!black,fill=src!55,line width=.6pt] (2.159,0) rectangle (2.561,0.925);
\node[above,font=\scriptsize] at (2.360,0.925) {10\%};
\node[below,font=\scriptsize,rotate=35,anchor=north east] at (2.360,-.05) {grid};
\filldraw[draw=src!80!black,fill=src!55,line width=.6pt] (2.749,0) rectangle (3.151,0.806);
\node[above,font=\scriptsize] at (2.950,0.806) {8\%};
\node[below,font=\scriptsize,rotate=35,anchor=north east] at (2.950,-.05) {arc};
\filldraw[draw=src!80!black,fill=src!55,line width=.6pt] (3.339,0) rectangle (3.741,0.012);
\node[above,font=\scriptsize] at (3.540,0.012) {0\%};
\node[below,font=\scriptsize,rotate=35,anchor=north east] at (3.540,-.05) {unbal};
\filldraw[draw=src!80!black,fill=src!55,line width=.6pt] (3.929,0) rectangle (4.331,0.984);
\node[above,font=\scriptsize] at (4.130,0.984) {10\%};
\node[below,font=\scriptsize,rotate=35,anchor=north east] at (4.130,-.05) {unif};
\draw[ax] (0,0) -- (4.720,0);
\draw[ax] (0,0) -- (0,3.320);
\node[rotate=90,above,font=\scriptsize] at (-.85,1.660) {instances where $\min R_K$ is suboptimal};
\node[below] at (2.360,-1.15) {\small (a) failure of the deterministic one-cut rule};
\draw[grid] (11.940,0) -- (11.940,3.320);
\draw[ax] (11.940,0) -- (11.940,-.08);
\node[below,font=\scriptsize] at (11.940,-.10) {$10^{-7}$};
\draw[grid] (11.153,0) -- (11.153,3.320);
\draw[ax] (11.153,0) -- (11.153,-.08);
\node[below,font=\scriptsize] at (11.153,-.10) {$10^{-6}$};
\draw[grid] (10.367,0) -- (10.367,3.320);
\draw[ax] (10.367,0) -- (10.367,-.08);
\node[below,font=\scriptsize] at (10.367,-.10) {$10^{-5}$};
\draw[grid] (9.580,0) -- (9.580,3.320);
\draw[ax] (9.580,0) -- (9.580,-.08);
\node[below,font=\scriptsize] at (9.580,-.10) {$10^{-4}$};
\draw[grid] (8.793,0) -- (8.793,3.320);
\draw[ax] (8.793,0) -- (8.793,-.08);
\node[below,font=\scriptsize] at (8.793,-.10) {$10^{-3}$};
\draw[grid] (8.007,0) -- (8.007,3.320);
\draw[ax] (8.007,0) -- (8.007,-.08);
\node[below,font=\scriptsize] at (8.007,-.10) {$10^{-2}$};
\draw[grid] (7.220,0) -- (7.220,3.320);
\draw[ax] (7.220,0) -- (7.220,-.08);
\node[below,font=\scriptsize] at (7.220,-.10) {$10^{-1}$};
\draw[grid] (7.220,0.000) -- (11.940,0.000);
\draw[ax] (7.220,0.000) -- (7.140,0.000);
\node[left,font=\scriptsize] at (7.120,0.000) {$10^{0}$};
\draw[grid] (7.220,1.107) -- (11.940,1.107);
\draw[ax] (7.220,1.107) -- (7.140,1.107);
\node[left,font=\scriptsize] at (7.120,1.107) {$10^{2}$};
\draw[grid] (7.220,2.213) -- (11.940,2.213);
\draw[ax] (7.220,2.213) -- (7.140,2.213);
\node[left,font=\scriptsize] at (7.120,2.213) {$10^{4}$};
\draw[grid] (7.220,3.320) -- (11.940,3.320);
\draw[ax] (7.220,3.320) -- (7.140,3.320);
\node[left,font=\scriptsize] at (7.120,3.320) {$10^{6}$};
\draw[ax] (7.220,0) rectangle (11.940,3.320);
\draw[draw=black!30,line width=2.6pt] (7.457,0.264) -- (8.243,0.930) -- (9.030,1.492) -- (9.817,2.047) -- (10.605,2.601) -- (11.406,3.165);
\draw[draw=tgt,line width=1.1pt] (7.457,0.264) -- (8.243,0.930) -- (9.030,1.492) -- (9.817,2.047) -- (10.605,2.601) -- (11.406,3.165);
\filldraw[draw=tgt,fill=tgt] (7.457,0.264) circle (1.8pt);
\filldraw[draw=tgt,fill=tgt] (8.243,0.930) circle (1.8pt);
\filldraw[draw=tgt,fill=tgt] (9.030,1.492) circle (1.8pt);
\filldraw[draw=tgt,fill=tgt] (9.817,2.047) circle (1.8pt);
\filldraw[draw=tgt,fill=tgt] (10.605,2.601) circle (1.8pt);
\filldraw[draw=tgt,fill=tgt] (11.406,3.165) circle (1.8pt);
\node[below] at (9.580,-1.15) {\small (b) adversarial family: $\varepsilon_1$ against $\delta/L$};
\node[below,font=\scriptsize] at (9.580,-.42) {$\delta/L$ (decreasing)};
\node[rotate=90,above,font=\scriptsize] at (6.370,1.660) {$\varepsilon_1$ at the $k=K$-optimal cut};
\node[right,font=\scriptsize,align=left] at (7.570,2.870) {grey: closed form $(d_2-\delta)/\delta$\\orange: measured, exact};
\end{tikzpicture}
\caption{(a) Failure rate of the deterministic rule $\min \Theta_K$ by instance family,
$800$ instances each.  (b) The family of \cref{ex:one-cut-unbounded}: measured
relative excess at $k=1$ of a cut that is exactly optimal at $k=K$, against the
closed form \eqref{eq:one-cut-blowup}, in exact arithmetic over six decades of
$\delta$.}
\label{fig:a1}
\end{figure}
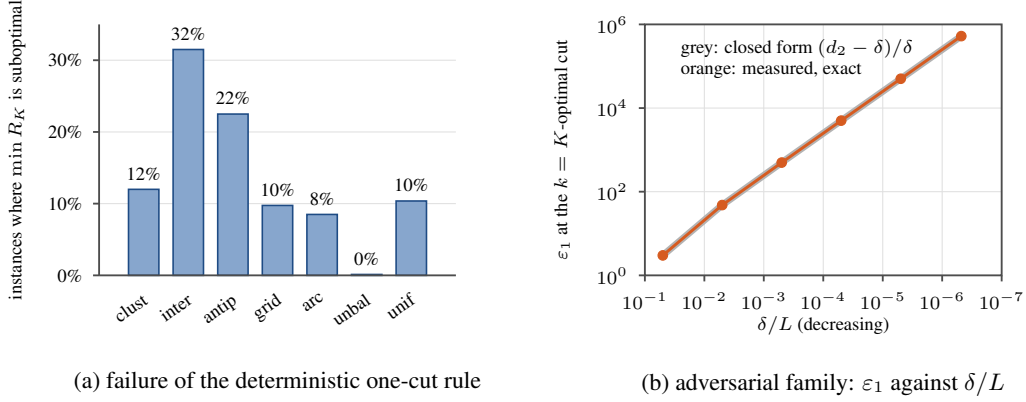

Three things are worth stating plainly, including the one that cuts against the
result we set out to find.

\emph{The deterministic rule usually works.}  It is optimal at every $k$ on $86.5\%$
of instances.  A reader who expects one-cut heuristics to fail routinely on random
data is wrong, and we do not claim otherwise.  What it lacks is a certificate: on
the $13.5\%$ where it fails there is nothing local distinguishing them, and the
excess reaches three orders of magnitude.  Its failure rate is also uneven --
$0.1\%$ on strongly unbalanced instances, $31.5\%$ on strictly alternating ones --
so it cannot be excused as a boundary effect.  Two further caveats sharpen the
comparison rather than soften it.  First, obtaining $\Theta_K$ is only cheap when
$n=m$, where the $k=K$ problem is balanced and its cut is a weighted median of
\eqref{eq:circular-median}; when $n\ne m$ the full-mass problem is itself partial,
so the rule presupposes the machinery it claims to avoid.  Restricted to $n=m$,
where it is genuinely cheap, its failure rate rises to $27.3\%$.  Second, the
choice within $\Theta_K$ matters: $\min \Theta_K$ is one of several optimal cuts at $k=K$,
and a different element of that same set fails on $46.5\%$ of instances.

\emph{A simultaneously optimal cut is common, not rare.}  The set $\Theta_\cap$ of gaps
optimal at every $k$ has median size between $0.17N$ and $0.92N$ depending on the
family.  \Cref{thm:simultaneous-cut} therefore does not locate a needle in a
haystack on random data, and we should not claim it does.  Its content is that such
a gap exists \emph{always} and is produced constructively, not that it is hard to
stumble on.  Correspondingly, a uniformly random element of $\Theta_K$ is simultaneously
optimal with median probability $1$, but with probability as low as $0.14$ on the
worst instance encountered -- which is the same point as before in another form:
the typical case is benign and carries no guarantee.

\emph{The optimal cut does move with $k$.}  Tracking $\min R_k$ as $k$ grows, the
index changes on a majority of instances in every family except the strongly
unbalanced one.  This is the concrete form of the claim in \cref{sec:intro} that
the optimal circulation varies with transported mass, and it is why a warm start
degrades as $k_0$ moves away from the cardinality of interest: fixing the cut at
$k_0=K/2$ already fails on $38.9\%$ of instances.

\subsection{The spherical construction}
\label{sec:numerics-sphere}

\Cref{sec:slicing} is a construction; this subsection is the evidence for it.  Two
things are measured: that the implementation reproduces an independent published
one at the single point where the two theories must coincide, and how the
Monte-Carlo estimator \eqref{eq:sspw-mc} actually behaves against
\cref{prop:sspw-concentration}.

\paragraph{Reduction to the balanced case.}
\Cref{rem:sspw-inherited}(4) states that at $n=m$ and $s=nw$ every slice problem is
balanced and $\mathrm{SSPW}(nw;\mu,\nu)=nw\,\mathrm{SSW}_1(\mu/nw,\nu/nw)$.  This is
the only quantity in \cref{sec:slicing} with an external reference, and reproducing
it exercises the whole pipeline: Stiefel sampling, geodesic projection, the
circular coordinate \eqref{eq:slice-angle}, \cref{alg:pawc-main}, and the slice average.
We compare against the implementation shipped with POT \citep{flamary2021pot},
which follows \citet{bonet2023spherical}.

The comparison uses \emph{the same slices} on both sides.  Letting each
implementation draw its own would compare two Monte-Carlo estimates, and agreement
would then only be required in distribution -- a weak test that a systematic bias
of the size of the Monte-Carlo noise would pass.  With shared slices both sides
evaluate the same finite average of the same integrand, so agreement is required to
solver tolerance.  \Cref{tab:c3} reports five checks, each isolating what the next
would aggregate away.

\begin{table}[t]
\centering
\small
\begin{tabular}{llrr}
\toprule
 & what it isolates & vs.\ POT & vs.\ LP oracle \\
\midrule
(a) & circular solver alone, no projection & $3.3\cdot10^{-15}$ & $1.4\cdot10^{-15}$ \\
(b) & per slice on the sphere, $384$ slices & $7.5\cdot10^{-9}$ & $5.6\cdot10^{-15}$ \\
(c) & the aggregate \eqref{eq:sspw} against $\mathrm{SSW}_1$ & $1.3\cdot10^{-9}$ & -- \\
(d) & the mass factor \eqref{eq:sspw-reduction}, four $(n,w)$ & $1.5\cdot10^{-15}$ & -- \\
(e) & the $d=2$ degeneracy of \cref{rem:sspw-scope} & $7.9\cdot10^{-16}$ & -- \\
\bottomrule
\end{tabular}
\caption{Maximum relative discrepancy of the reduction cross-check, over
$d\in\{3,5,10,50\}$, $n=m\in\{10,50,200\}$ and both instance families, with slices
shared between implementations.  Check (e) additionally has zero slice-to-slice
variance ($6.7\cdot10^{-16}$), as it must, since at $d=2$ the projection is an
isometry and every slice returns the same value.}
\label{tab:c3}
\end{table}

Checks (b) and (c) sit near $10^{-9}$ rather than near $10^{-15}$, and the cause is
worth recording because it is not on our side.  POT's circular solver dispatches to
a binary search for every $p$ -- it provides no exact $p=1$ circular routine -- and
that search terminates at a tolerance of $10^{-6}$, so $10^{-8}$ is its own accuracy
floor.  On the worst slice of check (b) our value agrees with an LP over the
geodesic cost matrix \emph{exactly}, to the last bit, while POT differs from that LP
by $7.5\cdot10^{-9}$.  The acceptance criterion is therefore split: exact against
the LP oracle, POT's own tolerance against POT.  Requiring $10^{-10}$ of POT would
be requiring it to be more accurate than it is.

\paragraph{Concentration.}
\Cref{prop:sspw-concentration} is next.  We take a reference at $M_{\mathrm{ref}}=10^{5}$
slices as ground truth, drawn from a slice stream \emph{disjoint} from the estimates
under test -- if the estimate's slices were a subset of the reference's, deviations
would be biased downward and the convergence would flatter itself -- and use
$200$ repetitions at each $M\in\{2^3,\dots,2^{12}\}$, over $d\in\{3,10,50\}$,
three families and both $n=m$ and $n\neq m$.

Three checks pass and are reported without further comment: the estimator converges
at the rate the bound implies, with fitted slope $-0.497$ across the eighteen
configurations and none outside $[-0.519,-0.476]$; the bound of
\eqref{eq:sspw-hoeffding} is never violated, on any of the $540$ configuration and
threshold combinations tested; and the piecewise linearity in $s$ that the proof of
\eqref{eq:sspw-uniform} relies on holds exactly, the deviation on a fine $s$-grid
never exceeding its value at the breakpoints.  That last one is a step of the proof
rather than an implementation detail, which is why it is checked directly.

What the measurements add is a quantitative account of where
\cref{prop:sspw-concentration} is loose, and there are two distinct answers.

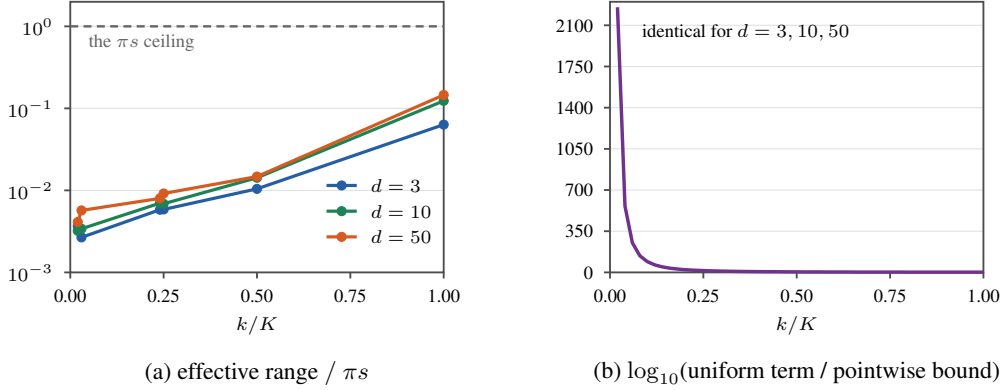
\begin{figure}[t]
\centering
\begin{tikzpicture}[font=\small,
  ax/.style={draw=black!70,line width=.7pt},
  grid/.style={draw=black!12,line width=.4pt}]
\draw[grid] (0,0.000) -- (4.940,0.000);
\draw[ax] (0,0.000) -- (-.08,0.000);
\node[left,font=\scriptsize] at (-.10,0.000) {$10^{-3}$};
\draw[grid] (0,1.085) -- (4.940,1.085);
\draw[ax] (0,1.085) -- (-.08,1.085);
\node[left,font=\scriptsize] at (-.10,1.085) {$10^{-2}$};
\draw[grid] (0,2.170) -- (4.940,2.170);
\draw[ax] (0,2.170) -- (-.08,2.170);
\node[left,font=\scriptsize] at (-.10,2.170) {$10^{-1}$};
\draw[grid] (0,3.255) -- (4.940,3.255);
\draw[ax] (0,3.255) -- (-.08,3.255);
\node[left,font=\scriptsize] at (-.10,3.255) {$10^{0}$};
\draw[ax] (0.000,0) -- (0.000,-.08);
\node[below,font=\scriptsize] at (0.000,-.10) {0.00};
\draw[ax] (1.235,0) -- (1.235,-.08);
\node[below,font=\scriptsize] at (1.235,-.10) {0.25};
\draw[ax] (2.470,0) -- (2.470,-.08);
\node[below,font=\scriptsize] at (2.470,-.10) {0.50};
\draw[ax] (3.705,0) -- (3.705,-.08);
\node[below,font=\scriptsize] at (3.705,-.10) {0.75};
\draw[ax] (4.940,0) -- (4.940,-.08);
\node[below,font=\scriptsize] at (4.940,-.10) {1.00};
\draw[ax] (0,0) rectangle (4.940,3.580);
\draw[densely dashed,black!60,line width=.9pt] (0,3.255) -- (4.940,3.255);
\node[right,font=\scriptsize,black!60] at (.12,3.035) {the $\pi s$ ceiling};
\draw[draw=src,line width=1.2pt] (0.099,0.613) -- (0.148,0.463) -- (1.186,0.829) -- (1.235,0.833) -- (2.470,1.105) -- (4.940,1.955);
\filldraw[draw=src,fill=src] (0.099,0.613) circle (1.7pt);
\filldraw[draw=src,fill=src] (0.148,0.463) circle (1.7pt);
\filldraw[draw=src,fill=src] (1.186,0.829) circle (1.7pt);
\filldraw[draw=src,fill=src] (1.235,0.833) circle (1.7pt);
\filldraw[draw=src,fill=src] (2.470,1.105) circle (1.7pt);
\filldraw[draw=src,fill=src] (4.940,1.955) circle (1.7pt);
\draw[draw=transportcol,line width=1.2pt] (0.099,0.550) -- (0.148,0.577) -- (1.186,0.916) -- (1.235,0.908) -- (2.470,1.252) -- (4.940,2.269);
\filldraw[draw=transportcol,fill=transportcol] (0.099,0.550) circle (1.7pt);
\filldraw[draw=transportcol,fill=transportcol] (0.148,0.577) circle (1.7pt);
\filldraw[draw=transportcol,fill=transportcol] (1.186,0.916) circle (1.7pt);
\filldraw[draw=transportcol,fill=transportcol] (1.235,0.908) circle (1.7pt);
\filldraw[draw=transportcol,fill=transportcol] (2.470,1.252) circle (1.7pt);
\filldraw[draw=transportcol,fill=transportcol] (4.940,2.269) circle (1.7pt);
\draw[draw=tgt,line width=1.2pt] (0.099,0.667) -- (0.148,0.820) -- (1.186,0.979) -- (1.235,1.044) -- (2.470,1.268) -- (4.940,2.347);
\filldraw[draw=tgt,fill=tgt] (0.099,0.667) circle (1.7pt);
\filldraw[draw=tgt,fill=tgt] (0.148,0.820) circle (1.7pt);
\filldraw[draw=tgt,fill=tgt] (1.186,0.979) circle (1.7pt);
\filldraw[draw=tgt,fill=tgt] (1.235,1.044) circle (1.7pt);
\filldraw[draw=tgt,fill=tgt] (2.470,1.268) circle (1.7pt);
\filldraw[draw=tgt,fill=tgt] (4.940,2.347) circle (1.7pt);
\draw[draw=src,line width=1.2pt] (3.390,1.150) -- (3.790,1.150);
\filldraw[draw=src,fill=src] (3.590,1.150) circle (1.7pt);
\node[right,font=\scriptsize] at (3.850,1.150) {$d=3$};
\draw[draw=transportcol,line width=1.2pt] (3.390,0.810) -- (3.790,0.810);
\filldraw[draw=transportcol,fill=transportcol] (3.590,0.810) circle (1.7pt);
\node[right,font=\scriptsize] at (3.850,0.810) {$d=10$};
\draw[draw=tgt,line width=1.2pt] (3.390,0.470) -- (3.790,0.470);
\filldraw[draw=tgt,fill=tgt] (3.590,0.470) circle (1.7pt);
\node[right,font=\scriptsize] at (3.850,0.470) {$d=50$};
\node[below] at (2.470,-1.05) {\small (a) effective range $/\ \pi s$};
\node[below,font=\scriptsize] at (2.470,-.45) {$k/K$};
\draw[grid] (7.140,0.000) -- (12.080,0.000);
\draw[ax] (7.140,0.000) -- (7.060,0.000);
\node[left,font=\scriptsize] at (7.040,0.000) {0};
\draw[grid] (7.140,0.545) -- (12.080,0.545);
\draw[ax] (7.140,0.545) -- (7.060,0.545);
\node[left,font=\scriptsize] at (7.040,0.545) {350};
\draw[grid] (7.140,1.090) -- (12.080,1.090);
\draw[ax] (7.140,1.090) -- (7.060,1.090);
\node[left,font=\scriptsize] at (7.040,1.090) {700};
\draw[grid] (7.140,1.634) -- (12.080,1.634);
\draw[ax] (7.140,1.634) -- (7.060,1.634);
\node[left,font=\scriptsize] at (7.040,1.634) {1050};
\draw[grid] (7.140,2.179) -- (12.080,2.179);
\draw[ax] (7.140,2.179) -- (7.060,2.179);
\node[left,font=\scriptsize] at (7.040,2.179) {1400};
\draw[grid] (7.140,2.724) -- (12.080,2.724);
\draw[ax] (7.140,2.724) -- (7.060,2.724);
\node[left,font=\scriptsize] at (7.040,2.724) {1750};
\draw[grid] (7.140,3.269) -- (12.080,3.269);
\draw[ax] (7.140,3.269) -- (7.060,3.269);
\node[left,font=\scriptsize] at (7.040,3.269) {2100};
\draw[ax] (7.140,0) -- (7.140,-.08);
\node[below,font=\scriptsize] at (7.140,-.10) {0.00};
\draw[ax] (8.375,0) -- (8.375,-.08);
\node[below,font=\scriptsize] at (8.375,-.10) {0.25};
\draw[ax] (9.610,0) -- (9.610,-.08);
\node[below,font=\scriptsize] at (9.610,-.10) {0.50};
\draw[ax] (10.845,0) -- (10.845,-.08);
\node[below,font=\scriptsize] at (10.845,-.10) {0.75};
\draw[ax] (12.080,0) -- (12.080,-.08);
\node[below,font=\scriptsize] at (12.080,-.10) {1.00};
\draw[ax] (7.140,0) rectangle (12.080,3.580);
\draw[draw=freegap,line width=1.3pt] (7.239,3.508) -- (7.338,0.878) -- (7.436,0.391) -- (7.535,0.220) -- (7.634,0.142) -- (7.733,0.099) -- (7.832,0.073) -- (7.930,0.056) -- (8.029,0.045) -- (8.128,0.036) -- (8.227,0.030) -- (8.326,0.026) -- (8.424,0.022) -- (8.523,0.019) -- (8.622,0.017) -- (8.721,0.015) -- (8.820,0.013) -- (8.918,0.012) -- (9.017,0.011) -- (9.116,0.010) -- (9.215,0.009) -- (9.314,0.009) -- (9.412,0.008) -- (9.511,0.007) -- (9.610,0.007) -- (9.709,0.006) -- (9.808,0.006) -- (9.906,0.006) -- (10.005,0.005) -- (10.104,0.005) -- (10.203,0.005) -- (10.302,0.005) -- (10.400,0.004) -- (10.499,0.004) -- (10.598,0.004) -- (10.697,0.004) -- (10.796,0.004) -- (10.894,0.004) -- (10.993,0.004) -- (11.092,0.003) -- (11.191,0.003) -- (11.290,0.003) -- (11.388,0.003) -- (11.487,0.003) -- (11.586,0.003) -- (11.685,0.003) -- (11.784,0.003) -- (11.882,0.003) -- (11.981,0.003) -- (12.080,0.003);
\node[below] at (9.610,-1.05) {\small (b) $\log_{10}$(uniform term / pointwise bound)};
\node[below,font=\scriptsize] at (9.610,-.45) {$k/K$};
\node[right,font=\scriptsize,align=left] at (7.440,3.180) {identical for $d=3,10,50$};
\end{tikzpicture}
\caption{(a) The effective Hoeffding range $\max_\ell Z_\ell-\min_\ell Z_\ell$ as a
fraction of the ceiling $\pi s$ used in \cref{prop:sspw-concentration}, against
transported cardinality; median over families, $M_{\mathrm{ref}}=10^{5}$.  The
ceiling is never approached, and the gap widens as $s$ shrinks.  (b) The ratio of
the per-breakpoint term of the uniform bound \eqref{eq:sspw-uniform}, which uses the
range $\pi K w$, to the pointwise bound \eqref{eq:sspw-hoeffding} at the same $s$,
which uses $\pi s$; $\varepsilon=0.05\,Kw$, $M=2^{12}$.  The two coincide at $s=Kw$
and diverge without limit as $s\to0$.}
\label{fig:c1}
\end{figure}

First, the range.  The bound uses $0\le Z_\ell\le\pi s$, attained only when every
transported pair is antipodal.  Measured, the effective range
$\max_\ell Z_\ell-\min_\ell Z_\ell$ -- which is what Hoeffding's inequality actually
depends on -- is between $0.3\%$ and $15\%$ of $\pi s$
(\cref{fig:c1}(a)), and the ceiling itself is never approached: $\max_\ell Z_\ell$
stays below $0.29\,\pi s$ throughout.  The looseness is worst at small $s$.  It also
\emph{grows} with dimension rather than shrinking -- the median effective range is
$0.0067$, $0.0091$ and $0.0100$ of $\pi s$ at $d=3,10,50$ -- so the slack is not
explained by the projected measures approaching uniformity on the circle as $d$
grows, which would have been the natural guess.  Whatever a dimension-dependent
refinement of \cref{rem:sspw-open} would exploit, it is not that.

Second, and separately, the uniform bound \eqref{eq:sspw-uniform} is obtained by
relaxing the range at every breakpoint from $\pi kw$ to $\pi Kw$ before taking the
union bound.  For controlling the supremum this costs little, since the $k=K$ term
dominates the sum regardless.  But it makes \eqref{eq:sspw-uniform} a poor guarantee
\emph{at} a small $s$: its per-breakpoint term at $s=kw$ exceeds the pointwise bound
\eqref{eq:sspw-hoeffding} there by a factor whose logarithm grows like $K^2/k^2$,
reaching $10^{361}$ at $k/K=0.05$ for the configuration of \cref{fig:c1}(b), and
identically so at every $d$ tested.  The practical reading is that
\eqref{eq:sspw-uniform} should be used for what it is -- a statement about the whole
profile -- and \eqref{eq:sspw-hoeffding} used when accuracy is wanted at a
particular transported mass.  \Cref{rem:relative-accuracy} records the natural way
to have both at once.

Finally, the slice count.  \Cref{prop:sspw-concentration} demands
$M\ge(\pi^2K^2w^2/2\varepsilon^2)\log(2(K+1)/\delta)$; the smallest tested $M$ at
which no repetition out of $200$ exceeded $\varepsilon$ is smaller than that by a
median factor of $470$, and the empirical figure is itself conservative, being an
observed tail of $0/200$ rather than of $\delta$.  Two orders of magnitude is the
usual price of a distribution-free bound, and it is stated here so that a reader
sizing a computation uses the measurement rather than the proposition.

\begin{remark}[On the two \PAWC{} variants]
\label{rem:numerics-variants}
The array implementation and the compiled one produce bit-identical output; the
latter is $3.7$ to $6.2$ times faster.  Both are reported because Baseline B
calls line \PAWL{}, which is itself compiled: comparing against the array
implementation understates \PAWC{}'s advantage by a constant factor, while the
compiled variant is the like-for-like comparison.  The asymptotic separation
visible in \cref{fig:timings} does not depend on the choice.
\end{remark}


\section{Application detail and what did not work}
\label{app:apps}

\subsection{Shape matching: protocol}

Contours are traced at the $0.5$ level of each \textsc{mpeg-7} mask, resampled to $128$
points uniformly in arclength, and converted to outward normal angles by a centred
difference over five samples.  Pixel quantisation makes long straight runs share a
tangent \emph{exactly}, which violates \cref{ass:standing}; a deterministic jitter of
$10^{-7}$ circle units (seeded by a checksum of the file name, so it is reproducible)
separates them.  That is $1.3\times10^{-4}$ degrees, four orders of magnitude below the
coarsest quantisation step and far below any cost difference these tasks resolve --- it
is the symbolic perturbation of \cref{sec:discussion} made concrete.

A query keeps a contiguous fraction $v$ of the arclength starting at a uniformly random
sample, replaces a fraction $c$ of the survivors by normals drawn uniformly on the
circle, and is rotated by a uniformly random angle.  Every method scores the same grid
of $360$ rotations.  For retrieval, the database holds the intact signatures of all
$1400$ shapes and the query's own entry is excluded; the bullseye score counts how many
of the $20$ shapes of the query's class appear in the top $40$.  For correspondences the
plan is read at the rotation that minimises the cost at $\rho=0.8$, and a pair is
correct if it matches the arclength index the query sample was built from; clutter
samples have no correct partner, so asserting one is an error.

The correlation baseline is the classical extended-Gaussian-image matcher: a von Mises
kernel density on each measure, cross-correlated over rotations by FFT, at bandwidths
$0.01$, $0.02$ and $0.04$ circle units, reported at its best bandwidth per cell.  The
nearest-neighbour baseline for correspondences is mutual nearest neighbour in normal
angle after alignment.

\subsection{Why not unbalanced transport, and is this the same optimum?}

Two checks on one shape per class over the same twelve cells ($864$ queries).  First, a
KL-relaxed unbalanced Sinkhorn cost \citep{peyre2019computational} replaces the partial
one, at marginal weights $\tau\in\{0.01,0.05,0.2\}$ and entropic regularisation
$10^{-2}$, reported at its best $\tau$ per cell.  It matches \PAWC{} while the query is
nearly complete and falls behind as the boundary disappears --- $0.71$ against $0.85$
success at $v=0.35$, $c=0.4$ --- and no single $\tau$ is competitive across cells: the
$\tau$ that is best at full visibility scores $0.61$ there and $\tau=0.2$ scores $0.44$.
The relaxation weight is a parameter to tune; $\rho=1-c$ is a rule.  The rotation sweep
also costs $0.58$\,s per query per $\tau$ against $0.029$\,s for all $360$ \PAWC{}
solves.

Second, on a sample of rotations from every query we compared \PAWC{}'s profile value
against POT's \texttt{partial\_wasserstein2} at the same cardinality with unit atom
weights --- the same problem, not a proportional one.  The largest absolute difference
over the $864$ queries is $4.3\times10^{-14}$.  The applications therefore read exactly
the optimum a general partial solver reports, $13$ times faster at $n=128$ and $2700$
times faster by $n=2048$.

\subsection{Spherical fitting: gradients and protocol}

For one slice $U$ and $p=U^{\top}x$, the circle coordinate is
$\theta_U(x)=\operatorname{atan2}(p_2,p_1)$ with
$\partial\theta/\partial x=U(-p_2,p_1)^{\top}/\|p\|^{2}$; the matched geodesic
distances contribute $\pm1$ through $\partial/\partial\theta$.  The matching found by
\PAWC{} is held fixed while differentiating (the envelope theorem), the gradient is
projected onto the tangent space at each particle and the step retracted by
renormalising.  Analytic gradients agree with central differences to $3\times10^{-6}$
relative.

Training uses $500$ particles from a uniform start, $32$ slices redrawn every step,
$2000$ Adam steps with a cosine schedule from $8\times10^{-2}$ to $4\times10^{-3}$, and
target batches of size $\gamma\cdot500$.  The synthetic target has five von
Mises--Fisher components with masses $0.28,0.22,0.18,0.20,0.12$ and concentrations
$\kappa=80,200,60,120,150$.  The catalogue is split $80/20$ once, with a fixed
permutation seed; the model never sees the evaluation fifth.  Real catalogues round
their coordinates, so a target batch can contain exact duplicates: the tie-break
perturbation has magnitude in $[10^{-9},2\times10^{-9})$ with random sign, bounded away
from zero because a Gaussian jitter of that size falls below one ulp often enough to
leave ties in place over $10^{8}$ draws.

\end{document}